\documentclass{article}

\usepackage[dvips,letterpaper,margin=1.2in]{geometry}
\newcommand{\citet}[1]{\cite{#1}}
\newcommand{\citep}[1]{\cite{#1}}

\usepackage{hyperref}

\usepackage{amsmath}
\usepackage{amssymb}
\usepackage{mathtools}
\usepackage{amsthm}

\usepackage[capitalize,noabbrev]{cleveref}

\theoremstyle{plain}
\newtheorem{defi}{Definition}
\newtheorem{theorem}{Theorem}
\newtheorem{lemma}{Lemma}

\newtheorem{coro}{Corollary}

\newtheorem{prop}{Proposition}

\newtheorem{obs}{Observation}

\theoremstyle{remark}
\newtheorem*{remark}{Remark}

\newtheorem{innercustomgeneric}{\customgenericname}
\providecommand{\customgenericname}{}
\newcommand{\newcustomtheorem}[2]{%
  \newenvironment{#1}[1]
  {%
   \renewcommand\customgenericname{#2}%
   \renewcommand\theinnercustomgeneric{##1}%
   \innercustomgeneric
  }
  {\endinnercustomgeneric}
}
\newcustomtheorem{customremark}{Remark}

\newenvironment{sketch}{\paragraph{Proof sketch}}{\hfill$\square$}

\usepackage[textsize=tiny]{todonotes}

\usepackage{authblk}

\usepackage{subfigure}
\usepackage[utf8]{inputenc} 
\usepackage[T1]{fontenc}    
\usepackage{titletoc}
\usepackage{url}            
\usepackage{booktabs}       
\usepackage{amsfonts}       
\usepackage{nicefrac}       
\usepackage{microtype}      
\usepackage{xcolor}         
\newcommand{\joan}[1]{}

\newcommand{\red}[1]{{\color{black}#1}}

\usepackage{graphicx}

\usepackage{enumitem}

\usepackage{tikz-cd}
\usetikzlibrary{cd}

\usepackage{pifont}

\newcommand\scalemath[2]{\scalebox{#1}{\mbox{\ensuremath{\displaystyle #2}}}}

\DeclareMathOperator{\proj}{proj}
\newcommand{\norm}[1]{\left\lVert#1\right\rVert}
\newcommand{\inprod}[2]{\langle #1,#2 \rangle}
\newcommand{\T}[1]{^{(#1)}}
\def\t{^{(t)}}
\def\tt{^{(t+1)}}
\def\O{\mathcal{O}}
\def\R{\mathbb{R}}
\def\bC{\mathbf{C}}
\def\bX{\mathbf{X}}
\def\bY{\mathbf{Y}}
\def\bZ{\mathbf{Z}}

\newcommand{\cc}[1]{\text{\ding{\number\numexpr#1 + 171\relax}}}

\newcommand{\cbb}[1]{\noindent{\color{blue}{#1}}}  
\newcommand{\crr}[1]{\noindent{\color{red}{#1}}}  

\newcommand{\eqtext}[1]{\ensuremath{\stackrel{\text{#1}}{=\joinrel=\joinrel=}}}

\title{Beyond the Edge of Stability via Two-step Gradient Updates}

\author[a]{Lei Chen\footnote{\texttt{lc3909@nyu.edu}}}
\author[a,b]{Joan Bruna\footnote{\texttt{bruna@cims.nyu.edu}}}

\affil[a]{Courant Institute of Mathematical Sciences, New York
  University, New York}
\affil[b]{Center for Data Science, New York University, New York}

\begin{document}

\maketitle

\begin{abstract}
Gradient Descent (GD) is a powerful workhorse of modern machine learning thanks to its scalability and efficiency in high-dimensional spaces. Its ability to find local minimisers is only guaranteed for losses with Lipschitz gradients, where it can be seen as a `bona-fide' discretisation of an underlying gradient flow. Yet, many ML setups involving overparametrised models do not fall into this problem class, which has motivated research beyond the so-called ``Edge of Stability'' (EoS), where the step-size crosses the admissibility threshold inversely proportional to the Lipschitz constant above. Perhaps surprisingly, GD has been empirically observed to still converge regardless of local instability and oscillatory behavior.

The incipient theoretical analysis of this phenomena has mainly focused in the overparametrised regime, where the effect of choosing a large learning rate may be associated to a `Sharpness-Minimisation' implicit regularisation within the manifold of minimisers, under appropriate asymptotic limits. In contrast, in this work we directly examine the conditions for such unstable convergence, focusing on simple, yet representative, learning problems, via analysis of two-step gradient updates. Specifically, we characterize a local condition involving third-order derivatives that guarantees existence and convergence to fixed points of the two-step updates, and leverage such property in a teacher-student setting, under population loss. Finally, starting from Matrix Factorization, 
we provide observations of period-2 orbit of GD in high-dimensional settings with intuition of its dynamics, along with exploration into more general settings.


\end{abstract}


\section{Introduction}

Given a differentiable objective function $f(\theta)$, where $\theta \in \mathbb{R}^d$ is a high-dimensional parameter vector, the most basic and widely used optimization method is gradient descent (GD), defined as 
\begin{align}
\label{eq:GD}
    \theta\tt = \theta\t - \eta \nabla_\theta f(\theta\t),
\end{align}
where $\eta$ is the learning rate. For all its widespread application across many different ML setups, a basic question remains: what are the convergence guarantees (even to a local minimiser) under typical objective functions, and how they depend on the (only) hyperaparameter $\eta$?
In the modern context of large-scale ML applications, an additional key question is not only to understand whether or not GD converges to minimisers, but to \emph{which} ones, since overparametrisation defines a whole manifold of global minimisers, all potentially enjoying drastically different generalisation performance. 

The sensible regime to start the analysis is $\eta\to 0$, where GD inherits the local convergence properties of the Gradient Flow ODE via standard arguments from numerical integration. However, in the early phase of training, a large learning rate has been observed to result in better generalization~\citep{lecun2012efficient,bjorck2018understanding,jiang2019fantastic,jastrzebski2021catastrophic}, 
where the extent of ``large'' is measured by comparing the learning rate $\eta$ and the curvature of the loss landscape, measured with $\lambda(\theta):=\lambda_{\text{max}}\left[\nabla_\theta^2 f(\theta)\right]$, the largest eigenvalue of the Hessian with respect to learnable parameters. Although one requires $\sup_\theta \lambda(\theta)< 2/\eta$ to guarantee the convergence of GD \citep{bottou2018optimization} to (local) minimisers \footnote{One can replace the uniform curvature bound by $\sup_{\theta; f(\theta) \leq f(\theta^{(0)})} \lambda(\theta)$.}, the work of \citep{cohen2020gradient} noticed a remarkable phenomena in the context of neural network training: even in problems where $\lambda(\theta)$ is unbounded (as in NNs), for a fixed $\eta$, the curvature $\lambda(\theta^{(t)})$ increases along the training trajectory (\ref{eq:GD}), bringing $\lambda(\theta^{(t)}) \ge 2/\eta$~\citep{cohen2020gradient}. After that, a surprising phenomena is that $\lambda(\theta^{(t)})$  \textit{stably} hovers above $2/\eta$ and the neural network 
still eventually achieves a decreasing training loss  --- the so-called ``Edge of Stability''. We would like to understand and analyse the conditions of such convergence with a large learning rate under a variety models that capture such observed empirical behavior. 

 Recently, some works have built connections between EoS and implicit bias~\citep{arora2022understanding, lyu2022understanding, damian2021label, damian2022self} in the context of large, overparametrised models such as neural networks. In this setting, GD is expected to converge to a manifold of minimisers, and the question is to what extent a large learning rate `favors' solutions with small curvature. In essence, these works show that under certain structural assumptions, GD is asymptotically tracking a continuous sharpness-reduction flow, in the limit of small learning rates. Compared with these, we study non-asymptotic properties of GD beyond EoS, by focusing on certain learning problems (\textit{e.g.}, single-neuron ReLU networks and matrix factorization). In particular, we characterize a range of learning rates $\eta$ \emph{above} the EoS such that GD dynamics hover around minimisers. Moreover, in the matrix factorization setup, where minimisers form a manifold with varying local curvature, our results give a non-asymptotic analogue of the `Sharpness-Minimisation' arguments from \cite{arora2022understanding, lyu2022understanding, damian2022self}.
 

The straightforward starting point for the local convergence analysis is via Taylor  approximations of the loss function. However, in a quadratic Taylor expansion, gradient descent diverges once $\lambda(\theta) > 2/\eta$~\citep{cohen2020gradient}, indicating that a higher order Taylor approximation is required. By considering a 1-D function with local minima $\theta^*$ of curvature $\lambda^* = \lambda(\theta^*)$, we show the existence of fixed points of two-step updates around the minima with $\eta$ slightly above the threshold $2/\lambda^*$, provided its high order derivative satisfies mild conditions as in Theorem~\ref{thm:1dlocal}, with generalization into matrix factorization in Theorem~\ref{thm:mf_1d_cond} and experiments of MLPs in Appendix~\ref{app:mnist}. A typical example of such functions is $f(x)=\frac{1}{4}(x^2-\mu)^2$ with $\mu>0$. Furthermore, we prove that it converges to an orbit of period 2 from a more global initialization rather than the analysis of high-order local approximation.

As it turns out, the analysis of such stable one-dimensional oscillations is sufficiently intrinsic to become useful in higher-dimensional problems. First,  
we leverage the analysis to a two-layer single-neuron ReLU network, where the task is to learn a teacher neuron with data on a uniform high-dimensional sphere. We show a convergence result under population loss with GD beyond EoS, where the direction of the teacher neuron can be learnt and the norms of two-layer weights stably oscillate, with empirical evidence of 16-neuron networks in Appendix~\ref{app:16-neuron}. We then focus on matrix factorization, a canonical non-convex problem whose geometry is characterized by a manifold of minimisers having different local curvature. 
We provide novel observations of its convergence to period-2 orbit with comprehensive theoretical intuition of the dynamics.
Finally, we extend previous works by proposing two models with observations in matrix factorization compatible for future analysis.
A further discussion is provided in Appendix~\ref{app:discuss}.

\section{Related Work}

\paragraph{Edge of stability.} \citet{cohen2020gradient} observes a two-stage process in gradient descent, where the first is loss curvature grows until the sharpness touches the bound $2/\eta$, and the second is the curvature hovers around the bound and training loss still decreases in a macro view regardless of local instability. \citet{gilmer2021loss} reports similar observations in stochastic gradient descent and conducts comprehensive experiments of loss sharpness on learning rates, architecture choices and initialization. \citet{lewkowycz2020large} argues that gradient descent would ``catapult'' into a flatter region if loss landscape around initialization is too sharp.

Some concurrent works~\citep{ahn2022understanding, ma2022multiscale, arora2022understanding, damian2022self} are also theoretically investigating the edge of stability. 
\citet{ahn2022understanding} suggests that unstable convergence happens when the loss landscape of neural networks forms a local forward-invariant set near the minima due to some ingredients, such as $\tanh$ as the nonlinear activation.
\citet{ma2022multiscale} empirically observes a multi-scale structure of loss landscape and, with it as an assumption, shows that gradient descent with different learning rates may stay in different levels.
\citet{arora2022understanding} shows the training provably enters the edge of stability with modified gradient descent or modified loss, and then its associated flow goes to flat regions. Under mild conditions, \citet{damian2022self} proves that GD beyond EoS follows an optimization trajectory subjected to a sharpness constraint so that a flatter region is found.
Note that our learning rate is strictly larger than that of \citet{damian2022self} so that their proposed manifold does not exists in our settings, as discussed in Section~\ref{sec:implication}.

\paragraph{Implicit regularization.} Due to its theoretical closeness to gradient descent with a small learning rate, gradient flow is a common setting to study the training behavior of neural networks. \citet{barrett2020implicit} suggests that gradient descent is closer to gradient flow with an additional term regularizing the norm of gradients. Through analysing the numerical error of Euler's method, \citet{elkabetz2021continuous} provides theoretical guarantees of a small gap depending on the convexity along the training trajectory. Neither of them fits in the case of our interest, because it is hard to track the parametric gap when $\eta>1/\lambda$. For instance, in a quadratic function, the trajectory jumps between the two sides once $\eta>1/\lambda$. \citet{damian2021label} shows that SGD with label noise is implicitly subjected to a regularizer penalizing sharp minimizers but the learning rate is constraint strictly below the edge of stability threshold. 

\paragraph{Balancing effect.} \citet{du2018algorithmic} proves that gradient flow automatically preserves the norms' differences between different layers of a deep homogeneous network. \citep{ye2021global} shows that gradient descent on matrix factorization with a constant small learning rate still enjoys the auto-balancing property. Also in matrix factorization, \citet{wang2021large} proves that gradient descent with a relatively large learning rate leads to a solution with a more balanced (perhaps not perfectly balanced) solution while the initialization can be in-balanced. In a similar spirit, we extend their finding to a larger learning rate, with which the perfect balance may be achieved in our setting. We estimate our learning rate is {strictly larger than} theirs~\citep{wang2021large}, where they show GD with large learning rates converges to a flat region in the interpolation manifold while the flat region w.r.t. our larger learing rate does not exists so GD is forced to wander around the flattest minima. Note that the implication of balancing effect is to get close to a flatter solution in the global minimum manifold, which may help improve generalization in some common arguments in the community.

\paragraph{Learning a single neuron.} \citet{yehudai2020learning} studies necessary conditions on both the distribution and activation functions to guarantee a one-layer single student neuron aligning with the teacher neuron under gradient descent, SGD and gradient flow. \citet{vardi2021learning} extends the investigation into a neuron with a bias term. \citet{vardi2021implicit} empirically studies the training dynamics of a two-layer single neuron, focusing on its implicit bias. In this work, we present a convergence analysis of a two-layer single-neuron ReLU network trained with population loss in a large learning rate beyond the edge of stability.




\section{Problem Setup}

We consider a differentiable objective function $f(\theta)$ with $\theta \in \mathbb{R}^d$, and the GD algorithm from (\ref{eq:GD}). 

\begin{defi}
A differentiable function $f$ is $L$-gradient Lipschitz if 
\begin{align}
    \norm{\nabla f(\theta_1)-\nabla f(\theta_2)}\le L\norm{\theta_1 - \theta_2}, ~~~~\forall~ \theta_1,\theta_2.
\end{align}
\end{defi}

The above definition is equivalent to saying that the spectral norm of the Hessian is bounded by $L$, or the \textit{local curvature} at each point is bounded by $L$.
Then $\eta$ needs to be bounded by $1/L$ in GD so that it is guaranteed to visit an approximate first-order stationary point~\citep{nesterov1998introductory}. The perturbed GD requires $\eta=1/L$ to visit an approximate second-order stationary point~\citep{jin2021nonconvex}, and stochastic variants share similar assumptions~\citep{ghadimi2013stochastic, jin2021nonconvex}.

However, in practice, such an assumption may be violated, or even impossible to satisfy when $\| \nabla^2 f \|$ is not uniformly bounded. \citet{cohen2020gradient} observes that, with learning rate $\eta$ fixed, the largest eigenvalue $\lambda_1$ of the loss Hessian of a neural network is below $2/\eta$ at initialization, but it grows above the threshold along training. Such a phenomena is more obvious when the network is deeper or narrower. This reveals the non-smooth nature of the loss landscape of neural networks.

Furthermore, another observation~from \citet{cohen2020gradient} is that once $\lambda_1\ge 2/\eta$, the training loss stops the monotone decreasing. This is not surprising because GD would diverge in a quadratic function with such a large curvature. However, despite of local instability, the training loss still decreases in a longer range of steps, during which the local curvature stays around $2/\eta$. A further phenomena is that, when GD is at the edge of stability, if the learning rate suddenly changes to a smaller value $\eta_s<\eta$, then the local curvature quickly grows to $2/\eta_s$ --- indicating the ability to `manipulate' the local curvature by adjusting the learning rate.

Besides the analysis of GD, the local curvature itself has also received a lot of attention. Due to the nature of over-parameterization in modern neural networks, the global minimizers of the objective $f$ form a manifold of solutions. There have been active directions to understand the \emph{implicit bias} of GD methods, namely where do they converge to in the manifold, and why some points in the manifold are more preferable than others. For the former question, it is believed that (stochastic) GD prefers flatter minima~\citep{barrett2020implicit, smith2021origin, damian2021label, ma2021sobolev}. For the latter, flatter minima brings better generalization~\citep{hochreiter1997flat,li2018visualizing,keskar2016large,ma2021sobolev,ding2022flat}. It would be meaningful if flatter minima could be obtained via GD with a large learning rate.

More specifically, it has been shown that the eigenvalues of the hessian of a deep homogeneous network could be manipulated to infinity via rescaling the weights of each layer~\citep{elkabetz2021continuous}. Fortunately, gradient flow preserves the difference of norms across layers along the training~\citep{du2018algorithmic}. As a result, a balanced initialization induces balanced convergence, while GD would break this balancing effect due to finite learning rate. However, recently it has been observed that GD with large learning rates enjoys a balancing effect ~\citep{wang2021large}, where it converges to a (not perfect) balanced result despite of imbalanced initialization.

Motivated by the connections of optimization, loss landscape and generalization, we would like to understand the training behavior of gradient descent with a large learning rate, from low-dimensional to representative models.






\section{Stable oscillation on 1-D functions: fixed point of two-step update}



In this section, we provide conditions of existence of fixed points of two-step GD on generic 1-D functions, which are on the third or higher derivatives at the local minima (Theorem~\ref{thm:1dlocal} and Lemma~\ref{lem:1dhigher}). More specifically, in the regression setting, these local conditions allow many differentiable non-linear activation functions to the base model (Prop~\ref{prop:l2loss}), and a composition rule is 
established to build complicated base models with simple base models (Prop~\ref{prop:comp}).

Within the framework of Theorem~\ref{thm:1dlocal}, we identify a specific 1-D function to investigate more:
we show the convergence to the fixed points (Theorem~\ref{thm:1dglobal}), along with its 2-D extension in Prop~\ref{prop:xy}, serving as the foundation of nonlinear (Section~\ref{sec:one_neuron}) and high-dimensional (Section \ref{sec:new_mf}) cases. Empirical verification of all theorems are provided in Appendix~\ref{app:add_exps}.

\subsection{Existence of fixed points}
\begin{defi}
    (Period-2 stable oscillation and fixed point of two-step update $F^2_\eta$.) Consider GD on a function $f$ in domain $\Omega$. Denote the update rule of GD as $F_\eta(x)$ for $x\in\Omega$ with learning rate $\eta$. A \textbf{period-2 stable oscillation} is $\exists~x\in\Omega$ such that $F_\eta(F_\eta(x))=x$ and $x$ is not a minima of $f$. Equivalently speaking, $\exists~x\in\Omega$ is a \textbf{fixed point} of the two-step update $F^2_\eta(\cdot)\triangleq F_\eta(F_\eta(\cdot))$.
\end{defi}

\begin{remark}
    It is obvious that fixed points of $F^2_\eta$ exist in pairs by the nature of period-2 oscillation.
\end{remark}

We initiate our analysis of existence of fixed point of $F^2_\eta$ in 1-D. Starting from a condition on general 1-D functions, we look into several specific 1-D functions to verify our arguments. 
Then, focusing on a function in the form of $f(x)=(x^2-\mu)^2$, we present the convergence analysis as a foundation for the following discussions. Furthermore, to shed light on the multi-layer setting, we propose a balancing effect on a 2-D function to make a connection to the 1-D analysis.

\label{sec:1d}

\paragraph{General 1-D functions.} Consider a 1-D function $f(x)$ with a learnable parameter $x\in\mathbb{R}$. The parameter updates following GD with the learning rate $\eta$ as
\begin{align}
    x\tt = F_\eta(x\t) \coloneqq x\t - \eta f'(x\t).
\end{align}
Assuming $f$ is differentiable and all derivatives are bounded, the function value in the next step can be approximated by
\begin{align*}
    f(x\tt) 
    = f(x\t) - \eta [f'(x\t)]^2\Big(1- \frac{\eta}{2}f''(x\t)\Big)
    + o((x\tt-x\t)^2).&
\end{align*}
If $\eta<2/f''(x\t)$, this approximation reveals that the function monotonically decreases for each step of GD, ignoring higher terms. Such an assumption would guarantee the convergence to a global minimum in a convex function. However, our interest is what happens if $\eta>2/f''(x)$. For instance, if $f$ is a quadratic function, the second-order derivative $f''$ is constant. As a result, once $\eta>2/f''$, GD diverges except when being initialized at the optimum. However, when trained with a large learning rate $\eta>2/f''(\Bar{x})$, there is still some hope for a function to stay around a local minima $\Bar{x}$, as stated in the following theorem.

\begin{theorem}
Consider any 1-D differentiable function $f(x)$ around a local minima $\Bar{x}$, satisfying (i) $f^{(3)}(\Bar{x})\neq 0$, and (ii) $3[f\T{3}]^2-f'' f\T{4}>0$ at $\Bar{x}$. Then, there exists $\epsilon$ with sufficiently small $|\epsilon|$ and $\epsilon\cdot f^{(3)}>0$ such that: for any point $x_0$ between $\Bar{x}$ and $\bar{x} - \epsilon$, there exists a learning rate $\eta$ such that $F^2_\eta(x_0)=x_0$, and
\[
\frac{2}{f''(\Bar{x})}<\eta<\frac{2}{f''(\Bar{x})-\epsilon\cdot f^{(3)}(\Bar{x})}.
\]
\label{thm:1dlocal}
\end{theorem}

\begin{customremark}{1}
    Here obviously we have $\eta>2/f''(\Bar{x})$ beyond EoS. If we take $f''(x_0)\approx f''(\Bar{x}) - \epsilon' f^{(3)}(\Bar{x})$ with $\epsilon'\approx \epsilon$, it holds $\eta<\frac{2}{f''(x_0)}$. Symmetrically, it holds $\frac{2}{f''(F_\eta(x_0))}<\frac{2}{f''(\Bar{x})}$. Hence, $\eta$ upper bounded by the EoS at one point in the period-2 orbit. 
\end{customremark}
\begin{customremark}{2}
    We prove the key condition, $3[f\T{3}]^2-f'' f\T{4}>0$, in the case of \textbf{matrix factorization} around any minima as Theorem~\ref{thm:mf_1d_cond} in Appendix~\ref{sec:app_mf_1d}. Meanwhile, we verify this condition in \textbf{multi-layer networks on MNIST}, as shown in Figure~\ref{fig:add_exp_mnist},~\ref{fig:add_exp_mnist_four_layer},~\ref{fig:add_exp_mnist_five_layer} in Appendix~\ref{app:mnist}.
\end{customremark}

The details of proof are presented in the Appendix~\ref{app:1dlocal}.
As stated in the Theorem~\ref{thm:1dlocal}, we provide a sufficient condition for existence of fixed point of $F^2_\eta$ around a local minima. But still we cannot tell whether or not some functions have it with $f\T{3}(\Bar{x})=0$. For instance, a quadratic function does not satisfy this condition since $f^{(3)} = f^{(4)} \equiv 0$ and it diverges when GD is beyond the edge of stability. But for $f(x)=\sin(x)$ around $\Bar{x}=-\frac{\pi}{2}$ where $f\T{3}(\Bar{x})=0$, it turns out the fixed point exists. Therefore, we extend the argument in Theorem~\ref{thm:1dlocal} to a higher order case in Lemma~\ref{lem:1dhigher}.
As a result, we verify that the sine function does allow stable oscillation as in Corollary~\ref{cor:sin}, because its lowest order of nonzero derivative (except $f''$) at the local minima is $f\T{4}(\Bar{x})<0$. 

\begin{lemma}
    Consider any 1-D differentiable function $f(x)$ around a local minima $\Bar{x}$, satisfying that the lowest order non-zero derivative (except the $f''$) at $\Bar{x}$ is $f\T{k}(\Bar{x})$ with $k\ge 4$. Then, there exists $\epsilon$ with sufficiently small $|\epsilon|$ such that: for any point $x_0$ between $\Bar{x}$ and $\Bar{x}-\epsilon$, and
    \begin{enumerate}
        \item if $k$ is odd and $\epsilon\cdot f\T{k}(\Bar{x})>0, f\T{k+1}(\Bar{x})<0$, 
        then there exists $\eta\in(\frac{2}{f''}, \frac{2}{f''-f\T{k}\epsilon^{k-2}})$,
        \item if $k$ is even and $f\T{k}(\Bar{x})<0$, then there exists $\eta\in(\frac{2}{f''}, \frac{2}{f''+f\T{k}\epsilon^{k-2}})$,
    \end{enumerate}
    such that $F^2_\eta(x_0)=x_0$.
    \label{lem:1dhigher}
    \end{lemma}

The details of proof are presented in the Appendix~\ref{app:1dhigher}.

\paragraph{$L_2$ loss on general 1-D functions.} 
However, we have to admit that the local conditions above are 1) too abstract to directly write down a meaningful function in this family, or 2) too complicated to compute the higher-order derivatives of a given non-trivial function.

Fortunately, both Theorem~\ref{thm:1dlocal} and Lemma~\ref{lem:1dhigher} provide a guarantee that squared-loss on any base function $g$ provably allows stable oscillation once $g$ satisfies some mild conditions, as stated in Prop~\ref{prop:l2loss}. Moreover, we provide a straightforward method to build a more complicated model from two simple base models, as stated in Prop~\ref{prop:comp}.
\begin{prop}
Consider a 1-D function $g(x)$ , and define the loss function $f$ as $f(x) = (g(x)-y)^2$. Assuming (i) $g'$ is not zero when $g(\Bar{x})=y$, (ii) $g'(\Bar{x})g\T{3}(\Bar{x}) < 6[g''(\Bar{x})]^2$, then it satisfies the condition in Theorem~\ref{thm:1dlocal} or Lemma~\ref{lem:1dhigher} have a fixed point of $F^2_\eta$ around $\Bar{x}$. 
\label{prop:l2loss}
\end{prop}
This setup covers a broad family of generic non-linear least squares problems, including the base model $g$ being \textbf{sine, tanh, high-order monomial, exponential, logarithm, sigmoid, softplus, gaussian}, etc. Many of these nonlinear (activation) functions are widely used in empirical or theoretical deep learning, together with the composition rule (Prop~\ref{prop:comp}), shedding light for future analysis of practical models with these as building blocks. 
\begin{prop}[Composition Rule]
    Consider two 1-D functions $p,q$. Assume both $p(x),q(y)$ at $x=\Bar{x}, y=p(\Bar{x})$ satisfies the conditions of $g$ in Prop~\ref{prop:l2loss}. Then $q(p(x))$ also satisfies the conditions to have a fixed point of $F^2_\eta$ around $x=\Bar{x}$.
    \label{prop:comp}
\end{prop}



In Appendix~\ref{app:1dhigher} and \ref{app:l2loss}, we provide the proof details of these settings of $g(x)$ as Corollaries \ref{cor:sin}-\ref{cor:sigmoid}, along with all lemmas and proposition.

After the above discussions on local conditions, a natural question rises up as 
\begin{center}
    \textbf{Q1}: with existence of a fixed point of $F^2_\eta$, can iterative runnings of $F^2_\eta$ converge to it?
\end{center}

With such a question, we are going to present a careful analysis on $g(x)=x^2$.

\subsection{Convergence to fixed points}\label{sec:convergence}

\paragraph{A special 1-D function.}

Consider $f(x)=\frac{1}{4}(x^2-\mu)^2$ with $\mu>0, f\T{3}(\sqrt{\mu})=6\sqrt{\mu}, f''(\sqrt{\mu})=2\mu$. Note that this function is more special to us because it can be viewed as a \textit{symmetric scalar factorization} problem subjected to the squared loss. Later we will leverage it to gain insights for asymmetric initialization, two-layer single-neuron networks and matrix factorization. Before that, we would like to show where it converges to when $\eta>\frac{2}{f''(\sqrt{\mu})}$ as follows.

\begin{theorem}
For $f(x)=\frac{1}{4}(x^2-\mu)^2$, consider GD with $\eta=K\cdot\frac{1}{\mu}$ where $1<K <\sqrt{4.5}-1\approx 1.121$, and initialized on any point $0<x_0<\sqrt{\mu}$. Then it converges to an orbit of period 2, except for a measure-zero initialization where it converges to $\sqrt{\mu}$. More precisely, the period-2 orbit are the solutions $x=\delta_1\in(0,\sqrt{\mu}), x=\delta_2\in(\sqrt{\mu},2\sqrt{\mu})$ of solving $\delta$ in
\begin{align}
    \eta  = \frac{1}{\delta^2\left( \sqrt{\frac{\mu}{\delta^2}-\frac{3}{4}} +\frac{1}{2}\right)}.
    \label{eq:1d_orbit}
\end{align}
\label{thm:1dglobal}
\end{theorem}


The details of proof are presented in the Appendix~\ref{app:1dglobal}. As shown above, Theorem~\ref{thm:1dlocal} and Theorem~\ref{thm:1dglobal} stand in two different levels: Theorem~\ref{thm:1dlocal} restricts the discussion in a local view because of Taylor approximation, while Theorem~\ref{thm:1dglobal} starts from local convergence and then generalizes it into a global view. However, Theorem~\ref{thm:1dlocal} builds a foundation for Theorem~\ref{thm:1dglobal} because the latter would degenerate to the former when $K$ is extremely close to 1.

\paragraph{A special 2-D function.} Similarly, consider a 2-D function $f(x,y)=\frac{1}{2}(xy-\mu)^2$ under different initialization for $x$ and $y$, which we would call ``in-balanced'' initialization. Note that all the global minima in 2-D case form a manifold $\{(x,y)|xy=\mu\}$ while the 1-D case only has two points of global minima. So we need to distinguish all points in the manifold by their sharpness. When $xy=\mu$, the leading eigenvalue of the loss Hessian is $\lambda_1 = (x-y)^2+2\mu$. Hence, in the global minima manifold, the local curvature of each point is larger if its two parameters are more imbalanced. Among all these points, the smallest curvature appears to be $\lambda_1 = 2\mu$ when $x=y=\sqrt{\mu}$. In other words, if the learning rate $\eta>2/2\mu$, all points in the manifold would be too sharp for GD to converge. We would like to investigate the behavior of GD in this case. It turns out the two parameters are driven to a perfect balance although they initialized differently, as follows.

\begin{theorem}
For $f(x,y)=\frac{1}{2}\left(xy-\mu\right)^2$, consider GD with learning rate $\eta=K\cdot\frac{1}{\mu}$. Assume both $x$ and $y$ are always positive during the whole process $\{x_i,y_i\}_{i\ge 0}$. In this process, denote a series of all points with $xy>\mu$ as $\mathcal{P}=\{(x_i,y_i)|x_i y_i>\mu\}_{i\ge0}$. Then $|x-y|$ \red{decays to 0} in $\mathcal{P}$, for any $1<K<1.5$.
\label{thm:xy_diff_decay}
\end{theorem}

Theorem~\ref{thm:xy_diff_decay} shows an effect that the two parameters are squeezed to a single variable, which re-directs to our 1-D analysis in Theorem~\ref{thm:1dglobal}. Therefore, actually both cases converge to the same orbit when $1<K< 1.121$, as stated in Prop~\ref{prop:xy}.

\begin{prop}
    Follow the setting in Theorem~\ref{thm:xy_diff_decay}. Further assume $1<K<\sqrt{4.5}-1\approx 1.121$. Then GD converges to an orbit of period 2. The orbit is formally written as $\{(x=y=\delta_i)|i=1,2\}$, with $\delta_1\in(0,\sqrt{\mu}), \delta_2\in(\sqrt{\mu},2\sqrt{\mu})$  as the solutions of solving $\delta$ in
\begin{align*}
    \eta  = \frac{1}{\delta^2\left( \sqrt{\frac{\mu}{\delta^2}-\frac{3}{4}} +\frac{1}{2}\right)}.
\end{align*}
\label{prop:xy}
\end{prop}

A natural follow-up question is what implications Theorem~\ref{thm:1dglobal} and Prop~\ref{prop:xy} bring, because 1-D and 2-D is far from the practice of neural networks that contain multi-layer structures, nonlinearity and high dimensions. We precisely incorporate two layers and nonlinearity in Section~\ref{sec:one_neuron}, and high dimensions in Section~\ref{sec:new_mf}.

\section{On a two-layer single-neuron homogeneous network}\label{sec:one_neuron}

We denote a two-layer single-neuron network as $f(x;\theta)=v\cdot\sigma(w^\top x)$ where $v\in\R,w\in\R^d$, the set of trained parameters $\theta=(v,w^\top)\in\R^{d+1}$, and the nonlinearity $\sigma$ is ReLU. We will keep such an order in $\theta$ to view it as a vector. The input $x\in\R^d$ is drawn uniformly from a unit sphere $\mathcal{S}^{d-1}$. The parameters are trained by GD subjected to $L_2$ population loss, as 
\begin{align*}
    \theta_{t+1}=\theta_t - \eta\nabla_{\theta} L(\theta_t),~~~~
    L(\theta_t)=\mathbb{E}_{x\in\mathcal{S}^{d-1}} \big(f(x;\theta_t)-y\big)^2.
\end{align*}

We generate labels from a single teacher neuron function, as $y|x = \sigma(\tilde{w}^\top x)$. Hence $\tilde{w}$ is our target neuron to learn. We denote the angle between $w$ and $\tilde{w}$ as $\alpha \ge0$. Note that $\alpha$ is set as non-negative because the loss function is symmetric w.r.t. the angle. Moreover, the rotational symmetry of the population data distribution results in a loss landscape that only depends on $w$ through the angle $\alpha$ and the norm $\|w\|$. 
Indeed, from the definition, we have
\begin{align*}
    \scalemath{0.8}{
    \nabla_{\theta} L
    = \frac{1}{d}
        \begin{bmatrix}
        v\norm{w}_2^2-\frac{\norm{w}}{\pi}\big(\sin{\alpha}+(\pi-\alpha)\cos{\alpha}\big)\norm{\tilde{w}}\\
        v^2w-\frac{v}{\pi}(\pi-\alpha+\frac{1}{2}\sin{2\alpha})\cdot\tilde{w}-\frac{v}{\pi}(-\frac{1}{2}\cos{2\alpha}+\frac{1}{2})\norm{\tilde{w}}\tilde{w}_\perp
        \end{bmatrix}},
\end{align*}
where we denote $\tilde{w}_\perp$ as 
the normalized of $w-\proj_{\tilde{w}} w$.
Consider the Hessian
\begin{align}
    \scalemath{0.8}{
    H \triangleq 
    \begin{bmatrix}
    \partial_v^2 L &  \partial_w\partial_v L \\
    \partial_v\partial_w L &  \partial_w^2 L
    \end{bmatrix}
    \eqtext{if $vw = \tilde{w}$}
    \frac{1}{d}
    \begin{bmatrix}
    \norm{w}^2 & vw^\top \\
    vw & v^2\mathbb{I}
    \end{bmatrix} \in \R^{(d+1)\times(d+1)}}.
\end{align}
Hence, in the global minima manifold where $vw=\tilde{w}$, the eigenvalues of the Hessian are $\lambda_1 = \frac{\norm{w}^2 + v^2}{d}, \lambda_{2\dots d}=\frac{v^2}{d}, \lambda_{d+1}=0$. Therefore, the largest eigenvalue $\lambda_1$ measures the imbalance (\textit{i.e.}, $|\norm{w}-v|$)  between the two layers again as $\lambda_1=\frac{(\norm{w}-v)^2+2\norm{\tilde{w}}}{d}$ similar to the 2-D case in Section~\ref{sec:convergence}. So we would like to investigate where GD converges if $\eta>\frac{2}{2\norm{\tilde{w}}/d}=d/\norm{\tilde{w}}$ that is too large even for the flattest minima. Note that a key difference between the current case and the previous 2-D analysis is that the current one includes a neuron as a vector and a nonlinear ReLU unit.

From the second row of $\nabla_\theta L$, which is $\nabla_w L$, it is clear that updates of $w$ always stay in the plane spanned by $\tilde{w}$ and $w^{(0)}$. Hence, this problem can be simplified to three variables $(v, w_x, w_y)$ with the target neuron $\tilde{w}=[1,0]$. The three variables stand for
\begin{gather*}
    v^{(t)}\coloneqq v^{(t)},~~~ w_x^{(t)}\coloneqq \proj_{\tilde{w}} w^{(t)}, \\
    w_y^{(t)}\coloneqq \proj_{\tilde{w}_\perp} w^{(t)} = \sqrt{\norm{w^{(t)}}^2-(w_x^{(t)})^2}.
\end{gather*}


We keep $w_y$ as nonnegative because the loss $L$ is invariant to its sign and our previous notation $\alpha\ge0$ requires a non-negative $w_y$. 
Then we show that $w_y$ decays to 0 as follows.

\begin{theorem}
In the above setting, consider a teacher neuron $\tilde{w}=[1,0]$ and set the learning rate $\eta=Kd$ with $K\in(1,1.1]$. Initialize the student as $\norm{w^{(0)}}=v^{(0)}\triangleq \epsilon\in(0,0.10]$ and $\inprod{w^{(0)}}{ \tilde{w}}\ge0$. Then, for $t\ge T_1 + 4$, $w_y^{(t)}$ decays as 
\begin{gather*}
    \red{w_y^{(t)} < 0.1\cdot (1-0.030K)^{t-T_1-4},}
    ~~~~~
    T_1 \le \left\lceil\log_{2.56}\frac{1.35}{\pi \beta^2}\right\rceil,
    \\
    ~~~~~~
    \beta = \left(1+\frac{1.1}{\pi}\right)\epsilon.
\end{gather*}
\label{thm:one_neuron}
\end{theorem}
\vspace{-10pt}

The details of proof are presented in the Appendix~\ref{app:one_neuron}.

With the guarantee of $w_y$ decaying in the above theorem, the dynamics of the single-neuron ReLU network follow the convergence of the 2-D case in Section~\ref{sec:convergence}, with a convergence result as follows. 
\begin{prop}
    The single-neuron model in Theorem~\ref{thm:one_neuron} converges to a period-2 orbit where $w_y=0$ and $(v,w_x)\in\gamma_K$ with $\gamma_K=\left\{(\delta_1, \delta_1), (\delta_2, \delta_2)\right\}$. Here $\delta_1\in(0,1), \delta_2\in(1,2)$ are the solutions $\delta$ in
    \begin{align}
        K  = \frac{1}{\delta^2\left( \sqrt{\frac{1}{\delta^2}-\frac{3}{4}} +\frac{1}{2}\right)}.
    \end{align}
\label{prop:single_neuron}
\end{prop}
\begin{remark}
    Actually this convergence is close to the flattest minima because: if the learning rate decays to infinitesimal after sufficient oscillations, then the trajectory walks towards the flattest minima $(v=w_x=1,w_y=0)$. Note that we provide an experiment on \textbf{16-neuron networks} in Appendix~\ref{app:16-neuron}, where GD converges to the period-2 orbit near the flattest minima while being initialized near unbalanced (sharp) minima.
\end{remark}


To summarize, the single-neuron model goes through three phases of training dynamics, with an intialization of the angle $\measuredangle(w,\tilde{w})$ as $\frac{\pi}{2}$ at most. First, the angle decreases monotonically but, due to the growth of norms, the absolute deviation $w_y$ still increases. Meanwhile, the imbalance $v-w_x$ stays in a bounded level. Second, $w_y$ starts to decrease and the parameters fall into a basin within four steps. Third, in the basin, $w_y$ decreases exponentially and, after $w_y$ at a reasonable low level, the model approximately follows the dynamic of the 2-D case and the imbalance $v-w_x$ decreases as well, following Theorem~\ref{thm:xy_diff_decay}. The model converges to a period-2 orbit as in the 1-D case in Theorem~\ref{thm:1dglobal}.

\section{Matrix Factorization and beyond}
\label{sec:new_mf}

In the last two sections, we have presented theoretical results that GD beyond EoS converges to the fixed points of $F^2_\eta$ from initialization that is far away. In this section, we address these follow-up questions, by raising observations in Matrix Factorization and discuss whether existing models can explain our observations or not:

\begin{center}
    \textbf{Q2}: does such a period-2 orbit exist in more complicated settings?\\
    \textbf{Q3}: what does the appropriate model need to cover such oscillation in high-dim problems? \\
    \textbf{Q4}: what will happen if the learning rate grows more?
\end{center}


\subsection{Observations from Matrix Factorization}


Consider a matrix factorization problem, parameterized by learnable weights $\mathbf{Y}\in\R^{d\times d}$, $\mathbf{Z}\in\R^{d\times d}$, and the target matrix is $\mathbf{C}\in\R^{d\times d}$, which is symmetric and positive definite. The loss $L$ is defined as
\begin{align}
    L(\bY,\bZ)=\frac{1}{2}\norm{\bY\bZ^\top - \mathbf{C}}_F^2.
\end{align}
Obviously $\{(\bY,\bZ):\bY\bZ^\top=\mathbf{C}\}$ forms a minimum manifold. 
Although we prove that the necessary 1-D condition holds around minimum as Theorem~\ref{thm:mf_1d_cond} (in Appendix~\ref{app:add_res_mf}), which is analogous to Theorem~\ref{thm:1dlocal}, it is more attracting to investigate GD in high dimensions.
We propose our first observation that Matrix Factorization converges to a period-2 orbit, \textit{i.e.}, fixed points of $F^2_\eta$, as follows.


\begin{obs}[Matrix Factorization with period-2 orbit]
\label{obs:generic_mf}
Consider GD with learning rate $\eta$ satisfying $\eta\sigma_1^2\in(1,1.121)$ and $\eta\left(\sigma_1^2+\sigma_2^2\right)<2$ where $\sigma_1^2,\sigma_2^2$ are the first and second largest eigenvalues of $\bC$. Then, there exists non-measure-zero initialization, from which GD converges to a period-2 orbit in the form of ($i\in\{1,2\}$)
\begin{align*}
    \bY = \rho_i u v^\top + \sum_{j=2}^d \sigma_{y,j} u_{y,j} v_{y,j}^\top, \\
    \bZ = \rho_i u v^\top + \sum_{j=2}^d \sigma_{z,j} u_{z,j} v_{z,j}^\top, \\
    \bY\bZ^\top - \bC = (\rho_i^2-\sigma_1^2) u u^\top,
\end{align*}
where $u$ is the leading eigenvector of $\bC$, $v$ is arbitrary unit vector in $\R^d$, $\{\rho_i\}_{i=1,2}$ are the two positive roots of 
\begin{align}
    \eta\sigma_1^2 = \frac{1}{\rho^2 \left(\sqrt{\frac{1}{\rho^2}-\frac{3}{4}}+\frac{1}{2} \right)},
    \label{eq:rho}
\end{align}
and the decompositions of $\bY,\bZ$ are SVD.
\end{obs}

\begin{remark}
    At any minimizer $(\bX,\bY)$ satisfying $\bX\bY^\top=\bC$, the largest eigenvalue of loss Hessian w.r.t. parameters is $\sigma_{\max}(\bX)^2+\sigma_{\max}(\bY)^2$. Consequently, the flattest minima has sharpness as $2\sigma_1^2$, because $\sigma_1^2=\lambda_{\max}(\bC)\le \sigma_{\max}(\bX)\sigma_{\max}(\bY)\le 0.5\left(\sigma_{\max}(\bX)^2+\sigma_{\max}(\bY)^2\right)$.
\end{remark}


To our knowledge, this observation is beyond all previous results. \citet{damian2022self} tracks the trajectory's projection onto the manifold with sharpness $<\nicefrac{2}{\eta}$. \citet{wang2021large} proposes that GD in a sharper region (sharpness$>\nicefrac{2}{\eta}$) converges to flatter region (sharpness$<\nicefrac{2}{\eta}$) for matrix factorization problem. But such a manifold (or flatter region) containing any minimizer does not exist in our setting because $\eta\sigma_1^2>1$ makes the flattest minima sharper than $2/\eta$, which means the probability of converging to a stationary point is zero~\citep{ahn2022understanding}.

However, it is difficult to prove Observation~\ref{obs:generic_mf} rigorously. Meanwhile, general initialization cannot illustrate well the phenomena that GD walks to flatter minima from a sharper one. Therefore, we provide an observation of a limited version of matrix factorization, called \textit{quasi-symmetric}, along with sufficient intuition on its dynamics and careful discussion on what is remaining to prove it.

\begin{defi}[Quasi-symmetric Matrix Factorization]
Given a symmetric and positive definite target matrix $\bC\triangleq\bX_0\bX_0^\top$, where $\bX_0=\R^{d\times d}$. Quasi-symmetric MF is solving the factorization problem with initialization near an unbalanced minima, where the minima is $(\alpha \bX_0, \nicefrac{1}{\alpha} \bX_0)$ with $\alpha\neq 1$.
\end{defi}


\begin{obs}[Quasi-symmetric Matrix Factorization with period-2 orbit]
    \label{obs:quasi_mf}
    Consider the above quasi-symmetric matrix factorization with learning rate $\eta \in (\nicefrac{1}{\sigma_1^2}, \nicefrac{1.121}{\sigma_1^2})$. Consider a minima $(\bY_0=\alpha\bX_0, \bZ_0=\nicefrac{1}{\alpha}\mathbf{X}_0), \alpha>0$. The initialization is around the minimum, as $\mathbf{Y}_1 = \bY_0 +\Delta \mathbf{Y}_1, \mathbf{Z}_1 = \bZ_0 + \Delta \mathbf{Z}_1$. When
    \begin{align}
        \eta\cdot\max\left\{
        (\frac{\sigma_1^2}{\alpha^2}+\sigma_2^2\alpha^2, \frac{\sigma_2^2}{\alpha^2}+\sigma_1^2\alpha^2)
        \right\}
        \le 2
        \label{eq:quasi_sig2}
    \end{align}
    GD would converge to a period-2 orbit $\gamma_{\eta}$ approximately with error in $\O(\epsilon)$, formally written as, $(i=1,2)$ 
    \begin{gather*}
        (\mathbf{Y}_t, \mathbf{Z}_t) \rightarrow \gamma_{\eta} + (\Delta\mathbf{Y}, \Delta\mathbf{Z}), 
        ~~~~~~\norm{\Delta\mathbf{Y}}, \norm{\Delta\mathbf{Z}}=\O(\epsilon), \\\
        \gamma_{\eta} = \{\left(\mathbf{Y}_0 + \left(\rho_i-\alpha\right) \sigma_1 u_1 v_1^\top , \mathbf{Z}_0 + \left(\rho_i-\nicefrac{1}{\alpha}\right) \sigma_1 u_1 v_1^\top\right)\},
    \end{gather*}
    where $\rho_1\in (1,2),\rho_2\in(0,1)$ are the same as in Eq.(\ref{eq:rho})
\end{obs}

\begin{remark} The intuition on the dynamics in Observation~\ref{obs:quasi_mf} is provided in Appendix~\ref{sec:app_mf_quasisym}, along with a discussion on what is missing for rigorous proof for future development. Without loss of generality, assume $\bX_0=\text{diag}([\sigma_1,\sigma_2,\dots,\sigma_d])\in\R^{d\times d}$, where $(\bX_0)_{i,i}=\sigma_i$ and $0$ in all other entries. Intuitively, the dynamics of the system is following
    \begin{gather*}
        \scalemath{0.78}{
        \bY=
            \left[ \begin{array}{c|c}
               \alpha\sigma_1 & \mathbf{0} \\
               \midrule
               \mathbf{0} & \text{diag}([\alpha\sigma_i]_{i=2}^d) \\
            \end{array}
            \right]+ O(\epsilon)
            \rightarrow
            \left[ \begin{array}{c|c}
                \rho_i & \mathbf{0} \\
                \midrule
                \mathbf{0} & \text{diag}([\alpha\sigma_i]_{i=2}^d) \\
             \end{array}
             \right]+ O(\epsilon),}
              \\
        \scalemath{0.78}{
        \bZ=
            \left[ \begin{array}{c|c}
               \alpha\sigma_1 & \mathbf{0} \\
               \midrule
               \mathbf{0} & \text{diag}([\nicefrac{\sigma_i}{\alpha}]_{i=2}^d) \\
            \end{array}
            \right]+ O(\epsilon) 
            \rightarrow
            \left[ \begin{array}{c|c}
                \rho_i & \mathbf{0} \\
                \midrule
                \mathbf{0} & \text{diag}([\nicefrac{\sigma_i}{\alpha}]_{i=2}^d) \\
             \end{array}
             \right]+ O(\epsilon), }
             \\
        \scalemath{0.78}{
        \bY\bZ^\top =
            \left[ \begin{array}{c|c}
                \sigma_1^2 & \mathbf{0} \\
                \midrule
                \mathbf{0} & \text{diag}([\sigma_i^2]_{i=2}^d) \\
            \end{array}
            \right]+ O(\epsilon) 
            \rightarrow
            \left[ \begin{array}{c|c}
                \rho_i^2 & \mathbf{0} \\
                \midrule
                \mathbf{0} & \text{diag}([\sigma_i^2]_{i=2}^d) \\
             \end{array}
             \right].
        }
    \end{gather*}
Note that the top singular values of $\bY,\bZ$ are always the same in the orbit although it is unbalanced at initialization. A benefit of this is that, if $\eta$ decays below $\nicefrac{1}{\sigma_1^2}$ after reaching the orbit, it would converge to $\bY,\bZ$ with same top singular value $\sigma_1$, satisfying $\bY\bZ^\top=\bC$.
\end{remark}


\textbf{How tight are Observation~\ref{obs:generic_mf} and~\ref{obs:quasi_mf}?} There are two aspects we would like to address: $\eta\sigma_1^2$ and $\eta\sigma_2^2$. The former $\eta\sigma_1^2$ is a natural constraint because it is necessary to carefully set its upper bound in 1-D analysis to contain the oscillation in some finite level set. However, the second $\eta\sigma_2^2$ is novel (and tight) to our knowledge, which is respectively $\eta(\sigma_1^2+\sigma_2^2)<2$ in Observation~\ref{obs:generic_mf} and $\eta\cdot\left(\nicefrac{\sigma_1^2}{\alpha^2}+\sigma_2^2\alpha^2\right)<2$ in Observation~\ref{obs:quasi_mf}. The tightness of this bound is verified in Figure~\ref{fig:sig1_vs_sig2}, where it approximates the linearity of the empirical boundary between \crr{infinite} and \cbb{finite} well when $\eta\sigma_1^2>1$ slightly. Furthermore, although we do not prefer asserting too much beyond our theorems, the linear trend between $\eta\sigma_1^2$ and $\eta\sigma_2^2$ keeps well when $\eta\sigma_1^2$ goes beyond $1.121$ for a long range. Intuitively, We gain the insight of this bound from the analysis of Observation~\ref{obs:quasi_mf} in Appendix~\ref{sec:app_mf_quasisym}. More precisely, it appears in Eq.(\ref{eq:quasi_q_lambda}) to guarantee a transition matrix to be semi-convergent, whose largest absolute eigenvalue is no larger than 1.

\begin{figure}[ht]
    \centering
    \includegraphics[width=0.49\linewidth]{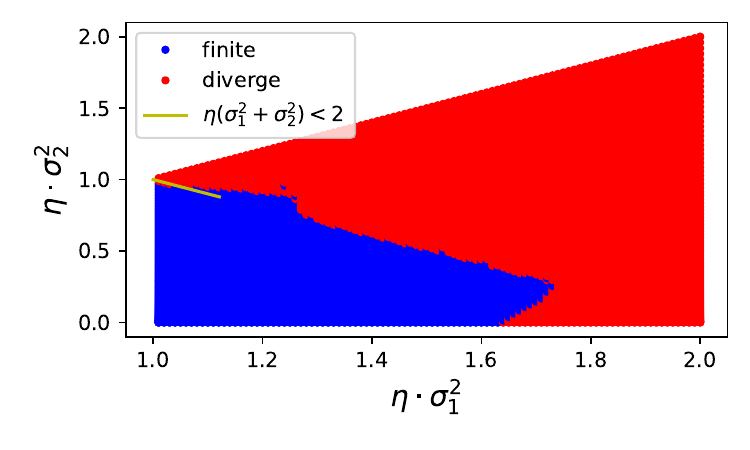}
    \includegraphics[width=0.49\linewidth]{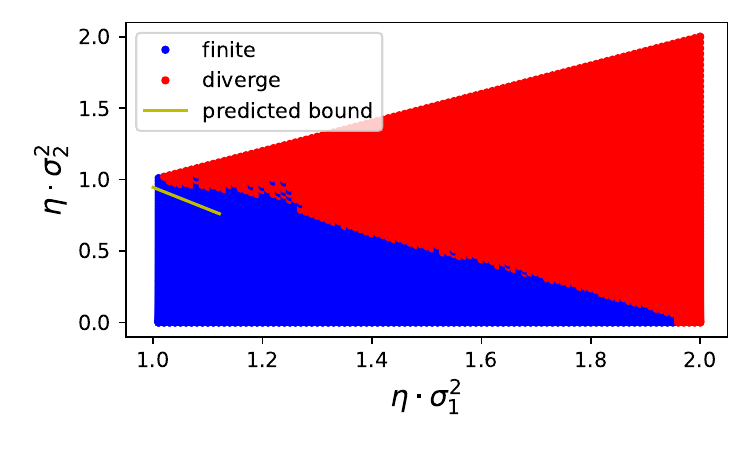}\\
    \hspace{80pt}
    (a) Generic init \quad\quad\hfill 
    (b) Quasi-sym init ($\alpha=0.9$)
    \hspace{40pt} \hfill
    \\
    \includegraphics[width=0.49\linewidth]{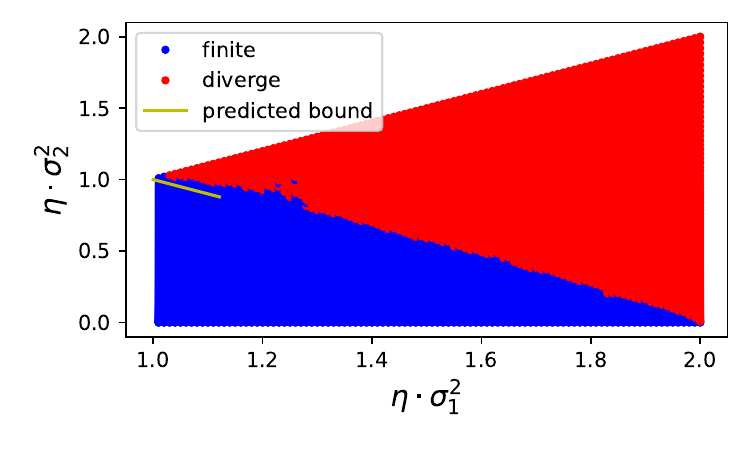}\\
    (c) Symmetric init ($\alpha=1$ in Quasi case)
    \caption{\textbf{Matrix Factorization}: grid search of $\eta\sigma_1^2$ \text{v.s.} $\eta\sigma_2^2$ on whether GD diverges or not. (a) Generic initialization: it verifies the condition $\eta\left(\sigma_1^2+\sigma_2^2\right)<2$. (b-c) Quasi-symmetric initialization: it verifies the predicted bound $\eta\cdot\left(\nicefrac{\sigma_1^2}{\alpha^2}+\sigma_2^2\alpha^2\right)<2$ in Eq.(\ref{eq:quasi_sig2}) as a sufficient condition.}
    \label{fig:sig1_vs_sig2}
\end{figure}





\textbf{Is there any other phenomena beyond period-2 orbit when $\eta$ grows larger?} 
The answer is yes.
We conduct experiments of matrix factorization with generic initialization with different $\eta$'s, as shown in Figure~\ref{fig:sigm_vs_eta}. It turns out when $\eta\sigma_1^2\in(1,1.23)$, it converges to period-2 orbit. When $\eta\sigma_1^2\in(1.23,1.28)$, it converges to a period-4 orbit, although the period-2 orbit still exists once $\eta\sigma_1^2<1.5$ as shown in Eq.(\ref{eq:eta_exist_bound}) (because the existence cannot guarantee convergence, and even local convergence does not hold). When $\eta\sigma_1^2>1.28$, it is rather chaotic. However, during most of these, the balancing effect holds, \textit{i.e.}, $\sigma_{\max}(\bY)=\sigma_{\max}(\bZ)$.

\begin{figure}[ht]
    \centering
    \includegraphics[width=0.5\linewidth]{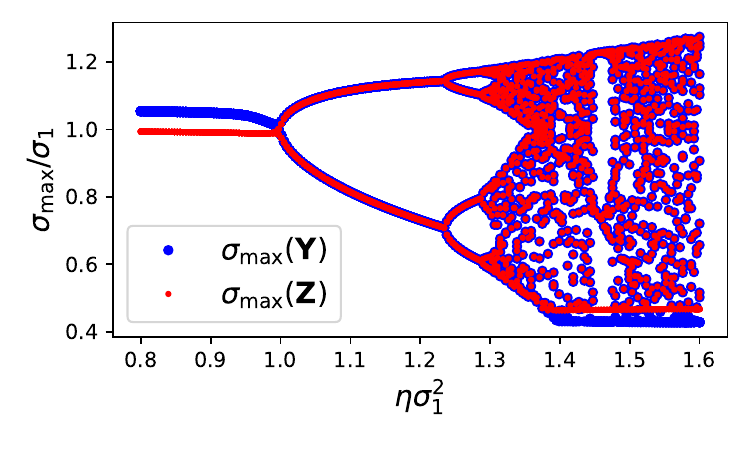}
    \caption{\textbf{Matrix Factorization}: $\sigma_{\max}(\bY), \sigma_{\max}(\bZ)$ for different $\eta$'s. For each $\eta$, the last 10 iterations are sampled for report, due to periodic and chaotic phenomenon. \textbf{Observations}: (1) when $\eta\sigma_1^2\in(1,1.38)$, all cases have $\sigma_{\max}(\bY)= \sigma_{\max}(\bZ)$; (2) when $\eta\sigma_1^2\in(1,1.23)$, it converges to a period-2 orbit; (3) when $\eta\sigma_1^2\in(1.23, 1.28)$, it converges to a period-4 orbit; (4) when $\eta\sigma_1^2>1.28$, it is rather chaotic; (5) when $\eta\sigma_1^2<1$, there is no oscillation.}
    \label{fig:sigm_vs_eta}
\end{figure}

\subsection{Implications for more complicated settings}
\label{sec:implication}

\paragraph{Existing models from \citet{ma2022multiscale} and \citet{damian2022self}.} \citet{ma2022multiscale} proposes a decomposition of high-dimensional functions into separable functions in eigendirections, in the form of 
\begin{align}
    \label{eq:ma_decomp}
    f(\theta)=f_1(p_1^\top \theta) + f_2(p_2^\top \theta) + \cdots + f_d(p_d^\top \theta),
\end{align}
where $\{p_i\in\R^{d}\}$ is an orthogonal basis of $\R^d$, $\theta\in\R^d$ is the parameter and each $f_i$ is a function that allows stable oscillation. Within such a framework, all $p_i^\top x$ can stably oscillation since the dynamics is separable in each eigendirection. However, this framework cannot explain the dynamics of matrix factorization, because our experiments in Figure~\ref{fig:sig1_vs_sig2} have shown that GD will blow up once $\eta\sigma_2>1$, which means the eigen-directions associated with $\sigma_1^2$ and $\sigma_2^2$ cannot be disentangled in this case.

\citet{damian2022self} proposes to track the trajectory's projection onto manifold $\mathcal{M}=\{\theta:\lambda(\theta)<\nicefrac{2}{\eta}, \nabla L(\theta)\cdot u(\theta)=0\}$, where $\lambda(\theta)$ and $u(\theta)$ are the leading eigenvalue and eigenvector of Hessian of loss $L$. However, such a manifold does not exist in the 2-D case we have studied in Section~\ref{sec:1d} because our setting is strictly beyond EoS. Furthermore, in high-order cases, such a manifold containing any minimizer does not exist (Proposition~\ref{prop:xn_prod}).

\begin{prop}
    For $L(x,y)=\nicefrac{1}{2}(xy-1)^2$ with $\eta>1$ on $\{x>0, y>0\}$, such a manifold $\mathcal{M}$ does not exist. 
\end{prop}

\begin{prop}
    For $L(x,y)=\nicefrac{1}{2}(xy-1)^2$ with $\eta<1$ on $\{x>0, y>0\}$, $\mathcal{M}=\{(x,y): xy=1, x+y<\sqrt{2+\nicefrac{2}{\eta}}\}$. 
    \label{prop:x2_le1}
\end{prop}

\begin{prop}
    For $L(\{x_i\}_{i=1}^n)=\frac{1}{n}(\prod_{i=1}^n x_i-1)^2$ with $\eta>1$ on $\{x_i>0,~\forall i\}$, such a manifold $\mathcal{M}$ containing any minimizer does not exist.
    \label{prop:xn_prod}
\end{prop}

Moreover, although $\mathcal{M}$ exists when $\eta<1$ (Proposition~\ref{prop:x2_le1}), the size of $\mathcal{M}$ is limited, which means the trajectory's projection onto it stays unchanged in the early steps, although the trajectory is moving efficiently from sharper region to flatter region, as shown in Figure~\ref{fig:xy_manifold}(b).

\begin{figure}[ht]
    \centering
    \includegraphics[width=0.4\linewidth]{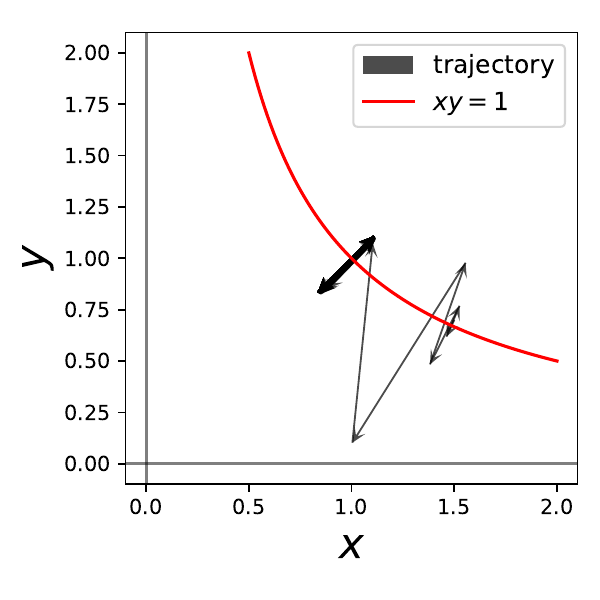}
    \includegraphics[width=0.4\linewidth]{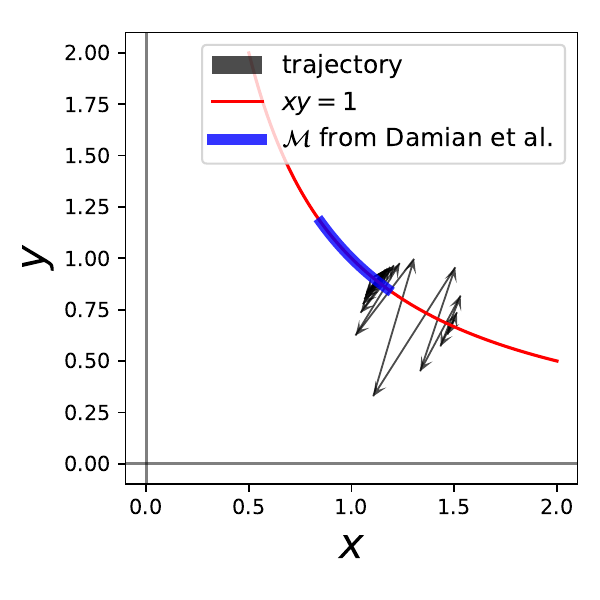}
    \\
    \hfill
    (a) $\eta=1.08$
    \hfill
    (b) $\eta=0.95$
    \hspace{80pt}
    \hfill
    \caption{Trajectories of minimizing $L(x,y)=\nicefrac{1}{2}(xy-1)^2$ with $\eta=1.08, 0.95$. For $\eta=1.08$, the manifold $\mathcal{M}$ proposed by \citet{damian2022self} does not exist. For $\eta=0.95$, the manifold $\mathcal{M}$ exists, but the projection onto it does not change for the first few steps.}
    \label{fig:xy_manifold}
\end{figure}

\textbf{Two candidate models.}
From the above discussion, we would like to raise two candidate models to contain the observations from matrix factorization, based on the models proposed in \citet{damian2022self} and \citet{ma2022multiscale}.

Following \citet{damian2022self}, we would like to propose

\begin{defi}[Projection onto manifold]
    $\mathcal{M}_c=\{\theta:\lambda_2(\theta)<\nicefrac{2}{\eta}, \nabla L(\theta)\cdot u(\theta)=0\}$, where $\lambda_2(\theta)$ and $u(\theta)$ are the second largest eigenvalue and the leading eigenvector of Hessian of loss $L$.
\end{defi}
The motivation of $\mathcal{M}_c$ is to contain points that have the leading eigenvalue greater than 1. For example, in the case of $\nicefrac{1}{2}(xy-1)^2$, it is $\mathcal{M}_c=\{(x,y):xy=1\}$ allowing to track the trajectory walking from sharper region to flatter region. Instead of constraining $\lambda<\nicefrac{2}{\eta}$, we set $\lambda_2<\nicefrac{2}{\eta}$ to make it compatible with our observations in matrix factorization.

The gap between \citet{ma2022multiscale} and observations from matrix factorization is that they assume the orthogonal decomposition of the loss function. However, even in the simplest setting of matrix factorization, this assumption does not hold.
Taking a symmetric matrix factorization as an example, we have
\begin{align}
        L(\bX) &=
        \frac{1}{4}
        \norm{
            \bX\bX^\top
            -
            \begin{bmatrix}
                1 & 0 \\
                0 & 1
            \end{bmatrix}
        }_F^2
        \nonumber
        \\
        &=
        \frac{1}{4}
        \bigg(
            \left(\norm{\bX_{0,:}}^2-1\right)^2
            +
            \left(\norm{\bX_{1,:}}^2-1\right)^2
            + 
            2\left(\langle\bX_{0,:}, \bX_{1,:}\rangle\right)^2
        \bigg),
\end{align}
where the first two terms in the last line are $f_i(p_i^\top x)$ in Eq.(\ref{eq:ma_decomp}) since the included are orthogonal to each other. However, the last term $\langle\bX_{0,:}, \bX_{1,:}\rangle$ breaks the separability in the decomposition. Meanwhile, $\left(\langle\bX_{0,:}, \bX_{1,:}\rangle\right)^2$ is implicitly $\left(\langle\bX_{0,:}, \bX_{1,:}\rangle - 0\right)^2$, because $\bX_{0,:}, \bX_{1,:}\in\R^2$ are expressive enough to form an orthogonal pair satisfying the constraints of norms.

In a similar spirit, we propose an extensive model of Eq.(\ref{eq:ma_decomp})~\citep{ma2022multiscale} as follows
\begin{align}
    \label{eq:new_decomp}
    f(\theta)=\sum_{i=1}^d g_i(p_i^\top \theta;a_i)  + \sum_{(i,j)\in\mathcal{E}}^d h_{ij}(p_i^\top \theta, p_j^\top \theta; b_{ij}),
\end{align}
where $g_i, h_{ij}$ are functions allowing stable oscillation parameterized by $a_i, b_{ij}$, $\{p_i\}_{i=1}^d$ are orthogonal basis of $\R^d$ and $\mathcal{E}\subset [d]\times[d]$ is a selected subset of tuples. A simple but effective example to imitate matrix factorization is $g_i(x;a)\triangleq (\norm{x}^2-a)^2$ and $h_{ij}(x,y|b)\triangleq (\langle x,y\rangle - b)^2$ and $\mathcal{E}=[d]\times[d]$. Intuitively, such a model with larger $|\mathcal{E}|$ allows fewer eigenvalues of Hessian to go beyond $\nicefrac{2}{\eta}$. Conversely, if $\mathcal{E}=\emptyset$, it allows all eigenvalues beyond $\nicefrac{2}{\eta}$, which degenerates to Eq.(\ref{eq:ma_decomp})~\citep{ma2022multiscale}.

\section{Conclusions}
In this work, we investigate gradient descent with a large step size that crosses the threshold of local stability, via investigating convergence of two-step updates instead of convergence of one-step updates.
In the low dimensional setting, we provide conditions on high-order derivatives that guarantees the existence of fixed points of two-step updates.
For a two-layer single-neuron ReLU network, we prove its convergence to align with the teacher neuron under population loss. 
For matrix factorization, we prove that the necessary 1-D condition holds around any minima. 
We provide novel observations of its convergence to period-2 orbit with comprehensive theoretical intuition of the dynamics.
Finally, we extend previous works by proposing two models with observations in matrix factorization compatible for future analysis.



\section*{Acknowledgements}
We are grateful to Alex Damian, Zhengdao Chen, Zizhou Huang, Yifang Chen and Kaifeng Lyu for helpful conversations. 
This work was partially supported by the Alfred P. Sloan Foundation, NSF RI-1816753, NSF CAREER CIF 1845360, NSF CHS-1901091, Capital One and Samsung Electronics. This research also received support by the generosity of Eric and Wendy Schmidt by recommendation of the Schmidt Futures program.

\bibliography{ref}
\bibliographystyle{plain}

\setcounter{tocdepth}{2}

\newpage

\onecolumn
\appendix
\tableofcontents


\section{Additional Results}
\label{app:addres}

\subsection{On a 2-D function}\label{sec:2d}


Similar to $f(x)=\frac{1}{4}(x^2-\mu)^2$, consider a 2-D function $f(x,y)=\frac{1}{2}(xy-\mu)^2$. Apparently, if $x$ and $y$ initialize as the same, then $(x\t,y\t)$ would always align with the 1-D case from the same initialization. Therefore, it is significant to analyze this problem under different initialization for $x$ and $y$, which we would call ``in-balanced'' initialization. Meanwhile, another giant difference is that all the global minima in 2-D case form a manifold $\{(x,y)|xy=\mu\}$ while the 1-D case only has two points of global minima. It would be great if we could understand which points in the global minima manifold, or in the whole parameter space, are preferable by GD. 

Note that reweighting the two parameters would manipulate the curvature to infinity as in~\citep{elkabetz2021continuous}, 
so the inbalance strongly affects the local curvature.
Viewing $f(x)$ as a symmetric scalar factorization problem, we treat $f(x,y)$ as asymmetric scalar factorization. The update rule of GD is 
\begin{align}
    x\tt \coloneqq x\t - \eta (x\t y\t-\mu)y\t,~~~~
    y\tt \coloneqq y\t - \eta (x\t y\t-\mu)x\t.
\end{align}
Consider the Hessian as
\begin{align}
    H\triangleq 
    \begin{bmatrix}
    \partial_x^2 f &  \partial_y\partial_x f \\
    \partial_x\partial_y f &  \partial_y^2 f
    \end{bmatrix}
    =
    \begin{bmatrix}
    y^2 & 2xy-\mu \\
    2xy-\mu & x^2
    \end{bmatrix}.\label{eq:2dhessian}
\end{align}
When $xy=\mu$, the eigenvalues of $H$ are $\lambda_1 = x^2+y^2, \lambda_2 = 0$. Note that $\lambda_1 = (x-y)^2+2\mu$. Hence, in the global minima manifold, the local curvature of each point is larger if its two parameters are more inbalanced. Among all these points, the smallest curvature appears to be $\lambda_1 = 2\mu$ when $x=y=\sqrt{\mu}$. In other words, if the learning rate $\eta>2/2\mu$, all points in the manifold would be too sharp for GD to converge. We would like to investigate the behavior of GD in this case. It turns out the two parameters are driven to a perfect balance although they initialized differently, as follows.

\begin{theorem}[Restatement of Theorem~\ref{thm:xy_diff_decay}]
For $f(x,y)=\frac{1}{2}\left(xy-\mu\right)^2$, consider GD with learning rate $\eta=K\cdot\frac{1}{\mu}$. Assume both $x$ and $y$ are always positive during the whole process $\{x_i,y_i\}_{i\ge 0}$. In this process, denote a series of all points with $xy>\mu$ as $\mathcal{P}=\{(x_i,y_i)|x_i y_i>\mu\}$. Then $|x-y|$ \red{decays to 0} in $\mathcal{P}$, for any $1<K<1.5$.
\label{thm:xy_diff_decay_repeat}
\end{theorem}

\begin{sketch}
The details of proof are presented in the Appendix~\ref{app:xy_diff_decay}. Start from a point $(x\t, y\t)$ where $x\t y\t>\mu$. Because $y\tt-x\tt=(y\t-x\t)(1+\eta(x\t y\t-\mu))$, it suffices to show 
\begin{align}
 \left|\frac{y\T{t+2}-x\T{t+2}}{y\t-x\t}\right|=|(1+\eta(x\t y\t-\mu))(1+\eta(x\tt y\tt-\mu))| < 1. \label{eq:ratiotp2}
\end{align}
Since $1+\eta(x\t y\t-\mu)>1$, the analysis of $1+\eta(x\tt y\tt-\mu)$ is divided into three cases considering the coupling of $(x\t,y\t), (x\tt, y\tt)$. 
\end{sketch}

\begin{remark}
Actually, for a larger $K\ge 1.5$, it is possible for GD to converge to an inbalanced orbit. For instance, Figure 15 in \citep{wang2021large} shows inbalanced orbits for $f(x)=\frac{1}{2}(xy-1)^2$ with $K=1.9$.
\end{remark}

Combining with the fact that the probability of GD converging to a stationary point that has sharpness beyond the edge of stability is zero~\citep{ahn2022understanding}, Theorem~\ref{thm:xy_diff_decay} reveals $x$ and $y$ would converge to a perfect balance. Note that this balancing effect is different from that of gradient flow~\citep{du2018algorithmic}, where the latter states that gradient flow preserves the difference of norms of different layers along training. As a result, in gradient flow, inbalanced initialization induces inbalanced convergence, while in our case inbalanced-initialized weights converge to a perfect balance. Furthermore, Theorem~\ref{thm:xy_diff_decay} shows an effect that the two parameters are squeezed to a single variable, which re-directs to our 1-D analysis in Theorem~\ref{thm:1dglobal}. Therefore, actually both cases converge to the same orbit when $1<K< 1.121$, as stated in Prop~\ref{prop:xy}. Numerical results are presented in Figure~\ref{fig:2d}.

\begin{prop}[Restatement of Prop~\ref{prop:xy}]
    Following the setting in Theorem~\ref{thm:xy_diff_decay}. Further assume $1<K<\sqrt{4.5}-1\approx 1.121$. Then GD converges to an orbit of period 2. The orbit is formally written as $\{(x=y=\delta_i)|i=1,2\}$, with $\delta_1\in(0,\sqrt{\mu}), \delta_2\in(\sqrt{\mu},2\sqrt{\mu})$  as the solutions of solving $\delta$ in
\begin{align*}
    \eta  = \frac{1}{\delta^2\left( \sqrt{\frac{\mu}{\delta^2}-\frac{3}{4}} +\frac{1}{2}\right)}.
\end{align*}
\label{prop:xy_2}
\end{prop}

\begin{remark}
    Actually this convergence is close to the flattest minima because: if the learning rate decays to infinitesimal after sufficient oscillations, then the trajectory walks towards the flattest minima.
\end{remark}

However, one thing to notice is that the inbalance at initialization needs to be bounded in Theorem~\ref{thm:xy_diff_decay} because both $x$ and $y$ are assumed to stay positive along the training. More precisely, we have
\begin{align}
    x\tt y\tt = x\t y\t (1-\eta(x\t y\t-\mu))^2 - \eta(x\t y\t-\mu)(x\t-y\t)^2,
\end{align}
and then $x\tt y\tt<0$ when $|x\t-y\t|$ is large with $x\t y\t >\mu$ fixed. Therefore, we provide a condition to guarantee both $x,y$ positive as follows, with details presented in the Appendix~\ref{app:xy_positive}.
\begin{lemma}
In the setting of Theorem~\ref{thm:xy_diff_decay}, denote the initialization as $m=\frac{|y_0-x_0|}{\sqrt{\mu}}$ and $x_0 y_0>\mu$. Then, during the whole process, both $x$ and $y$ will always stay positive, denoting $p=\frac{4}{\left( m+\sqrt{m^2+4} \right)^2}$ and $q=(1+p)^2$, if
\begin{align*}
    \max\left\{
    \eta(x_0 y_0-\mu),
    \frac{4}{27}\left(1+K\right)^3+\left(\frac{2}{3}K^2-\frac{1}{3}K+\frac{q K^2}{2(K+1)}m^2\right)q m^2-K
    \right\}
    <
    p.
\end{align*}
\label{thm:xy_positive}
\end{lemma}


\subsection{On Matrix Factorization}\label{app:add_res_mf}

In this section, we present two additional results of matrix factorization.

\subsubsection{Asymmetric Case: 1D function at the minima}

Before looking into the theorem, we would like to clarify the definition of the loss Hessian. Inherently, we squeeze $\mathbf{X},\mathbf{Y}$ into a vector $\theta=\text{vec}(\mathbf{X},\mathbf{Y})\in\R^{mp+pq}$, which vectorizes the concatnation. As a result, we are able to represent the loss Hessian w.r.t. $\theta$ as a matrix in $\R^{(mp+pq)\times(mp+pq)}$. Meanwhile, the support of the loss landscape is in $\R^{mp+pq}$. Similarly, we use $(\Delta\mathbf{X},\Delta\mathbf{Y})$ in the same shape of $(\mathbf{X},\mathbf{Y})$ to denote . In the following theorem, we are to show the leading eigenvector $\Delta\triangleq \text{vec}(\Delta\mathbf{X}, \Delta\mathbf{Y})\in\R^{mp+pq}$ of the loss Hessian. Since the cross section of the loss landscape and $\Delta$ forms a 1D function $f_{\Delta}$, we would also show the stable-oscillation condition on 1D function holds at the minima of $f_{\Delta}$.

\begin{theorem}
    For a matrix factorization problem, assume $\mathbf{X}\mathbf{Y}=\mathbf{C}$.
    Consider SVD of both matrices as $\mathbf{X} = \sum_{i=1}^{\min\{m,p\}} \sigma_{x,i} u_{x,i} v_{x,i}^\top$ and $\mathbf{Y} = \sum_{i=1}^{\min\{p,q\}} \sigma_{y,i} u_{y,i} v_{y,i}^\top$, where both groups of $\sigma_{\cdot,i}$'s are in descending order and both top singular values $\sigma_{x,1}$ and $\sigma_{y,1}$ are unique. Also assume $v_{x,1}^\top u_{y,1}\neq 0$. Then the leading eigenvector of the loss Hessian is $\Delta=\text{vec}(C_1 u_{x,1}u_{y,1}^\top, C_2 v_{x,1}v_{y,1}^\top)$ with $C_1=\frac{\sigma_{y,1}}{\sqrt{\sigma_{x,1}^2+\sigma_{y,1}^2}}, C_2=\frac{\sigma_{x,1}}{\sqrt{\sigma_{x,1}^2+\sigma_{y,1}^2}}$. Denote $f_\Delta$ as the 1D function at the cross section of the loss landscape and the line following the direction of $\Delta$ passing $\text{vec}(\Delta\mathbf{X}, \Delta\mathbf{Y})$. Then, at the minima of $f_\Delta$, it satisfies
    \begin{align}
        3[f_\Delta^{(3)}]^2 - f_\Delta^{(2)}f_\Delta^{(4)}>0.
    \end{align}
    \label{thm:mf_1d_cond}
\end{theorem}

The proof is provided in Appendix~\ref{sec:app_mf_1d}. This theorem aims to generalize our 1-D analysis into higher dimension, and it turns out the 1-D condition is sastisfied around any minima for two-layer matrix factorization.
In Theorem~\ref{thm:1dlocal} and Lemma~\ref{lem:1dhigher}, if such 1-D condition holds, there must exist a period-2 orbit around the minima for GD beyond EoS. However, this is not straightforward to generalize to high dimensions, because 1) directions of leading eigenvectors and (nearby) gradient are not necessarily aligned, and 2) it is more natural and practical to consider initialization \textit{in any direction} around the minima instead of strictly along leading eigenvectors. Therefore, below we present a convergence analysis with initialization near the minima, but in any direction instead.

\section{Additional Experiments}
\label{app:add_exps}



In Appendix~\ref{app:proven_exp}, we provide numerical experiments to verify our theorems. Then, we provide additional experiments on MLP and MNIST.

\subsection{Proven Settings}\label{app:proven_exp}

\paragraph{1-D functions.} As discussed in the Section~\ref{sec:1d}, we have $f(x)=\frac{1}{4}(x^2-1)^2$ satisfying the condition in Theorem~\ref{thm:1dlocal} and $g(x)=2\sin(x)$ satisfying Lemma~\ref{lem:1dhigher}, so we estimate that both $f$ and $g$ allow stable oscillation around the local minima. It turns out GD stably oscillates around the local minima on both functions, when $\eta>\frac{2}{f''(\Bar{x})}$ slightly, as shown in Figure~\ref{fig:1d}.

\begin{figure}[ht]
\begin{minipage}{\textwidth}
    \centering
    \includegraphics[width=0.45\textwidth]{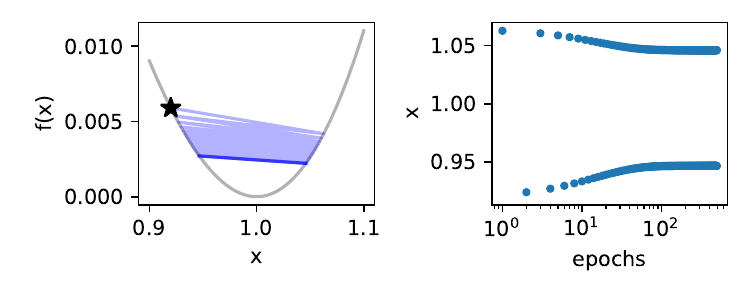}
    \includegraphics[width=0.45\textwidth]{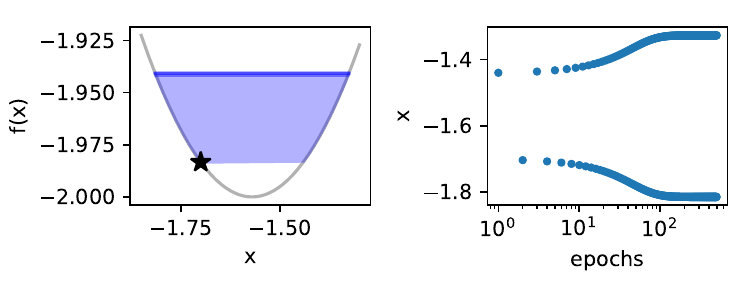}
\end{minipage}
    \caption{Running GD around the local minima of $f(x)=\frac{1}{4}(x^2-1)^2$ (left two) and $f(x)=2\sin(x)$ (right two) with learning rate $\eta=1.01>\frac{2}{f''(\Bar{x})}=1.$ Stars denote the start points. It turns out both functions allow stable oscillation around the local minima.}
    \label{fig:1d}
\end{figure}

\paragraph{Two-layer single-neuron model.} As discussed in the Section~\ref{sec:one_neuron}, with a learning rate $\eta\in(d,1.1d]$, a single-neuron network $f(x)=v\cdot\sigma(w^\top x)$ is able to align with the direction of the teacher neuron under population loss. We train such a model in empirical loss on 1000 data points uniformly sampled from a sphere $\mathcal{S}^1$, as shown in Figure~\ref{fig:one_neuron}. The student neuron is initialized orthogonal to the teacher neuron. In the end of training, $w_y$ decays to a small value before the inbalance $|v-w_x|$ decays sharply, which verifies our argument in Section~\ref{sec:one_neuron}. With a small $w_y$, this nonlinear problem degenerates to a 2-D problem on $v,w_x$. Then, the balanced property makes it align with the 1-D problem where $v$ and $w_x$ converge to a period-2 orbit. Note that the small residuals of $|v-w_x|$ and $w_y$ are due to the difference between population loss and empirical loss.

\begin{figure}[ht]
    \centering
    \includegraphics[width=0.9\linewidth]{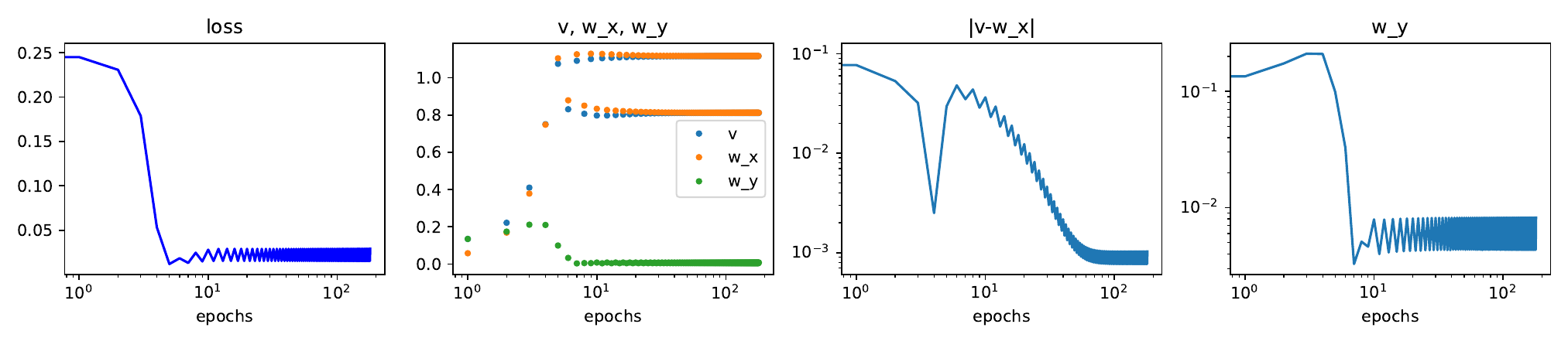}
    \caption{Running GD in the teacher-student setting with learning rate $\eta=2.2=1.1d$, trained on 1000 points uniformly sampled from sphere $\mathcal{S}^1$ of $\norm{x}=1$. The teacher neuron is $\tilde{w}=[1,0]$ and the student neuron is initialized as $w\T{0}=[0,0.1]$ with $v\T{0}=0.1$. }
    \label{fig:one_neuron}
\end{figure}

\paragraph{Quasi-symmetric matrix factorization.} As discussed in the Section~\ref{sec:new_mf}, with mild assumptions, the quasi-symmetric case stably wanders around the flattest minima. We train GD on a matrix factorization problem with $\mathbf{X}_0\mathbf{X}_0^\top=\mathbf{C}\in\R^{8\times 8}$. The learning rate is $1.02\times$ EoS threshold. Following the setting in Section~\ref{sec:new_mf}, for symmetric case, the training starts near $(\bX_0,\bX_0)$ and, for quasi-symmetric case, it starts near $(\alpha\bX_0,\nicefrac{1}{\alpha}\bX_0)$ with $\alpha=0.8$, as shown in Figure~\ref{fig:mf}. Although starting with a re-scaling, the quasi-symmetric case achieves the same top singular values in $\bY$ and $\bZ$, which verifies the balancing effect of 2-D functions in Theorem~\ref{thm:xy_diff_decay}. Then, the top singular values of both cases converge to the same period-2 orbit, which verifies Observation~\ref{obs:quasi_mf}.

\begin{figure}[ht]
\begin{minipage}{\textwidth}
    \centering
    \includegraphics[width=0.3\linewidth]{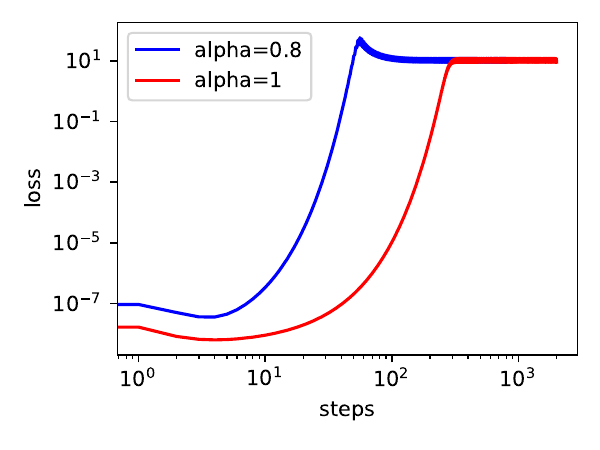}
    \includegraphics[width=0.3\linewidth]{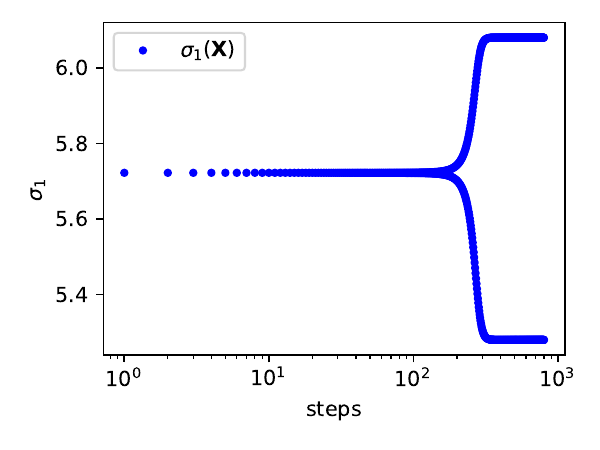}
    \includegraphics[width=0.3\linewidth]{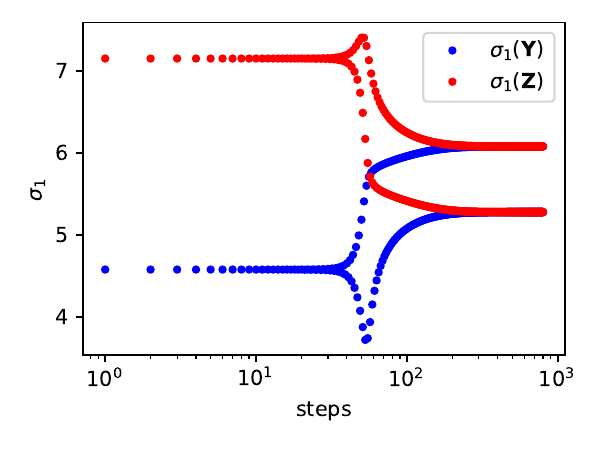}
\end{minipage}
    \caption{Symmetric and Quasi-symmetric Matrix factorization: running GD around flat ($\alpha=1$) and sharp ($\alpha=0.8$) minima. In both cases, their leading singular values converge to the same period-2 orbit (about 6.1 and 5.3). (Left: Training loss. Middle: Largest singular value of symmetric case. Right: Largest singular values of quasi-symmetric case.)}
    \label{fig:mf}
\end{figure}

\subsection{2-D function}

As discussed in the Appendix~\ref{sec:2d}, on the function $f(x,y)=\frac{1}{2}(xy-1)^2$, we estimate that $|x-y|$ decays to 0 when $\eta\in(1,1.5)$, as shown in Figure~\ref{fig:2d}. Since it achieves a perfect balance, the two parameters follows convergence of the corresponding 1-D function $f(x)=\frac{1}{4}(x^2-1)^2$. As shown in Figure~\ref{fig:2d},  $xy$ with $\eta=1.05$ converges to a period-2 orbit, as stated in the 1-D discussion of Theorem~\ref{thm:1dglobal} while $xy$ with $\eta=1.25$ converges to a period-4 orbit, which is out of our range in the theorem. But still it falls into the range for balance in Theorem~\ref{thm:xy_diff_decay}.

\begin{figure}[ht]
\begin{minipage}{\textwidth}
    \centering
    \includegraphics[width=0.75\linewidth]{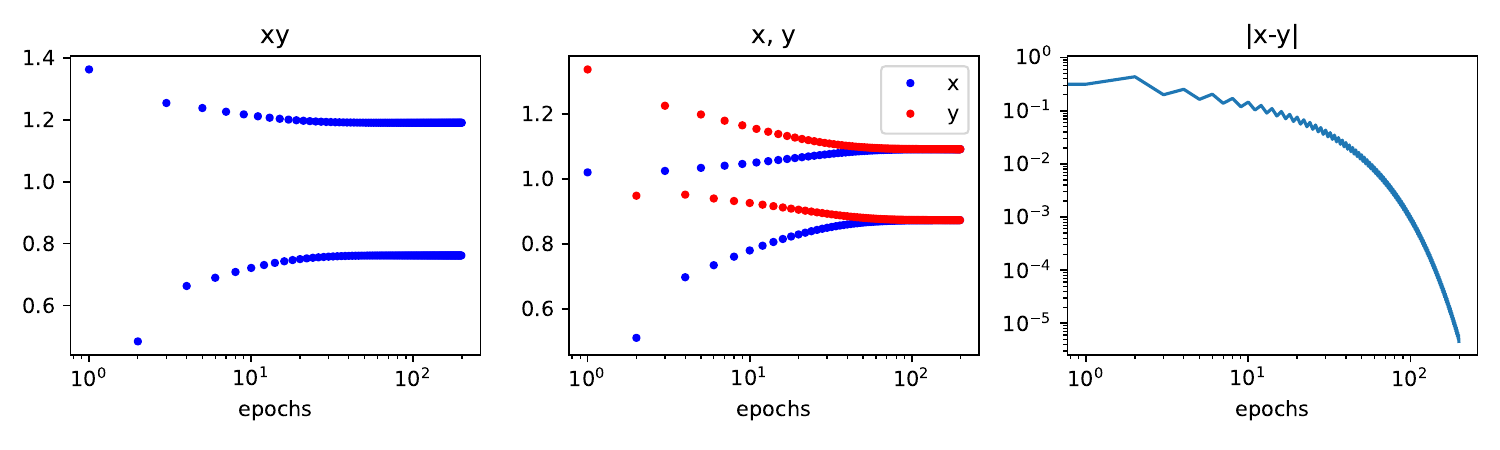}\\
    (a) $\eta = 1.05$
    \\
    \includegraphics[width=0.75\linewidth]{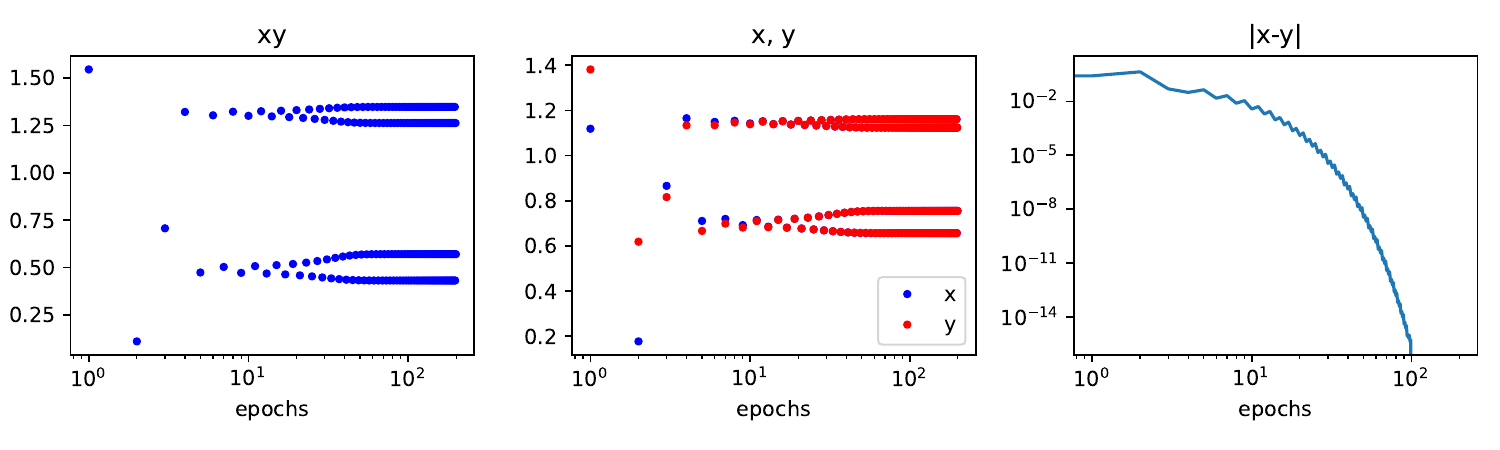}\\
    (b) $\eta = 1.25$
\end{minipage}
    \caption{Running GD on $f(x,y)=\frac{1}{2}(xy-1)^2$ with learning rate $\eta=1.05$ (top) and $\eta=1.25$ (bottom). When $\eta=1.05$, it converges to a period-2 orbit. When $\eta=1.25$, it converges to a period-4 orbit. In both cases, $|x-y|$ decays sharply.}
    \label{fig:2d}
\end{figure}

\subsection{High dimension and MNIST} \label{app:high_dim_exp}

We perform two experiments in relatively higher dimension settings. We are to show two observations that coincides with our discussions in the low dimension:
\vspace{10pt}

\centerline{\textbf{Observation 1}: GD beyond EoS drives to flatter minima.}

\vspace{5pt}
\centerline{\textbf{Observation 2}: GD beyond EoS is in a similar style with the low dimension.}

\subsubsection{2-layer high-dim homogeneous ReLU NNs with planted teacher neurons} \label{app:16-neuron}

We conduct a synthetic experiment in the high-dimension teacher-student framework. The teacher network is in the form of 
\begin{align}
    y|x \coloneqq f_{\text{teacher}}(x;\Tilde{\theta}) = \sum_{i=1}^{16} \textsf{ReLU}( \mathbf{e}_i^\top x),
\end{align}
where $x\in\R^{16}$ and $\mathbf{e_i}$ is the $i$-th vector in the standard basis of $\R^{16}$. The student and the loss are in forms of
\begin{gather}
     f(x;\theta) = \sum_{i=1}^{16} v_i\cdot \textsf{ReLU}( w_i^\top x),\\
     L(\theta;\Tilde{\theta})=\frac{1}{m}\sum_{i}^{16}
     \left( f(x;\theta)- y|x_i \right)^2.
\end{gather}
Apparently, the global minimum manifold contains the following set $\mathcal{M}$ as (w.l.o.g., ignoring any permutation)
\begin{align}
    \mathcal{M}=\{(v_i,w_i)_{i=1}^{16} ~|~ \forall i\in[16],  w_i=k_i\cdot \mathbf{e}_i, v_i=\frac{1}{k_i}, k_i>0\}.
\end{align}
However, different choices of $\{k_i\}_{i=1}^{16}$ induce different extents of sharpness around each minima. Our aim is to show that \textbf{GD with a large learning rate beyond the edge of stability drives to the flattest minima from sharper minima}.

\paragraph{Initialization.} We initialize all student neurons directionally aligned with the teachers as $w_i\parallel \mathbf{e}_i$ but choose various $k_i$, as $k_i= 1+0.0625(i-1)$. Obviously, such a choice of $\{k_i\}_{i=1}^{16}$ is not at the flattest minima, due to the isotropy of teacher neurons. Also we add small noise to $w_i$ to make the training start closely (but not exactly) from a sharp minima, as
\begin{align}
    w_i = k_i\cdot (\mathbf{e}_i+0.01\epsilon), ~~~\epsilon\sim \mathcal{N}(0,I).
\end{align}

\paragraph{Data.} We uniformly sample 10000 data points from the unit sphere $\mathcal{S}^{15}$.

\paragraph{Training.} We run gradient descent with two learning rates $\eta_1 = 0.5, \eta_2=2.6$. Later we will show with experiments that the EoS threshold of learning rate is around 2.5, so $\eta_2$ is beyond the edge of stability. GD with these two learning rates starts from the same initialization for 100 epochs. Then we extend another 20 epochs with learning rate decay to 0.5 from 2.6 for the learning-rate case.

\paragraph{Results.} All results are provided in Figure~\ref{fig:add_exp_high_ts}. Both Figure~\ref{fig:add_exp_high_ts} (a, b) present the gap between these two trajectories, where GD with a small learning rate stays around the sharp minima, while that with a larger one drives to flatter minima. Then GD stably oscillates around the flatter minima.

Meanwhile, from Figure~\ref{fig:add_exp_high_ts} (b), when we decrease the learning rate from 2.6 to 0.5 after 100 epochs, GD converges to a nearby minima which is significantly flatter, compared with that of lr=0.5.

Figure~\ref{fig:add_exp_high_ts} (c) provides a more detailed view of $\frac{\norm{w_i}}{v_i}$ for all 16 neurons. All neurons with lr=0.5 stay at the original ratio $k_i^2$. But those with lr=2.6 all converge to the same ratio around $k^2=\frac{\norm{w}}{v}=1.21$, as shown in Figure~\ref{fig:add_exp_high_ts} (d). We compute the relationship between the sharpness of global minima in $\mathcal{M}$ and different choices of $k$, as shown in Figure~\ref{fig:add_exp_high_ts} (e, f).
Actually, $k^2=1.21$ is the best choice of $\{k_i\}_{i=1}^{16}$ such that the minima is the flattest.

Therefore, we have shown that, \textbf{in such a setting of high-dimension teacher-student network, GD beyond the edge of stability drives to the flattest minima}.

\newpage 

\begin{figure}[t]
\begin{minipage}{\textwidth}
    \centering
    \includegraphics[width=0.45\linewidth]{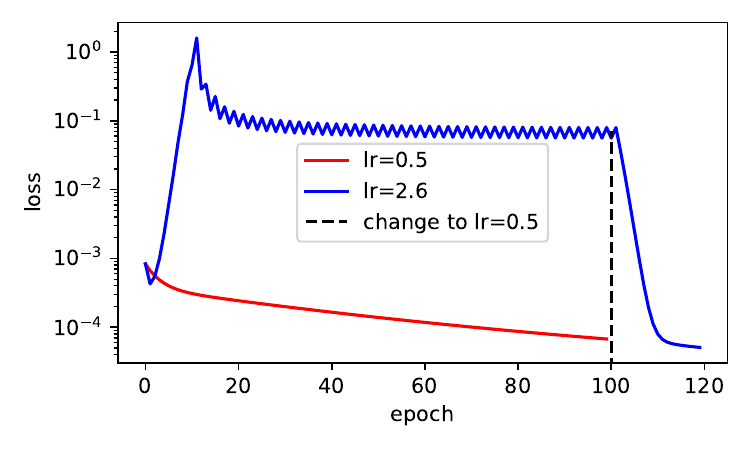}
    \includegraphics[width=0.45\linewidth]{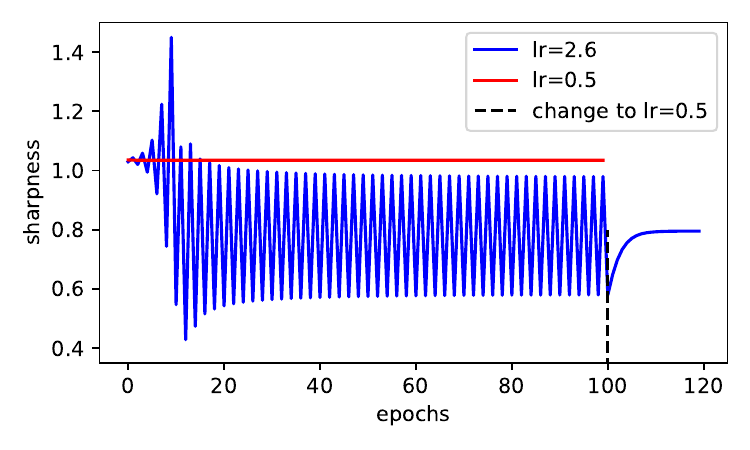}\\
    \hfill \hspace{-45pt} 
    (a) Training loss
    \hfill 
    (b) Sharpness
    \hspace{70pt} \hfill \\
    \vspace{10pt}\includegraphics[width=0.45\linewidth]{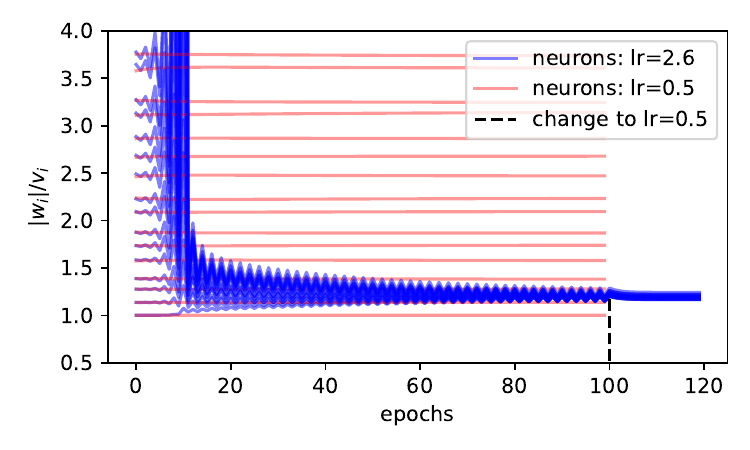}
    \includegraphics[width=0.45\linewidth]{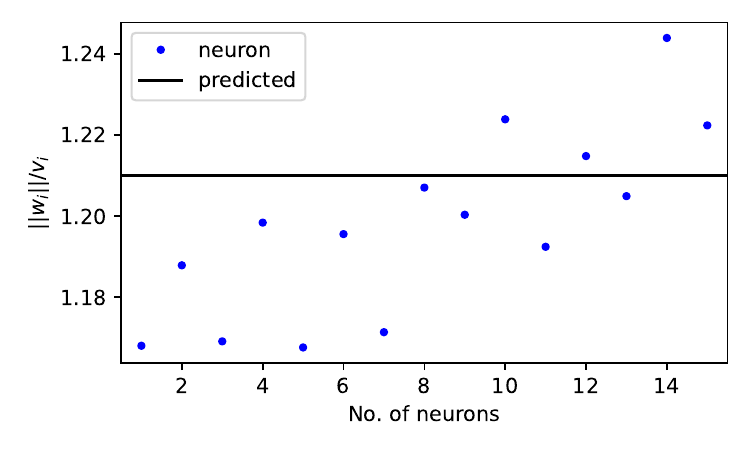}\\
    \hfill \hspace{10pt} 
    (c) ratio of $\| w_i\|$ and $v_i$ during training
    \hfill 
    (d) final ratio of $\| w_i\|$ and $v_i$ when lr=2.6
    \hspace{10pt} \hfill \\
    \vspace{10pt}\includegraphics[width=0.45\linewidth]{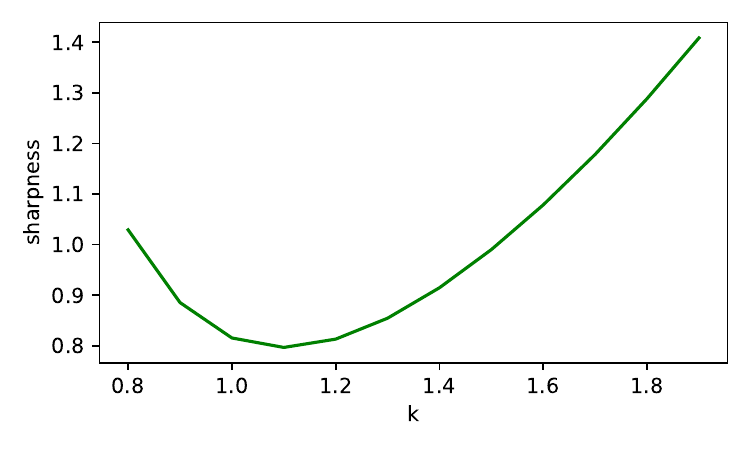}
    \includegraphics[width=0.45\linewidth]{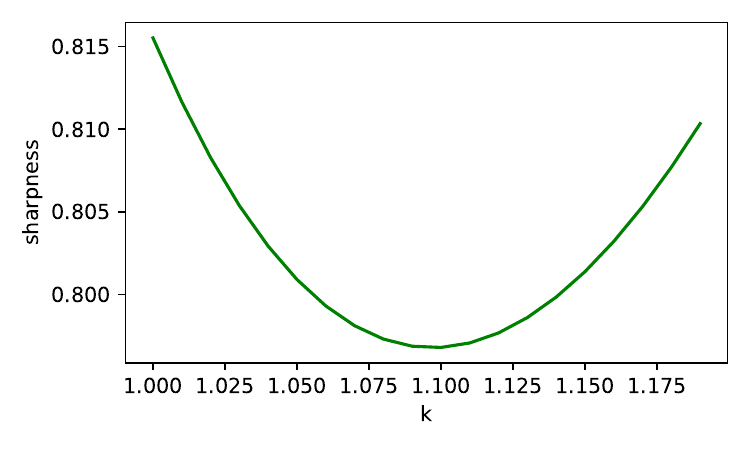}\\
    \hfill \hspace{30pt} 
    (e) sharpness for different ratio $\sqrt{\frac{\norm{w_i}}{v_i}}$
    \hfill 
    (f) sharpness for different ratio $\sqrt{\frac{\norm{w_i}}{v_i}}$ (zoom-in) 
    \hspace{30pt} \hfill \\
\end{minipage}
\caption{Result of 2-layer 16-neuron teacher-student experiment.}
\label{fig:add_exp_high_ts}
\end{figure}

\subsubsection{3, 4, 5-layer non-homogeneous MLPs on MNIST} \label{app:mnist}

We conduct an experiment on real data to show that \textbf{our finding in the low-dimension setting in Theorem~\ref{thm:1dlocal} is possible to generalize to high-dimensional setting}. More precisely, our goals are to show, when GD is beyond EoS,
\begin{enumerate}
    \item the oscillation direction (gradient) aligns with the top eigenvector of Hessian.
    \item the 1D function at the cross-section of oscillation direction and high-dim loss landscape satisfies the conditions in Theorem~\ref{thm:1dlocal}.
\end{enumerate}

\paragraph{Network, dataset and training.}
We run 3, 4, 5-layer ReLU MLPs on MNIST~\cite{lecun1998gradient}. The networks have 16 neurons in each layer. To make it easier to compute high-order derivatives, we simplify the dataset by 1) only using 2000 images from class 0 and 1, and 2) only using significant input channels where the standard deviation over the dataset is at least 110, which makes the network input dimension as 79. We train the networks using MSE loss subjected to GD with large learning rates $\eta=0.5, 0.4, 0.35$ and a small rate $\eta=0.1$ (for 3-layer). Note that the larger ones are beyond EoS.

\begin{defi}[line search minima]
Consider a function $f$, learning rate $\eta$ and a point $x\in \text{domain}(f)$. We call $\tilde{x}$ as the line search minima of $x$ if
\begin{align}
    \tilde{x}&=x-c^*\cdot \eta \nabla f(x),\\
    c^* &= \text{argmin}_{c\in[0,1]} ~f\left(x-c\cdot \eta\nabla f(x)\right).
\end{align}
\end{defi}

The line search minima $\tilde{x}$ can interpreted as the lowest point on the 1D function induced by the gradient at $x$. If GD is beyond EoS, $\tilde{x}$ stays in the valley below the oscillation of $x$.

\paragraph{Results.} All results are presented in Figure~\ref{fig:add_exp_mnist}, \ref{fig:add_exp_mnist_four_layer} and \ref{fig:add_exp_mnist_five_layer}. 

Take the 3-layer as an example. From Figure~\ref{fig:add_exp_mnist} (a, b), GD is beyond EoS during epochs 10-14 and 21-60. For these epochs, cosine similarity between the top Hessian eigenvector $v_1$ and the gradient is pretty close to 1, as shown in Figure~\ref{fig:add_exp_mnist} (c), which verifies our goal 1.

In Figure~\ref{fig:add_exp_mnist} (d), we compute $3[f^{(3)}]^2 - f^{(2)}f^{(4)}$ at line search minima along training, which is required to be positive in Theorem~\ref{thm:1dlocal} to allow stable oscillation. Then it turns out most points have $3[f^{(3)}]^2 - f^{(2)}f^{(4)}>0$ except a few points, all of which are not in the EoS regime, and these few exceptional points might be due to approximation error to compute the fourth-order derivative since their negativity is quite small. This verifies our goal 2.

Both the above arguments are the same in the cases of 4 and 5 layers as shown in Figure~\ref{fig:add_exp_mnist_four_layer} and \ref{fig:add_exp_mnist_five_layer}.

\begin{figure}[ht]
\begin{minipage}{\textwidth}
    \centering
    \includegraphics[width=0.45\linewidth]{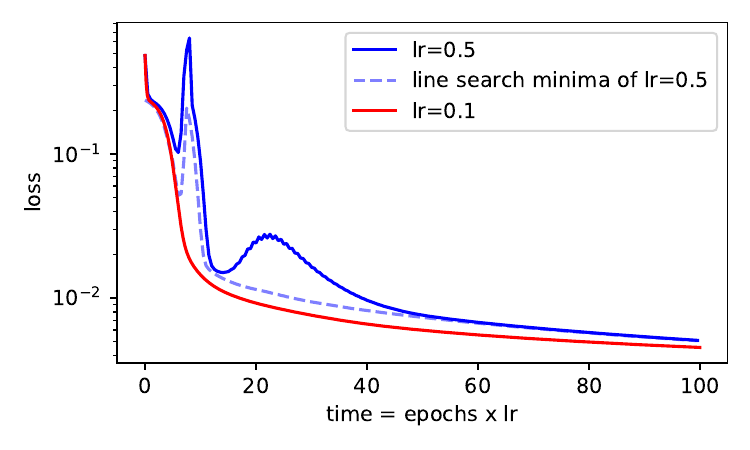}
    \includegraphics[width=0.45\linewidth]{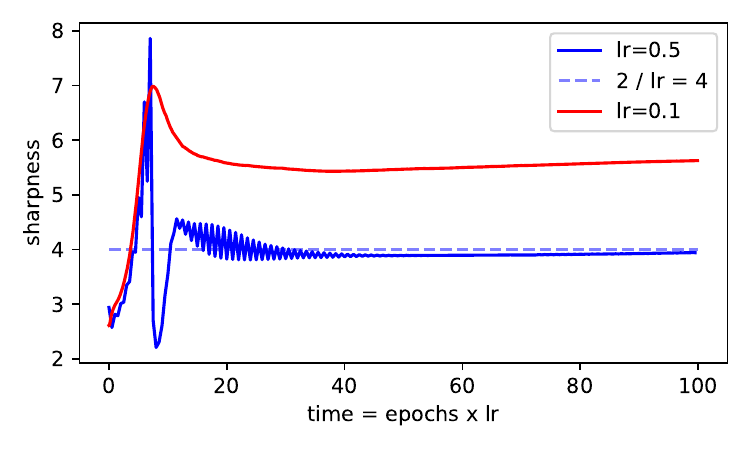}\\
    \hfill \hspace{-45pt} 
    (a) Training loss
    \hfill 
    (b) Sharpness
    \hspace{70pt} \hfill \\
    \vspace{10pt}\includegraphics[width=0.45\linewidth]{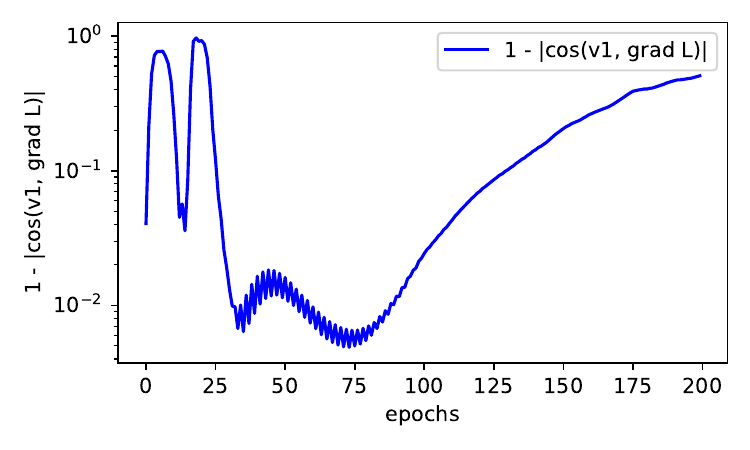}
    \includegraphics[width=0.45\linewidth]{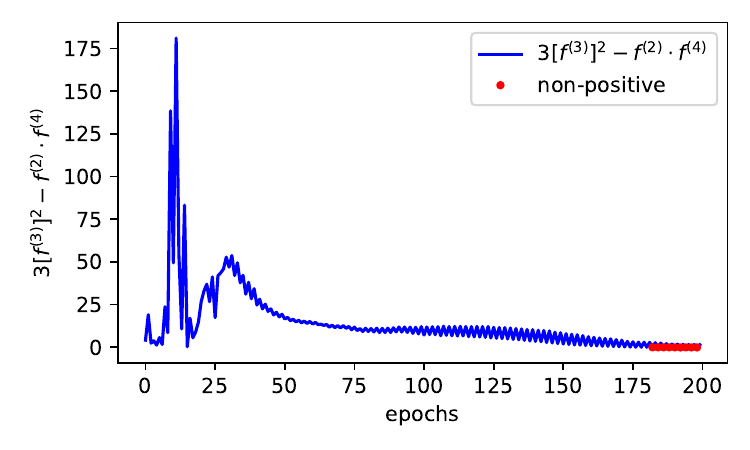}\\
    \hfill \hspace{-10pt} 
    (c) similarity of gradient and top eig-vector $v_1$
    \hfill 
    (d) $3[f^{(3)}]^2 - f^{(2)}f^{(4)}$ at line search minima
    \hspace{10pt} \hfill
\end{minipage}
\caption{Result of \textbf{3-layer} ReLU MLPs on MNIST. Both (c) and (d) are for learning rate as 0.5.}
\label{fig:add_exp_mnist}
\end{figure}

\begin{figure}[ht]
\begin{minipage}{\textwidth}
    \centering
    \includegraphics[width=0.45\linewidth]{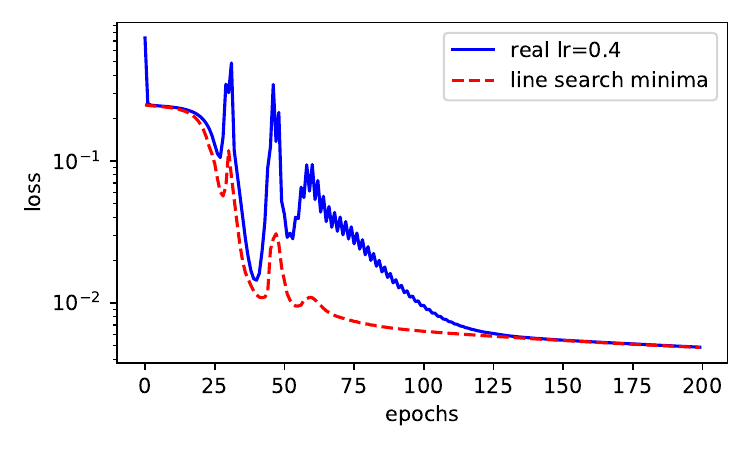}
    \includegraphics[width=0.45\linewidth]{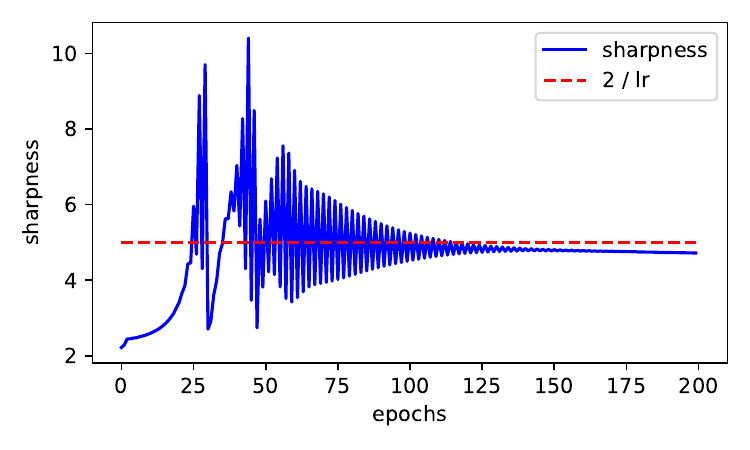}\\
    \hfill \hspace{-45pt} 
    (a) Training loss
    \hfill 
    (b) Sharpness
    \hspace{70pt} \hfill \\
    \vspace{10pt}\includegraphics[width=0.45\linewidth]{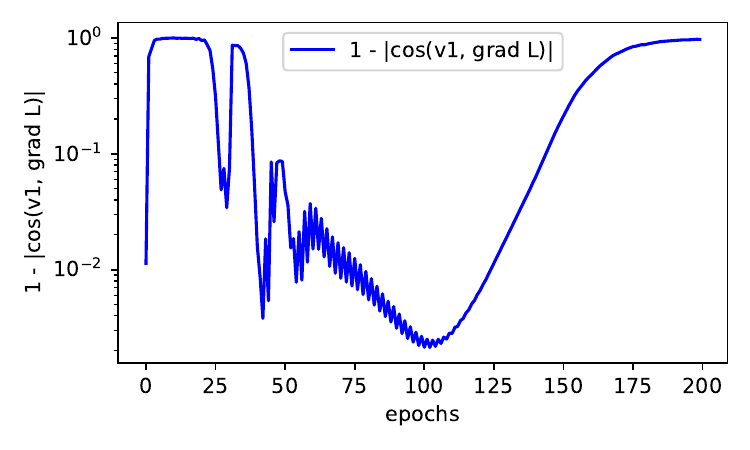}
    \includegraphics[width=0.45\linewidth]{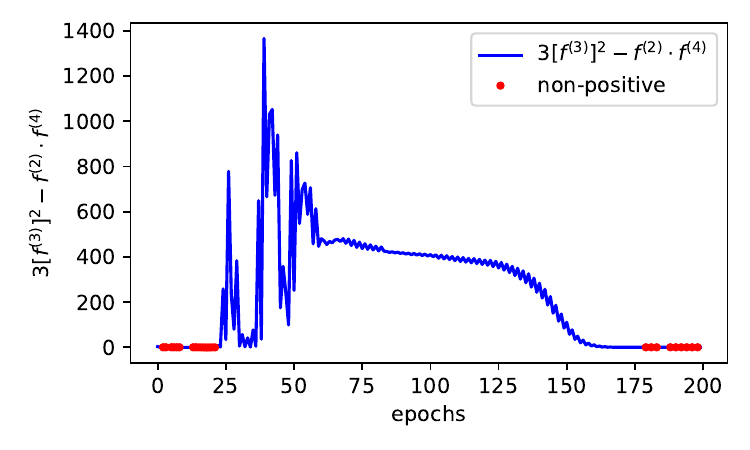}\\
    \hfill \hspace{-10pt} 
    (c) similarity of gradient and top eig-vector $v_1$
    \hfill 
    (d) $3[f^{(3)}]^2 - f^{(2)}f^{(4)}$ at line search minima
    \hspace{10pt} \hfill
\end{minipage}
\caption{Result of \textbf{4-layer} ReLU MLPs on MNIST.}
\label{fig:add_exp_mnist_four_layer}
\end{figure}

\begin{figure}[ht]
\begin{minipage}{\textwidth}
    \centering
    \includegraphics[width=0.45\linewidth]{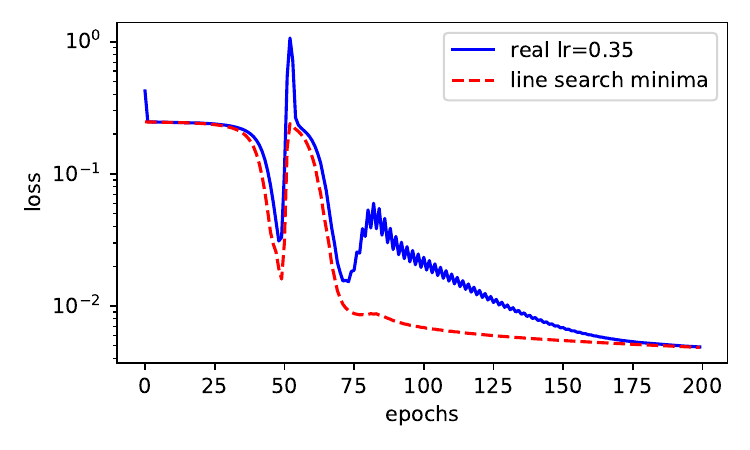}
    \includegraphics[width=0.45\linewidth]{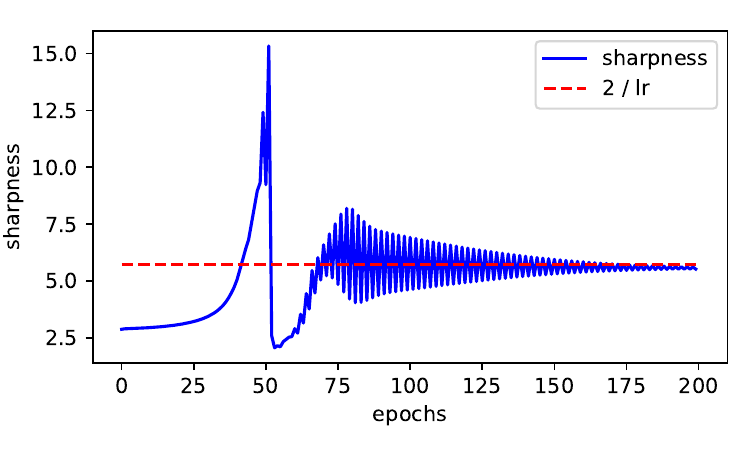}\\
    \hfill \hspace{-45pt} 
    (a) Training loss
    \hfill 
    (b) Sharpness
    \hspace{70pt} \hfill \\
    \vspace{10pt}\includegraphics[width=0.45\linewidth]{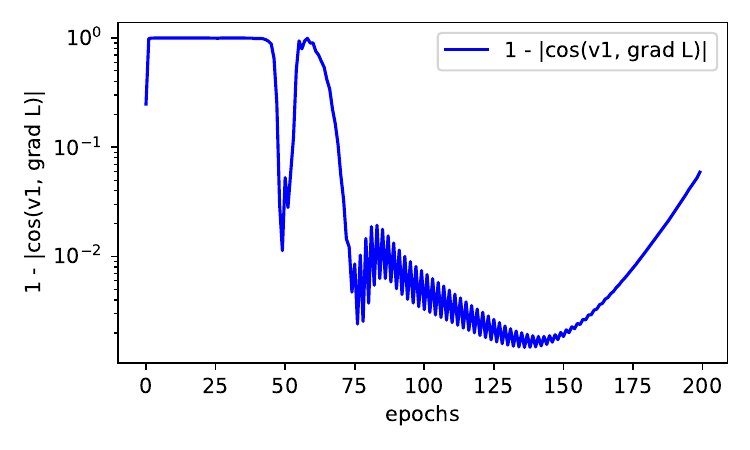}
    \includegraphics[width=0.45\linewidth]{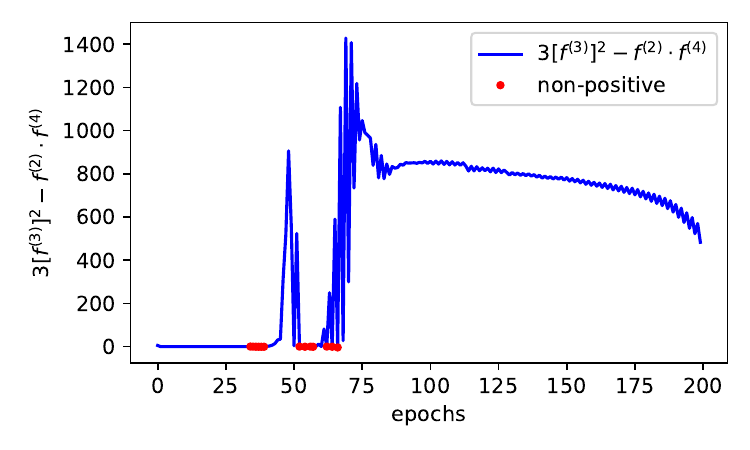}\\
    \hfill \hspace{-10pt} 
    (c) similarity of gradient and top eig-vector $v_1$
    \hfill 
    (d) $3[f^{(3)}]^2 - f^{(2)}f^{(4)}$ at line search minima
    \hspace{10pt} \hfill
\end{minipage}
\caption{Result of \textbf{5-layer} ReLU MLPs on MNIST.}
\label{fig:add_exp_mnist_five_layer}
\end{figure}
\clearpage

\section{Proof of Theorem~\ref{thm:1dlocal}}
\label{app:1dlocal}


\begin{theorem}[Restatement of Theorem~\ref{thm:1dlocal}]
Consider any 1-D differentiable function $f(x)$ around a local minima $\Bar{x}$, satisfying (i) $f^{(3)}(\Bar{x})\neq 0$, and (ii) $3[f\T{3}]^2-f'' f\T{4}>0$ at $\Bar{x}$. Then, there exists $\epsilon$ with sufficiently small $|\epsilon|$ and $\epsilon\cdot f^{(3)}>0$ such that: for any point $x_0$ between $\Bar{x}$ and $\bar{x} - \epsilon$, there exists a learning rate $\eta$ such that the update rule $F_\eta$ of GD satisfies $F_\eta(F_\eta(x_0))=x_0$, and
\[
\frac{2}{f''(\Bar{x})}<\eta<\frac{2}{f''(\Bar{x})-\epsilon\cdot f^{(3)}(\Bar{x})}.
\]
\end{theorem}

\begin{proof}
For simplicity, we assume $f\T{3}(\Bar{x})>0$. 
Imagine a starting point $x_0=\Bar{x}-\epsilon, \epsilon>0$. We omit $f'(\Bar{x}),f''(\Bar{x}),f\T{3}(\Bar{x}), f\T{4}(\Bar{x})$ as $f',f'',f\T{3},f\T{4}$. After running two steps of gradient descent, we have
\begin{align*}
    x_0 &= \Bar{x} - \epsilon,\\
    f'(x_0) &= f' - f''\epsilon + \frac{1}{2}f\T{3}\epsilon^2 - \frac{1}{6}f\T{4}\epsilon^3 +\O(\epsilon^4) \\
    &= - f''\epsilon + \frac{1}{2}f\T{3}\epsilon^2 - \frac{1}{6}f\T{4}\epsilon^3 +\O(\epsilon^4), \\
    x_1 &= x_0 - \eta f'(x_0) = \Bar{x} - \epsilon - \eta\big(-f''\epsilon + \frac{1}{2}f\T{3}\epsilon^2- \frac{1}{6}f\T{4}\epsilon^3\big)+\O(\epsilon^4),\\
    f'(x_1) &= f''\cdot(x_1-\Bar{x})+\frac{1}{2}f\T{3}\cdot(x_1-\Bar{x})^2+\frac{1}{6}f\T{4}\cdot(x_1-\Bar{x})^3+\O(\epsilon^4),\\
    x_2 &= x_1 - \eta f'(x_1),\\
    \frac{x_2-x_0}{\eta} &= - \left(-f''\epsilon + \frac{1}{2}f\T{3}\epsilon^2- \frac{1}{6}f\T{4}\epsilon^3\right)-f''\cdot\left(- \epsilon - \eta\big(-f''\epsilon + \frac{1}{2}f\T{3}\epsilon^2- \frac{1}{6}f\T{4}\epsilon^3\big)\right) \\
    & -\frac{1}{2} f\T{3}\left(- \epsilon - \eta\big(-f''\epsilon + \frac{1}{2}f\T{3}\epsilon^2- \frac{1}{6}f\T{4}\epsilon^3\big)\right)^2
    -\frac{1}{6} f\T{4}\cdot\left(-\epsilon-\eta(-f''\epsilon)\right)^3 +\O(\epsilon^4)\\
    &=\left(2 f'' -\eta f'' f''\right)\epsilon + \left(-\frac{1}{2}f\T{3}+\frac{1}{2}\eta f''f\T{3}-\frac{1}{2}f\T{3}(-1+\eta f'')^2\right)\epsilon^2 \\
    &+\left(\frac{1}{6}f\T{4}-\frac{1}{6}\eta f''f\T{4}+\frac{1}{2}(-1+\eta f'')\eta f\T{3}f\T{3}-\frac{1}{6}(-1+\eta f'')^3 f\T{4}\right)\epsilon^3+\O(\epsilon^4).
\end{align*}
When $\eta=\frac{2}{f''}$, it holds
\begin{align}
    \frac{x_2-x_0}{\eta}= \left(\frac{1}{2}\eta f\T{3}f\T{3}-\frac{1}{3}f\T{4}\right)\epsilon^3 + \O(\epsilon^4), \label{eq:temp3}
\end{align}
which would be positive if $\frac{1}{2}\eta f\T{3}f\T{3}-\frac{1}{3}f\T{4}=\frac{1}{3f''}(3[f\T{3}]^2-f'' f\T{4})>0$ and $|\epsilon|$ is sufficiently small.

When $\eta=\frac{2}{f''-\epsilon\cdot f\T{3}}$ then $\eta f'' = 2+2\frac{f\T{3}}{f''}\epsilon+\O(\epsilon^2)$, it holds
\begin{align}
    \frac{x_2-x_0}{\eta} = -2f\T{3}\epsilon^2+\left(-\frac{1}{2}f\T{3}+f\T{3}-\frac{1}{2}f\T{3}\right)\epsilon^2 + \O(\epsilon^3)
    =-2f\T{3}\epsilon^2 + \O(\epsilon^3),
\end{align}
which is negative  when $|\epsilon|$ is sufficiently small.

Therefore, there exists a learning rate $\eta\in(\frac{2}{f''}, \frac{2}{f''-\epsilon\cdot f\T{3}})$ such that $x_2=x_0$ due to the continuity of $(x_2-x_0)$ with respect to $\eta$.

The above proof can be generalized to the case of $x_0=\Bar{x}-\epsilon'$ with $\epsilon'\in(0,\epsilon]$ and the learning rate is still bounded as $\eta\in(\frac{2}{f''}, \frac{2}{f''-\epsilon\cdot f\T{3}})$.
\end{proof}

\section{Proof of Lemma~\ref{lem:1dhigher}}
\label{app:1dhigher}

\begin{lemma}[Restatement of Lemma~\ref{lem:1dhigher}]
Consider any 1-D differentiable function $f(x)$ around a local minima $\Bar{x}$, satisfying that the lowest order non-zero derivative (except the $f''$) at $\Bar{x}$ is $f\T{k}(\Bar{x})$ with $k\ge 4$. Then, there exists $\epsilon$ with sufficiently small $|\epsilon|$ such that: for any point $x_0$ between $\Bar{x}$ and $\Bar{x}-\epsilon$, and
\begin{enumerate}
    \item if $k$ is odd and $\epsilon\cdot f\T{k}(\Bar{x})>0, f\T{k+1}(\Bar{x})<0$, 
    then there exists $\eta\in(\frac{2}{f''}, \frac{2}{f''-f\T{k}\epsilon^{k-2}})$,
    \item if $k$ is even and $f\T{k}(\Bar{x})<0$, then there exists $\eta\in(\frac{2}{f''}, \frac{2}{f''+f\T{k}\epsilon^{k-2}})$,
\end{enumerate}
such that: the update rule $F_\eta$ of GD satisfies $F_\eta(F_\eta(x_0))=x_0$.
\end{lemma}

\begin{proof}
(1) 
If $k$ is odd, assuming $f\T{k}>0$ for simplicity, we have
\begin{align*}
    x_0 &= \Bar{x}-\epsilon, \\
    f'(x_0) &= -f''\epsilon + \frac{1}{(k-1)!}f\T{k}\epsilon^{k-1}-\frac{1}{k!}f\T{k+1}\epsilon^k + \O(\epsilon^{k+1}),\\
    x_1 &= x_0-\eta f'(x_0) = \Bar{x}-\epsilon +\eta f''\epsilon - \frac{1}{(k-1)!}\eta f\T{k}\epsilon^{k-1}+\frac{1}{k!}\eta f\T{k+1}\epsilon^k  + \O(\epsilon^{k+1}),\\
    f'(x_1) &= f''\cdot(x_1-\Bar{x})+\frac{1}{(k-1)!}f\T{k}\cdot(x_1-\Bar{x})^{k-1}+\frac{1}{k!}f\T{k+1}\cdot(x_1-\Bar{x})^{k}+\O(\epsilon^{k+1}), \\
    \frac{x_2-x_0}{\eta} &= \frac{x_1-\eta f'(x_1)-x_0}{\eta} =-f'(x_0)-f'(x_1) \\
    &= \left(2f''-\eta f''f''\right)\epsilon \\
    &~ +\left(-\frac{1}{(k-1)!}f\T{k}+\frac{1}{(k-1)!}\eta f''f\T{k}-\frac{1}{(k-1)!}f\T{k}\cdot(-1+\eta f'')^{k-1}\right)\epsilon^{k-1}\\
    &~ +\left(\frac{1}{k!}f^{k+1}-\frac{1}{k!}\eta f'' f\T{k+1}-\frac{1}{k!}f\T{k+1}\cdot(-1+\eta f'')^k\right)\epsilon^k + \O(\epsilon^{k+1})
\end{align*}
When $\eta=\frac{2}{f''}$, it holds
\begin{align}
    \frac{x_2-x_0}{\eta}=-\frac{2}{k!}f\T{k+1}\epsilon^k+ \O(\epsilon^{k+1}).\label{eq:temp1}
\end{align}

When $\eta=\frac{2}{f''-f\T{k}\epsilon^{k-2}}$ then $\eta f''=2+2\frac{f\T{k}}{f''}\epsilon^{k-2}+\O(\epsilon^{2k-4})$, then it holds
\begin{align}
    \frac{x_2-x_0}{\eta}=-2f\T{k}\epsilon^{k-1}+\O(\epsilon^k).\label{eq:temp2}
\end{align}

Since $k$ is odd and $\epsilon\cdot f\T{k}(\Bar{x})>0, f\T{k+1}(\Bar{x})<0$, the above two estimations of $\nicefrac{x_2-x_0}{\eta}$ have one positive and one negative exactly. Therefore, due to the continuity of $x_2-x_0$ wrt $\eta$, there exists a learning rate $\eta\in(\frac{2}{f''}, \frac{2}{f''-f\T{k}\epsilon^{k-2}})$ such that $x_2=x_0$.

The above proof can be generalized to any $x_0$ between $\Bar{x}$ and $\Bar{x}-\epsilon$ with the same bound for $\eta$.

(2)
If $k$ is even, we have
\begin{align*}
    x_0 &= \Bar{x}-\epsilon, \\
    f'(x_0) &= -f''\epsilon - \frac{1}{(k-1)!}f\T{k}\epsilon^{k-1} + \O(\epsilon^{k}),\\
    x_1 &= x_0-\eta f'(x_0) = \Bar{x}-\epsilon +\eta f''\epsilon + \frac{1}{(k-1)!}\eta f\T{k}\epsilon^{k-1}  + \O(\epsilon^{k}),\\
    f'(x_1) &= f''\cdot(x_1-\Bar{x})+\frac{1}{(k-1)!}f\T{k}\cdot(x_1-\Bar{x})^{k-1}+\O(\epsilon^{k}), \\
    \frac{x_2-x_0}{\eta} &= \frac{x_1-\eta f'(x_1)-x_0}{\eta} =-f'(x_0)-f'(x_1) \\
    &= \left(2f''-\eta f'' f''\right)\epsilon \\
    &~ + \left(\frac{1}{(k-1)!}f\T{k}-\frac{1}{(k-1)!}\eta f''f\T{k}-\frac{1}{(k-1)!}(-1+\eta f'')^{k-1}\right)\epsilon^{k-1} + \O(\epsilon^k).
\end{align*}
When $\eta=\frac{2}{f''}$, it holds
\begin{align*}
    \frac{x_2-x_0}{\eta}=-\frac{2}{(k-1)!}f\T{k}\epsilon^{k-1}+\O(\epsilon^k).
\end{align*}

When $\eta=\frac{2}{f''+c\cdot f\T{k}\epsilon^{k-2}}$ with $c>0$ as some constant implying $\eta f''=2(1-c\frac{f\T{k}}{f''}\epsilon^{k-2})+\O(\epsilon^{2k-4})$, then it holds
\begin{align*}
    \frac{x_2-x_0}{\eta}=2\left(c-\frac{1}{(k-1)!}\right)f\T{k}\epsilon^{k-1} + \O(\epsilon^k),
\end{align*}
where we then set $c=1$.

Hence, the above two estimations of $\nicefrac{x_2-x_0}{\eta}$ have one positive and one negative with sufficiently small $|\epsilon|$.
Therefore, due to the continuity of $x_2-x_0$, there exists a learning rate $\eta\in(\frac{2}{f''}, \frac{2}{f''+f\T{k}\epsilon^{k-2}})$ such that $x_2=x_0$. 

The above proof can be generalized to any $x_0$ between $\Bar{x}$ and $\Bar{x}-\epsilon$ with the same bound for $\eta$.
\end{proof}

\begin{coro}
$f(x)=\sin(x)$ allows stable oscillation around its local minima $\Bar{x}$.
\label{cor:sin}
\end{coro}
\begin{proof}
Its lowest order nonzero derivative (expect $f''$) is $f\T{4}{\Bar{x}}=\sin(\Bar{x})=-1<0$ and the order $4$ is even. Then Lemma~\ref{lem:1dhigher} gives the result.
\end{proof}

\section{Proof of Prop~\ref{prop:l2loss}}
\label{app:l2loss}

\begin{prop}[Restatement of Prop~\ref{prop:l2loss}]
Consider a 1-D function $g(x)$ , and define the loss function $f$ as $f(x) = (g(x)-y)^2$. Assuming (i) $g'$ is not zero when $g(\Bar{x})=y$, (ii) $g'(\Bar{x})g\T{3}(\Bar{x}) < 6[g''(\Bar{x})]^2$, then it satisfies the condition in Theorem~\ref{thm:1dlocal} or Lemma~\ref{lem:1dhigher} to allow period-2 stable oscillation around $\Bar{x}$. 
\end{prop}

\begin{proof}
From the definition, we have
\begin{align}
    f''(x) &= 2[g(x)-y]g''(x) + 2[g'(x)]^2, \\
    f\T{3}(x) &= 2[g(x)-y]g\T{3}(x)+6 g''(x) g'(x),\\
    f\T{4}(x) &= 2[g(x)-y]g\T{4}(x)+6 g''(x)g''(x)+8g'(x)g\T{3}(x). \label{eq:temp6}
\end{align}
Then at the global minima where $g(x)=y$, we have $f''(x)=2[g'(x)]^2$ and $f\T{3}(x)=6g''(x)g'(x)$. If we assume $y$ is not a trivial value for $g(x)$, which means $g'(x)\neq 0$ at the minima, and $g$ is not linear around the minima (implies $g''\neq 0$), then $f$ satisfies $f\T{3}(\Bar{x})\neq 0$ in Theorem~\ref{thm:1dlocal}. Meanwhile, we need $3f\T{3}f\T{3}-f''f\T{4}>0$ as in Theorem~\ref{thm:1dlocal}, hence it requires
\begin{gather}
    \frac{1}{2g'(x)g'(x)}36g''(x)g''(x)g'(x)g'(x)-\frac{1}{3}\left(6g''(x)g''(x)+8g'(x)g\T{3}(x)\right)>0\\
    6g''(x)g''(x)>g'(x)g\T{3}(x). \label{eq:temp7}
\end{gather}

The remaining case is, if $g'(x)\neq 0$ and $g''=0$ at the minima, it satisfies the condition for Lemma~\ref{lem:1dhigher} with $k=4$, because $f\T{3}=0$ and $f\T{4}<0$ due to (\ref{eq:temp6}, \ref{eq:temp7})
\end{proof}

\begin{coro}
$f(x)=(x^2-1)^2$ allows stable oscillation around the local minima $\Bar{x}=1$.
\label{cor:x2}
\end{coro}
\begin{proof}
With $g(x)=x^2$, it has $g'(1)=2\neq0, g''(1)=2\neq0$. All higher order derivatives of $g$ are zero. Then Prop~\ref{prop:l2loss} gives the result.
\end{proof}

\begin{coro}
$f(x)=(\sin(x)-y)^2$ allows stable oscillation around the local minima $\Bar{x}=\arcsin(y)$ with $y\in(-1,1)$.
\label{cor:sin2}
\end{coro}
\begin{proof}
With $g(x)=\sin(x)$, it has $g'(\Bar{x})=\cos(\Bar{x})\neq 0, g\T{3}(\Bar{x})=-\cos(\Bar{x})$. We have $g\T{3}(\Bar{x})$ is bounded as
$g'g\T{3}-6[g'']^2=-\cos^2(\Bar{x})-6\sin^2(\Bar{x})<0.$ Then Prop~\ref{prop:l2loss} gives the result.
\end{proof}

\begin{coro}
$f(x)=(\tanh(x)-y)^2$ allows stable oscillation around the local minima $\Bar{x}=\tanh^{-1}(y)$ with $y\in(-1,1)$.
\label{cor:tanh2}
\end{coro}
\begin{proof}
With $g(x)=\tanh(x)$, it has $g'(\Bar{x})=\text{sech}^2(\Bar{x})\neq 0,
$ and $
g\T{3}(\Bar{x})=-2 \text{sech}^4(\Bar{x}) + 4 \text{sech}^2(\Bar{x}) \tanh^2(\Bar{x})$ is bounded as
\begin{align*}
    g'g\T{3}-6[g'']^2=-2\text{sech}^6+4\text{sech}^4\tanh^2-24\text{sech}^4\tanh^2=-2\text{sech}^6-20\text{sech}^4\tanh^2<0.
\end{align*}
Then Prop~\ref{prop:l2loss} gives the result.
\end{proof}
\begin{coro}
    $f(x)=(x^{\alpha}-y)^2$ (with $k\in\mathbb{Z}, k\ge 2$) allows stable oscillation around the local minima $\Bar{x}=y^{\nicefrac{1}{\alpha}}$ except $y=0$.
    \label{cor:mono}
\end{coro}
\begin{proof}
    With $g(x)=x^{\alpha}$, it has $g'(\Bar{x})=\alpha x^{\alpha-1}, g''(\Bar{x})=\alpha(\alpha-1)x^{\alpha-2}, g^{(3)}(\Bar{x})=\alpha(\alpha-1)(\alpha-2)x^{\alpha-3}$. Then we have $g'g\T{3}-6[g'']^2=\alpha^2(\alpha-1)(-5\alpha+4) x^{2\alpha-4}<0.$ Then Prop~\ref{prop:l2loss} gives the result.
\end{proof}
\begin{coro}
    $f(x)=(\exp(x)-y)^2$ allows stable oscillation around the local minima $\Bar{x}=\log{y}$ for $y>0$.
    \label{cor:exp}
\end{coro}
\begin{proof}
    With $g(x)=\exp{x}$, it has $g'(\Bar{x})=g''(\Bar{x})=g^{(3)}(\Bar{x})=\exp(\Bar{x})$. Then we have $g'g\T{3}-6[g'']^2<0.$ Then Prop~\ref{prop:l2loss} gives the result.
\end{proof}
\begin{coro}
    $f(x)=(\log(x)-y)^2$ allows stable oscillation around the local minima $\Bar{x}=\exp{y}$.
    \label{cor:log}
\end{coro}
\begin{proof}
    With $g(x)=log{x}$, it has $g'(\Bar{x})=\frac{1}{\Bar{x}}, g''(\Bar{x})=-\frac{1}{\Bar{x}^2}, g^{(3)}(\Bar{x})=-\frac{2}{\Bar{x}^3}$. Then we have $g'g\T{3}-6[g'']^2<0.$ Then Prop~\ref{prop:l2loss} gives the result.
\end{proof}
\begin{coro}
    $f(x)=(\frac{1}{1+\exp(-x)}-y)^2$ allows stable oscillation around the local minima $\Bar{x}=\text{sigmoid}^{-1}(y)$ for $y\in(0,1)$.
    \label{cor:sigmoid}
\end{coro}
\begin{proof}
    With $g(x)=\frac{1}{1+\exp(-x)}$, it has $g'(\Bar{x})=\frac{\exp(-x)}{(\exp(-x)+1)^2}, g''(\Bar{x})=-\frac{\exp(x)(\exp(x)-1)}{(\exp(x)+1)^3}, g^{(3)}(\Bar{x})=\frac{\exp(x)(-4\exp(x)+\exp(2x)+1)}{(\exp(x)+1)^4}$. Then we have $g'g\T{3}-6[g'']^2 \propto -4\exp(x)+\exp(2x)+1-6(\exp(x)-1)^2<0.$ Then Prop~\ref{prop:l2loss} gives the result.
\end{proof}
\begin{prop}[Restatement of Prop~\ref{prop:comp}]
    Consider two functions $f,g$. Assume both $f(x),g(y)$ at $x=\Bar{x}, y=f(\Bar{x})$ satisfies the conditions in Prop~\ref{prop:l2loss} to allow stable oscillations. Then $g(f(x))$ allows stable oscillation around $x=\Bar{x}$.
\end{prop}
\begin{proof}
    Denote $F(x)\triangleq g(f(x))$. Then we have
    \begin{align*}
        F'(x)&=g'(f(x))f'(x), \\
        F''(x)&= g''(f(x))[f'(x)]^2 + g'(f(x))f''(x), \\
        F^{(3)}(x)&= g^{(3)}(f(x))[f'(x)]^3+3g''(f(x))f'(x)f''(x)+g'(f(x))f^{(3)}(x).
    \end{align*}
    Thus, omitting all variables $\Bar{x}$ and $f(\Bar{x})$ in the derivatives, it holds
    \begin{align*}
        F'(\Bar{x})F^{(3)}(\Bar{x}) - 6[F''(\Bar{x})]^2 &= g'f'\left(g^{(3)}(f')^3 + 3g''f'f''+g'f^{(3)}\right)-6\left(g''(f')^2+g'f''\right)^2 \\
        &\le -9 g'g''(f')^2 f'',
    \end{align*}
    where the inequality is due to all conditions in Prop~\ref{prop:l2loss}. So the only problem is whether we can achieve $g'g'' f''>0$. The good news is that, even if it holds $g'g''f''<0$, we can still find functions to re-represent $g(f(x))$ as $\hat{g}(\hat{f}(x))$ such that $\hat{g}'\hat{g}''\hat{f}''<0$ and all other conditions in Prop~\ref{prop:l2loss} are satisfied by $\hat{g},\hat{f}$.

    For $g'g''f''<0$, construct $\hat{g}(y)\triangleq g(-y), \hat{f}(x)\triangleq -f(x)$. In this sense, it holds $\hat{g}(\hat{f}(\Bar{x}))=g(f(\Bar{x}))$. It is easy to verify that both $\hat{g}, \hat{f}$ at $y=-f(\Bar{x}),x=\Bar{x}$ satisfy the conditions in Prop~\ref{prop:l2loss}, because 
    \begin{gather*}
        \hat{g}'(y)=-g'(-y)=-g'(f(\Bar{x})), ~~ \hat{g}''(y)=g''(-y)=g''(f(\Bar{x})),~~ \hat{g}^{(3)}(y)=-g^{(3)}(-y)=-g^{(3)}(f(\Bar{x})),\\
        \hat{f}'(\Bar{x})=-f'(\Bar{x}),~~ \hat{f}''(\Bar{x})=-f''(\Bar{x}),~~ \hat{f}^{(3)}(y)=-f^{(3)}(\Bar{x}).
    \end{gather*}
    Then, it has $\hat{g}'(y)\hat{g}''(y)\hat{f}''(x)=-g'g''f''>0$ at $y=-f(\Bar{x}),x=\Bar{x}$. Therefore, we have $F'(\Bar{x})F^{(3)}(\Bar{x}) - 6[F''(\Bar{x})]^2<0$ and Prop~\ref{prop:l2loss} gives the result.
    \end{proof}


\section{Proof of Theorem~\ref{thm:1dglobal}}
\label{app:1dglobal}

\begin{theorem}[Restatement of Theorem~\ref{thm:1dglobal}]
For $f(x)=\frac{1}{4}(x^2-\mu)^2$, consider GD with $\eta=K\cdot\frac{1}{\mu}$ where $1<K< \sqrt{4.5}-1\approx 1.121$, and initialized on any point $0<x_0<\sqrt{\mu}$. Then it converges to an orbit of period 2, except for a measure-zero initialization where it converges to $\sqrt{\mu}$. More precisely, the period-2 orbit are the solutions $x=\delta_1\in(0,\sqrt{\mu}), x=\delta_2\in(\sqrt{\mu},2\sqrt{\mu})$ of solving $\delta$ in
\begin{align}
    \eta  = \frac{1}{\delta^2\left( \sqrt{\frac{\mu}{\delta^2}-\frac{3}{4}} +\frac{1}{2}\right)}.
\end{align}
\end{theorem}

\begin{proof}
Assume the 2-period orbit is $(\Bar{x}_0, \Bar{x}_1)$, which means
\begin{align*}
    \Bar{x}_1 &= \Bar{x}_0 - \eta\cdot f'(\Bar{x}_0)
    =\Bar{x}_0 + \eta\cdot(\mu-\Bar{x}_0^2)\Bar{x}_0
    ,\\
    \Bar{x}_0 &= \Bar{x}_1 - \eta\cdot f'(\Bar{x}_1)
    =\Bar{x}_1 + \eta\cdot(\mu-\Bar{x}_1^2)\Bar{x}_1.
\end{align*}

First, we show the existence and uniqueness of such an orbit when $K\in(1,1.5]$ via solving a high-order equation, some roots of which can be eliminated. 
Then, we conduct an analysis of global convergence by defining a special interval $I$. GD starting from any point following our assumption will enter $I$ in some steps, and any point in $I$ will back to this interval after two steps of iteration. Finally, any point in $I$ will converge to the orbit $(\Bar{x}_0, \Bar{x}_1)$.

Before diving into the proof, we briefly show it always holds $x>0$ under our assumption. If $x_{t-1}>0$ and $x_{t}\le 0$, the GD rule reveals $\eta(\mu-x_{t-1}^2)\le -1$ which implies $x_{t-1}^2 \ge \mu+\frac{1}{\eta}$. However, the maximum of $x+\eta(\mu-x^2)x$ on $x\in(0, \sqrt{\mu+\frac{1}{\eta}})$ is achieved when $x^2=\frac{1}{3}(\mu+\frac{1}{\eta})$ so the maximum value is $\sqrt{\frac{1}{3}(\mu+\frac{1}{\eta})}(\frac{2}{3}+\frac{2}{3}\eta\mu)\le 1.4\sqrt{\frac{1}{3}(\mu+\frac{1}{\eta})}<\sqrt{\mu+\frac{1}{\eta}}$. As a result, it always holds $x>0$.


\textbf{Part I. Existence and uniqueness of $(\Bar{x}_0, \Bar{x}_1)$}.

In this part, we simply denote both $\Bar{x}_0, \Bar{x}_1$ as $x_0$. This means $x_0$ in all formulas in this part can be interpreted as $\Bar{x}_0$ and $\Bar{x}_1$. Then the GD update rule tells, for the orbit in two steps,
\begin{align*}
    x_0
    &\mapsto
    x_1 \coloneqq
    x_0+\eta(\mu-x_0^2) x_0,\\
    x_1
    &\mapsto
    x_0=
    x_1+\eta(\mu-x_1^2) x_1,
\end{align*}
which means
\begin{align*}
    0=\eta(\mu-x_0^2)x_0+\eta\left(\mu-\left(x_0+\eta(\mu-x_0^2)x_0\right)^2\right)\left(x_0+\eta(\mu-x_0^2)x_0\right),\\
    0=\mu-x_0^2+\left(\mu-\left(x_0+\eta(\mu-x_0^2)x_0\right)^2\right)\left(1+\eta(\mu-x_0^2)\right).
\end{align*}
Denote $z\coloneqq 1+\eta(\mu-x_0^2)$, it is equivalent to
\begin{align*}
    0 &= \mu-x_0^2+(\mu-z^2 x_0^2)z=(z+1)(-x_0^2 z^2+x_0^2z+\mu-x_0^2)\\
    &= (z+1)\left(-x_0^2(z-\frac{1}{2})^2+\mu-\frac{3}{4}x_0^2 \right).
\end{align*}
If $z+1=0$, it means $x_1 = -x_0$ which is however out of the range of our discussion on the $x>0$ domain. So we require $-x_0^2(z-\frac{1}{2})^2+\mu-\frac{3}{4}x_0^2=0$. To ensure the existence of solutions $z$, it is natural to require
\begin{align*}
    \mu-\frac{3}{4}x_0^2 \ge 0.
\end{align*}
Then, the solutions are
\begin{align*}
    z=\frac{1}{2}\pm\sqrt{\frac{\mu}{x_0^2}-\frac{3}{4}}.
\end{align*}
However, $z=\frac{1}{2}-\sqrt{\frac{\mu}{x_0^2}-\frac{3}{4}}$ can be ruled out. If it holds, $\eta(\mu-x_0^2)=z-1<-\frac{1}{2}$ which means $x_0^2>\mu+\frac{1}{2\eta}$. Since we restrict $\eta\mu\in(1,1.121],$ it tells $x_0^2>\mu(1+\frac{1}{1.242})$ contradicting with $\mu\ge\frac{3}{4}x_0^2$.

Hence, $z=\frac{1}{2}+\sqrt{\frac{\mu}{x_0^2}-\frac{3}{4}}$ is the only reasonable solution, which is saying
\begin{align*}
    \eta(\mu-x_0^2)=-\frac{1}{2}+\sqrt{\frac{\mu}{x_0^2}-\frac{3}{4}}.
\end{align*}

Given a certain $\eta$, the above expression is a third-order equation of $x_0^2$ to solve. Apparently $x_0^2=\mu$ is one trivial solution, since for any learning rate, the gradient descent stays at the global minimum. Then the two other solutions are exactly the orbit $(\Bar{x}_0, \Bar{x}_1)$, if the equation does have three different roots. This also guarantees the uniqueness of such an orbit.

Assuming $x_0^2\neq\mu$, the above expression can be reformulated as
\begin{align}
    \eta=\frac{1}{x_0^2\left( \sqrt{\frac{\mu}{x_0^2}-\frac{3}{4}}+\frac{1}{2} \right)}.
    \label{eq:eta_x0}
\end{align}

One necessary condition for existence is $\mu\ge \frac{3}{4}x_0^2$. Note that here $x_0$ can be both $\Bar{x}_0, \Bar{x}_1$, one of which is larger than $\sqrt{\mu}$. For simplicity, we assume $\Bar{x}_0 <\sqrt{\mu}< \Bar{x}_1$. Since $\eta$ from Eq(\ref{eq:eta_x0}) is increasing with $x_0^2$ when $\mu<x_0^2$, let $x_0^2=\frac{4}{3}\mu$ and achieve the upper bound as
\begin{align}
    \eta\mu \le \frac{3}{2}, \label{eq:eta_exist_bound}
\end{align}
which is satisfied by our assumption $1<\eta\mu< \sqrt{4.5}-1\approx 1.121.$

Therefore, we have shown the existence and uniqueness of a period-2 orbit.

\textbf{Part II. Global convergence to $(\Bar{x}_0, \Bar{x}_1)$.}

The proof structure is as follows:
\begin{enumerate}
    \item There exists a special interval $I\coloneqq[x_s, \sqrt{\mu})$ such that any point in $I$ will back to this interval surely after two steps of gradient descent. And $\Bar{x}_0\in I$.
    \item Initialized from any point in $I$, the gradient descent process will converge to $\Bar{x}_0$ (every two steps of GD).
    \item Initialized from any point between 0 and $\sqrt{\mu}$, the gradient descent process will fall into $I$ in some steps.
\end{enumerate}

\textbf{(II.1)} Consider a function $F_{\eta}(x)=x+\eta(\mu-x^2)x$ performing one step of gradient descent. Since $F_{\eta}'(x)=1+\eta\mu-3\eta x^2$, we have $F'_\eta(x)>0$ for $0<x^2<\frac{1}{3}\left(\mu+\frac{1}{\eta}\right)$ and $F'_\eta(x)<0$ otherwise. It is obvious that the threshold has $x_s^2\coloneqq\frac{1}{3}\left(\mu+\frac{1}{\eta}\right)<\mu$. In the other words, for any point on the right of $x_s$, GD returns a point in a decreasing manner.

To prove anything further, we would like to restrict $\Bar{x}_0 \ge x_s$, which is 
\begin{align*}
    \Bar{x}_0^2\ge \frac{1}{3}\left(\mu+\frac{1}{\eta}\right)
    =\frac{1}{3}
    \left(\mu 
    +
    \Bar{x}_0^2\left( \sqrt{\frac{\mu}{\Bar{x}_0^2}-\frac{3}{4}}+\frac{1}{2} \right)
    \right).
\end{align*}
Solving this inequality tells 
\begin{align}
    \Bar{x}_0^2\ge\frac{3+\sqrt{2}}{7}\mu.
\end{align}
Consequently, by applying Eq(\ref{eq:eta_x0}), we have
\begin{align}
    \eta\mu\le\sqrt{4.5}-1\approx 1.121.
\end{align}

With the above discussion of $x_s$, we are able to define the special internal $I\coloneqq[x_s,\sqrt{\mu})$. From the definition of $F_\eta$, consider a function representing two steps of gradient descent $F^2_\eta(x)\coloneqq F_\eta(F_\eta(x))$. From previous discussion, we know $F^2_\eta(\Bar{x}_0)=\Bar{x}_0$. What about $F_{\eta}^2(x_s)$?

It turns out $F_{\eta}^2(x_s)>x_s$: we have $F_{\eta}(x_s)=x_s(1+\eta\mu-\eta x_s^2) = x_s\cdot \frac{2}{3}(1+\eta\mu)$ and, furthermore, $F^2_\eta(x_s)=F_\eta(x_s\cdot \frac{2}{3}(1+\eta\mu))=x_s\cdot \frac{2}{3}(1+\eta\mu)\cdot\left(1+\eta\mu-\frac{4}{27}(1+\eta\mu)^3\right)$. Then we get $F^2_\eta(x_s)>x_s$ because
\begin{align}
    \frac{2}{3}(1+\eta\mu)\cdot\left(1+\eta\mu-\frac{4}{27}(1+\eta\mu)^3\right) > 1 \text{~~~if~~~~} \eta\mu\in(1,\sqrt{4.5}-1).
\end{align}
Combining the following facts, i) $F^2_\eta(x)-x$ is continous wrt $x$, ii) $F^2_\eta(x_s)-x_s>0$, and iii) $F^2_\eta(\Bar{x}_0)-\Bar{x}_0=0$ is the only zero point on $x\in[x_s, \Bar{x}_0]$, we can conclude that 
\begin{align}
    F^2_\eta(x)>x, ~~\forall x \in[x_s,\Bar{x}_0).
\end{align}

Meanwhile, we can prove $F^2_\eta(x)<x$ for any $x\in(\Bar{x}_0,\sqrt{\mu})$. Since $F^2_\eta(\mu)-\mu=0$ and $F^2_\eta(\Bar{x}_0)-\Bar{x}_0=0$ are the only two zero cases, we only need to show $\exists ~\hat{x}\in (\Bar{x}_0,\sqrt{\mu}),$ such that $F^2_\eta(\hat{x})<\hat{x}$. We compute the derivative of $F^2_\eta(x)-x$ at $x^2=\mu$, which is $\frac{d}{dx}F^2_\eta(x)-x|_{x^2=\mu}=-1+F'(F(x))F'(x)|_{x^2=\mu}=-1+[F'(\sqrt{\mu})]^2=-1+(1-2\eta\mu)^2>0$. Then combining it with $F^2_\eta(\Bar{x}_0)=\Bar{x}_0$, there exists a point $\hat{x}\in(\Bar{x}_0,\sqrt{\mu})$ that is very close to $\sqrt{\mu}$ such that $F^2_\eta(\hat{x})<\hat{x}$. Hence, we can conclude that 
\begin{align}
    F^2_\eta(x)<x, ~~\forall x \in(\Bar{x}_0,\sqrt{\mu}).
\end{align}




Since $F_\eta(\cdot)$ is decreasing on $[x_s,\infty)$ and $F_\eta(x)>x_s$ for $x\in[x_s,\sqrt{\mu}]$, it is fair to say $F^2_\eta(x)$ is increasing on $x\in[x_s,\sqrt{\mu}]$. Hence, we have $F^2_\eta(x)\le F^2_\eta(\Bar{x}_0)=\Bar{x}_0,~\forall x \in[x_s,\Bar{x}_0]$. And $F^2_\eta(x)\ge F^2_\eta(\Bar{x}_0)=\Bar{x}_0,~\forall x (\Bar{x}_0,\sqrt{\mu})$

Combining the above results, we have
\begin{align}
     F^2_\eta(x)&\in(x,\Bar{x}_0], ~~\forall x \in[x_s,\Bar{x}_0),
     \label{eq:f2_x0_left}
     \\
     F^2_\eta(x)&\in[\Bar{x}_0,x), ~~\forall x \in(\Bar{x}_0,\sqrt{\mu}).
     \label{eq:f2_x0_right}
\end{align}

\textbf{(II.2)} A consequence of Exp(\ref{eq:f2_x0_left}, \ref{eq:f2_x0_right}) is that any point in $I$ will converge to $\Bar{x}_0$ with even steps of gradient descent. For simplicity, we provide the proof for $x\in[x_s,\Bar{x}_0)$.

Denote $a_0\in[x_s,\Bar{x}_0)$ and $a_n\coloneqq F^2_\eta(a_{n-1}), n\ge 1$. The series $\{a_i\}_{i\ge0}$ satisfies
\begin{align}
    \Bar{x}_0\ge a_{n+1} >a_n>a_0.
\end{align}

Since the series is bounded and strictly increasing, it is converging. Assume it is converging to $a$. If $a<\Bar{x}_0$, then
\[
\Bar{x}_0\ge F^2_\eta(a)>a>F^2_\eta(a_n).
\]
Since $F^2_\eta(\cdot)$ is continuous, so $\exists~ \delta>0$, such that, when $|x-a|<\delta$, we have
\begin{align}
    |F^2_\eta(x)-F^2_\eta(a)|<F^2_\eta(a)-a.
    \label{eq:f2_contin}
\end{align}
Since $a$ is the limit, so $\exists~ N>0$, such that, when $n>N$, $0<a-F^2_\eta(a_n)<\delta$. So, combining with Exp(\ref{eq:f2_contin}), we have
\[
|F^2_\eta(F^2_\eta(a_n))-F^2_\eta(a)|<F^2_\eta(a)-a.
\]
But LHS $=F^2_\eta(a)-a_{n+2}>F^2_\eta(a)-a$, so we reach a contradiction.

Hence, we have $\{a_i\}$ converges to $\Bar{x}_0$.

\textbf{(II.3)} Obviously, any initialization in $(0,\sqrt{\mu})$ will have gradient descent run into (i) the interval $I$, or (ii) the interval on the right of $\sqrt{\mu}$, $\textit{i.e.}$, $(\sqrt{\mu},\infty)$. The first case is exactly our target.

Now consider the second case. From the definition of $x_s$ in part III.1,  we know $F_\eta(x_s)=\max_{x\in [0,\sqrt{\mu}]} F_\eta(x)$. So it is fair to say this case is $x_n\in(\sqrt{\mu}, F_\eta(x_s)]$. Then the next step will go into the interval $I$, because
\begin{align*}
    F_\eta(x_n)\ge F_\eta(F_\eta(x_s))=F^2_\eta(x_s)>x_s,
\end{align*}
where the first inequality is from the decreasing property of $F_\eta(\cdot)$ and the second inequality is due to $F^2_\eta(x)>x$ on $x\in[x_s,\Bar{x}_0)$.
\end{proof}

\section{Proof of Theorem~\ref{thm:xy_diff_decay}}
\label{app:xy_diff_decay}

\begin{theorem}[Restatement of Theorem~\ref{thm:xy_diff_decay}]
For $f(x,y)=\frac{1}{2}\left(xy-\mu\right)^2$, consider GD with learning rate $\eta=K\cdot\frac{1}{\mu}$. Assume both $x$ and $y$ are always positive during the whole process $\{x_i,y_i\}_{i\ge 0}$. In this process, denote a series of all points with $xy>\mu$ as $\mathcal{P}=\{(x_i,y_i)|x_i y_i>\mu\}$. Then $|x-y|$ \red{decays to 0} in $\mathcal{P}$, for any $1<K<1.5$.
\end{theorem}

\begin{proof}
Consider the current step is at $(x_t,y_t)$ with $x_t y_t>\mu$. After two steps of gradient descent, we have
\begin{align}
    x_{t+1} &= x_t + \eta(\mu-x_t y_t)y_t \label{eq:ite_xtp1}\\
    y_{t+1} &= y_t + \eta(\mu-x_t y_t)x_t
    \label{eq:ite_ytp1}\\
    x_{t+2} &= x_{t+1} + \eta(\mu-x_{t+1} y_{t+1})y_{t+1}\\
    y_{t+2} &= y_{t+1} + \eta(\mu-x_{t+1} y_{t+1})x_{t+1},
\end{align}
with which we have the difference evolve as
\begin{align}
    y_{t+1}-x_{t+1}
    &=
    \left(y_t-x_t\right)
    \left(1-\eta\left(\mu-x_t y_t\right)\right)
    \label{eq:diff_evolv_t}
    \\
    y_{t+2}-x_{t+2}
    &=
    \left(y_{t+1}-x_{t+1}\right)
    \left(1-\eta\left(\mu-x_{t+1} y_{t+1}\right)\right)
    \label{eq:diff_evolv_t1}.
\end{align}
Meanwhile, we have
\begin{align}
    x_{t+1}y_{t+1} &=x_t y_t+\eta\left(\mu-x_t y_t\right)\left(x_t^2+y_t^2\right)+\eta^2\left(\mu-x_t y_t\right)^2 x_t y_t
    \nonumber
    \\
    &=
    x_t y_t\left(1+\eta\left(\mu-x_t y_t\right)\right)^2 + \eta\left(\mu-x_t y_t\right)\left(x_t-y_t\right)^2
    \label{eq:xy_evolv}
\end{align}
Note that the second term in Eq(\ref{eq:xy_evolv}) vanishes when $x$ and $y$ are balanced. When they are not balanced, if $x_t y_t>\mu$, it holds $x_{t+1}y_{t+1}<x_t y_t\left(1+\eta\left(\mu-x_t y_t\right)\right)^2$. Incorporating this inequality into Eq(\ref{eq:diff_evolv_t}, \ref{eq:diff_evolv_t1}) and assuming $y_t-x_t>0$, it holds
\begin{align}
    y_{t+2}-x_{t+2}<
    \left(y_t - x_t\right)
    \left(1-\eta\left(\mu-x_t y_t\right)\right)
    \left(1-\eta\left(\mu-x_t y_t\left(1+\eta\left(\mu-x_t y_t\right)\right)^2\right)\right).
    \label{eq:diff_evolv_t2}
\end{align}

To show that $|x-y|$ is decaying as in the theorem, we are to show
\begin{enumerate}
    \item $y_{t+2}-x_{t+2}<y_t-x_t$
    \item $y_{t+2}-x_{t+2}>-(y_t-x_t)$
\end{enumerate}
Note that, although $x_t y_t>\mu$, it is not sure to have $x_{t+2} y_{t+2}>\mu$. However, for any $0<x_i y_i<\mu$ and $K<2$, we have
\begin{align}
    \frac{|x_{i+1}-y_{i+1}|}{|x_i-y_i|}=\left| 1-\eta\left(\mu-x_i y_i\right)\right| <1, \label{eq:temp4}
\end{align}
which is saying $|x-y|$ decays until it reaches $xy>\mu$. So it is enough to prove the above two inequalities, whether or not $x_{t+2} y_{t+2}>\mu$.

\textbf{Part I. To show $y_{t+2}-x_{t+2}<y_t-x_t$}

Since we wish to have $y_{t+2}-x_{t+2}<y_t-x_t$, it is sufficient to require
\begin{align}
    \left(1-\eta\left(\mu-x_t y_t\right)\right)
    \left(1-\eta\left(\mu-x_t y_t\left(1+\eta\left(\mu-x_t y_t\right)\right)^2\right)\right)<1. \label{eq:xybalance}
\end{align}


\red{
Since we assume $x_{t+1}, y_{t+1}>0$, Eq (\ref{eq:ite_xtp1}, \ref{eq:ite_ytp1}) tells $\eta\left(\mu-x_t y_t\right)>-\min\{\frac{x_t}{y_t}, \frac{y_t}{x_t}\}$, which is equivalent to $1-\eta\left(\mu-x_t y_t\right)<1+\min\{\frac{x_t}{y_t}, \frac{y_t}{x_t}\}$.}

\textbf{(I.1) If $\eta(\mu-x_{t+1}y_{t+1})\ge\frac{1}{2}$}

Then we have $1-\eta(\mu-x_{t+1}y_{t+1})\le\frac{1}{2}$. As a result, 
\begin{align}
    \frac{y_{t+2}-x_{t+2}}{y_t-x_t}&=\left(1-\eta\left(\mu-x_t y_t\right)\right)\left(1-\eta\left(\mu-x_{t+1} y_{t+1}\right)\right)< \red{\left(1+\min\{\frac{x_t}{y_t}, \frac{y_t}{x_t}\}\right)\times\frac{1}{2}} \\
    &\red{=\frac{1}{2}+\frac{1}{2}\min\{\frac{x_t}{y_t}, \frac{y_t}{x_t}\}}
    \label{eq:xydiff_rate1}
\end{align}

\textbf{(I.2) If $\eta(\mu-x_{t+1}y_{t+1})<\frac{1}{2}$ and $x_{t+1} y_{t+1}\le x_s^2=\frac{1}{3}\left(\mu+\frac{1}{\eta}\right)$}

The second condition reveals
\begin{align}
    \frac{y_{t+2}-x_{t+2}}{y_{t+1}-x_{t+1}}
    &=1-\eta\left(\mu-x_{t+1} y_{t+1}\right)
    \le 1-\eta\left(\mu-\frac{1}{3}\left(\mu+\frac{1}{\eta}\right)\right)
    \nonumber
    \\
    &=\frac{4}{3}-\frac{2}{3}\red{K}.
    \label{eq:diff_ratio_t2_bound}
\end{align}

The first condition is equivalent to $x_{t+1} y_{t+1}>\mu-\frac{1}{2\eta}$. Since the second term in Eq(\ref{eq:xy_evolv}) is negative, we have
\begin{align}
    x_t y_t\left(1+\eta\left(\mu-x_t y_t\right)\right)^2>\mu-\frac{1}{2\eta},
    \label{eq:xy_upbound}
\end{align}
with which we would like to find an upper bound of $x_t y_t$.

Denoting $b=x_t y_t$, consider a function $q(b)=b\left(1+\eta\left(\mu-b\right)\right)^2$. Obviously $q(\mu)=\mu$. Its derivative is $q'(b)=\left(1+\eta\mu-
\eta b\right)\left(1+\eta\mu-3\eta b\right)<0$ on the domain of our interest. If we can show an (negative) upper bound for the derivative as $q'(b)<-1$ on a proper domain, then it is fair to say that, from Exp(\ref{eq:xy_upbound}), $x_t y_t<\mu+\frac{1}{2\eta}$. Then we have
\begin{align}
    \frac{y_{t+1}-x_{t+1}}{y_t-x_t}=1-\eta(\mu-x_t y_t)<1-\eta\left(\mu-\left(\mu-\frac{1}{2\eta}\right)\right)=\frac{3}{2}.
    \label{eq:diff_ratio_t1_bound}
\end{align}
Then, combining Exp(\ref{eq:diff_ratio_t1_bound}, \ref{eq:diff_ratio_t2_bound}), it tells
\begin{align}
    \frac{y_{t+2}-x_{t+2}}{y_{t}-x_{t}}<\red{2-K}.
    \label{eq:xydiff_rate2}
\end{align}

The remaining is to show $q'(b)<-1$ on a proper domain. We have $q'(b)=(1+\eta\mu-2\eta b)^2-(\eta b)^2$, which is equal to $1-2\eta\mu<-1$ when $b=\mu$. Meanwhile, the derivative of $q'(b)$ is $q''(b)=-2\eta(\eta b+(1+\eta\mu-2\eta b))=-2\eta(1+\eta\mu-\eta b)$, which is negative when $b<\mu+\frac{1}{\eta}$. As a result, it always holds $q'(b)<-1$ when $b<\mu+\frac{1}{\eta}$.

\textbf{(I.3) If $x_{t+1} y_{t+1}\ge x_s^2$}

Denoting again $b=x_t y_t$, the above inequality in  is saying, with $b>\mu$,
\begin{align}
    p(b)=\left(1-\eta\left(\mu-b\right)\right)
    \left(1-\eta\left(\mu-b\left(1+\eta\left(\mu-b\right)\right)^2\right)\right)<1.
\end{align}
After expanding $p(\cdot)$, we have
\begin{align*}
    p(b)-1=\eta \left(\mu-b\right)
    \left(
    -2+\eta\left(\mu-b\right)+2\eta b-\eta^2 b\left(\mu-b\right)-\eta^3 b\left(\mu-b\right)^2
    \right).
\end{align*}
Apparently $p(\mu)=1$. So it is necessary to investigate whether $p'(b)<0$ on $b>\mu$, as
\begin{align*}
    p'(b)=2-2\eta b+\left(\mu-b\right)
    \left(
    \eta^2\left(1+\eta\left(\mu-b\right)\right)\left(-\mu+3b\right)+\eta^3 b\left(\mu-b\right)
    \right).
\end{align*}
Since $\eta b>1$ and $b>\mu$, it is enough to require
\begin{gather*}
    \left(1+\eta\left(\mu-b\right)\right)\left(-\mu+3b\right)+\eta b\left(\mu-b\right)>0\\
    (1+\eta(\mu-b))(-\mu+b)+\eta b(\mu-b)+2b(1+\eta(\mu-b))>0.
\end{gather*}
It suffices to show
\begin{align}
    \eta (\mu-b)+2(1+\eta(\mu-b))=2+3\eta(\mu-b)>0.
\end{align}
Since $x_{t+1} y_{t+1}\ge x_s^2=\frac{1}{3}\left(\mu+\frac{1}{\eta}\right)$, it holds
\begin{align*}
    b\left(1+\eta(\mu-b)\right)^2
    &\ge\frac{1}{3}\left(\mu+\frac{1}{\eta}\right) \\
    2+3\eta(\mu-b)
    &\ge \sqrt{\frac{3\left(\mu+\frac{1}{\eta}\right)}{b}}-1>0,
\end{align*}
where the last inequality holds because: if $b\ge 3\left(\mu+\frac{1}{\eta}\right)$, then $1+\eta(\mu-b)\le -2\eta\mu-2<0$, which contradicts with the assumption that both $x_{t+1}, y_{t+1}$ are positive. \red{As a result, the above argument gives

\begin{align}
    \frac{y_{t+2}-x_{t+2}}{y_t-x_t} < p(b)< 1-2(K-1)(b-\mu).
    \label{eq:xydiff_rate3}
\end{align}
}

\textbf{Part II. To show $y_{t+2}-x_{t+2}>-(y_t-x_t)$}

Since $x_t y_t>\mu$, we have $1-\eta(\mu-x_t y_t)>1$. Combining with $1-\eta(\mu-x_t y_t)<2$, it holds
\begin{align*}
    \frac{y_{t+1}-x_{t+1}}{y_t-x_t}=1-\eta(\mu-x_t y_t)\in(1,2).
\end{align*}
So the remaining is to have $\frac{y_{t+2}-x_{t+2}}{y_{t+1}-x_{t+1}}>-0.5$. Actually it is $1-\eta(\mu-x_{t+1}y_{t+1})\ge 1-\eta\mu=\red{1-K}$. 
\red{
Therefore, we have 
\begin{align}
    \frac{y_{t+2}-x_{t+2}}{y_t-x_t}> -1 +(3-2K),
    \label{eq:xydiff_rate4}
\end{align}
as required.
}

\red{
\textbf{Part III. To show $y_{t}-x_{t}$ converges to 0}

From Exp (\ref{eq:xydiff_rate1}, \ref{eq:xydiff_rate2}, \ref{eq:xydiff_rate3}, \ref{eq:xydiff_rate4}),
we have for points in $\mathcal{P}$, $|y-x|$ is a monotone strictly decreasing sequence lower bounded by 0. Hence it is convergent. Actually it converges to 0. If not, assuming it converges to $\epsilon>0$, the next point will have the difference as $\tilde{\epsilon}<\epsilon$ as well as all following points. Hence, the contradiction gives the convergence to 0.
}
\end{proof}

\section{Proof of Lemma~\ref{thm:xy_positive}}
\label{app:xy_positive}

\begin{lemma}[Restatement of Lemma~\ref{thm:xy_positive}]
In the setting of Theorem~\ref{thm:xy_diff_decay}, denote the initialization as $m=\frac{|y_0-x_0|}{\sqrt{\mu}}$ and $x_0 y_0>\mu$. Then, during the whole process, both $x$ and $y$ will always stay positive, denoting $p=\frac{4}{\left( m+\sqrt{m^2+4} \right)^2}$ and $q=(1+p)^2$, if
\begin{align*}
    \max\left\{
    \eta(x_0 y_0-\mu),
    \frac{4}{27}\left(1+K\right)^3+\left(\frac{2}{3}K^2-\frac{1}{3}K+\frac{q K^2}{2(K+1)}m^2\right)q m^2-K
    \right\}
    <
    p.
\end{align*}
\end{lemma}

\begin{proof}
Considering $x_t y_t>\mu$, one step of gradient descent returns
\begin{align*}
    x_{t+1} &= x_t + \eta(\mu-x_t y_t)y_t\\
    y_{t+1} &= y_t + \eta(\mu-x_t y_t)x_t.
\end{align*}
To have both $x_{t+1}>0, y_{t+1}>0$, it suffices to have
\begin{align}
    \eta(x_t y_t-\mu)<\min\left\{ \frac{y_t}{x_t}, \frac{x_t}{y_t} \right\}.
    \label{eq:cond_xy_pos}
\end{align}
This inequality will be the main target we need to resolve in this proof.

First, we are to show
\begin{align*}
    \min\left\{ \frac{y_0}{x_0}, \frac{x_0}{y_0} \right\}>\frac{4}{\left( m+\sqrt{m^2+4} \right)^2}.
\end{align*}
With the difference fixed as $m=(y_0-x_0)/\sqrt{\mu}$, assuming $y_0>x_0$, we have $m/y_0=(1-x_0/y_0)/\sqrt{\mu}$. if $x_0 y_0$ increases, both $x_0$ and $y_0$ increase then $m/y_0$ decreases, which means $x_0/y_0$ increases. As a result, we have
\begin{align*}
  \min\left\{ \frac{y_0}{x_0}, \frac{x_0}{y_0} \right\} 
  >
  \min\left\{ \frac{y_0}{x_0}, \frac{x_0}{y_0} \right\} \Bigg|_{x_0 y_0=\mu}=\frac{4}{\left( m+\sqrt{m^2+4} \right)^2}.
\end{align*}
Therefore, at initialization, to have positive $x_1$ and $y_1$, it is enough to require
\begin{align*}
    \eta(x_0 y_0-\mu)<\frac{4}{\left( m+\sqrt{m^2+4} \right)^2}\triangleq r.
\end{align*}
From Theorem~\ref{thm:xy_diff_decay}, it is guaranteed that $|x_t-y_t|<|x_0-y_0|$ with $t\ge2$ until it reaches $x_t y_t>\mu$, with which $r$ is still a good lower bound for $\min\{y_t/x_t, x_t/y_t\}$. So what remains to show is it satisfies $\eta(x_t y_t-\mu)<r$ for the next first time $x_t y_t>\mu$. If this holds, we can always iteratively show, for any $x_t y_t>\mu$ along gradient descent,
\begin{align*}
    \eta(x_t y_t-\mu)<r<\min\left\{ \frac{y_t}{x_t}, \frac{x_t}{y_t}\right\}.
\end{align*}
Note that $r$ itself is independent of $x_t y_t$ and all the history, so it is ideal to compute a uniform upper bound of $\eta(x_t y_t-\mu)$ with any pair of $(x_{t-1}, y_{t-1})$ satisfying $x_{t-1} y_{t-1}<\mu$. Actually it is possible, since we have $|x_{t-1}-y_{t-1}|$ bounded as in Theorem~\ref{thm:xy_diff_decay}.

Assume $x_i y_i>\mu$ and it satisfies the condition of $\eta(x_i y_i-\mu)<r$ and $|x_i-y_i|<|x_0-y_0|$. As in (\ref{eq:diff_evolv_t}), we have
\begin{align}
    \frac{x_{i+1}-y_{i+1}}{x_i-y_i}= 1-\eta\left(\mu-x_i y_i\right)\in(1,1+r).
\end{align}
Hence, it suffices to get the maximum value of $g(z)$, with $z\in(0,\mu)$, as
\begin{align}
    g(z)=z\left(1+\eta(\mu-z)\right)^2+\eta(\mu-z) (1+r)^2 (x_0-y_0)^2,
\end{align}
which is from (\ref{eq:xy_evolv}). Denote $\Bar{z}=\text{argmax}~g(z)$. Obviously $\Bar{z}<\frac{1}{3}(\mu+\frac{1}{\eta})\triangleq z_b$, because the first term of $g(z)$ achieves maximum at $z=\frac{1}{3}(\mu+\frac{1}{\eta})$ and the second term is in a decreasing manner with $z$. Then let's take the derivative of $g(z)$ as
\begin{align}
    g'(z) &=\left(1+\eta(\mu-z)\right)\left(1+\eta\mu-3\eta z\right)-\eta (1+r)^2 (x_0-y_0)^2 \nonumber \\
    &= \left(1+\eta(\mu-z)\right)\left(1+\eta\mu-3\eta z-\frac{\eta (1+r)^2 (x_0-y_0)^2 }{1+\eta(\mu-z)}\right) \nonumber,
\end{align}
where the first term is always positive, so we have
\begin{gather}
    1+\eta\mu-3\eta \Bar{z}-\frac{\eta (1+r)^2 (x_0-y_0)^2 }{1+\eta(\mu-\Bar{z})}=0,
\end{gather}
which means
\begin{align}
    \Bar{z} &= \frac{1}{3\eta}\left(1+\eta\mu-\frac{\eta (1+r)^2 (x_0-y_0)^2 }{1+\eta(\mu-\Bar{z})}\right) \\
    &>  \frac{1}{3\eta}\left(1+\eta\mu-\frac{\eta (1+r)^2 (x_0-y_0)^2 }{1+\eta(\mu-\frac{1}{3}(\mu+\frac{1}{\eta}))}\right) \\
    &= \frac{1}{3}\left(\mu+\frac{1}{\eta}-\frac{3(1+r)^2}{2(\eta+1)} (x_0-y_0)^2\right)\\
    &\triangleq z_s,
\end{align}
where the inequality is from $\Bar{z}<\frac{1}{3}(\mu+\frac{1}{\eta})$. As a result, it is safe to say
\begin{align}
    g(z) &\le z\left(1+\eta(\mu-z)\right)^2\bigg|_{z=z_b}+\eta(\mu-z) (1+r)^2 (x_0-y_0)^2\bigg|_{z=z_s} \\
    &=\frac{4}{27}(1+\eta\mu)^3\cdot\frac{1}{\eta} + \eta (1+r)^2\left(\frac{2}{3}\mu-\frac{1}{3\eta}+\frac{2}{\eta\mu+1}(x_0-y_0)^2\right)(x_0-y_0)^2,
\end{align}
with which we are able to compute $\max~\eta(g(z)-\mu)$, which is exactly the final result.
\end{proof}

\section{Proof of Theorem~\ref{thm:one_neuron}}
\label{app:one_neuron}

\begin{theorem}[Restatement of Theorem~\ref{thm:one_neuron}]
In the above setting, consider a teacher neuron $\tilde{w}=[1,0]$ and set the learning rate $\eta=Kd$ with $K\in(1,1.1]$. Initialize the student as $\norm{w^{(0)}}=v^{(0)}\triangleq \epsilon\in(0,0.10]$ and $\inprod{w^{(0)}}{ \tilde{w}}\ge0$. Then, for $t\ge T_1 + 4$, $w_y^{(t)}$ decays as 
\begin{gather*}
    \red{w_y^{(t)} < 0.1\cdot (1-0.030K)^{t-T_1-4},}
    ~~~~~
    T_1 \le \left\lceil\log_{2.56}\frac{1.35}{\pi \beta^2}\right\rceil,
    ~~~~~~
    \beta = \left(1+\frac{1.1}{\pi}\right)\epsilon.
\end{gather*}
\end{theorem}

\begin{sketch}
    The proof is divided into two stages, depending on whether $w_y$ grows or not. The key is that the change of $w_y$ follows (omitting all superscripts $t$)
    \begin{align}
        \frac{\Delta w_y}{w_y}
        \propto
        -vw_x+\frac{1}{\pi}\frac{\frac{w_y}{w_x}}{1+(\frac{w_y}{w_x})^2}, ~~~~
        w_y\tt = \left| w_y + \Delta w_y \right|.
    \end{align}
    where the second term in $\nicefrac{\Delta w_y}{w_y}$ is bounded in $[0, \frac{1}{2\pi}]$. In stage 1 where $v w_x$ is relatively small, we show the growth ratio of $w_y$ is smaller than those of $w_x$ and $vw_x$, resulting in an upper bound of number of iterations for $vw_x$ to reach $\frac{1}{2\pi}$, so $\max(w_y)$ is bounded too. Although the initialization is balanced as $v\T{0}=\norm{w\T{0}}$ for simplicity of proof, $v-w_x$ is also bounded at the end of stage 1. From the beginning of stage 2, thanks to the relatively narrow range of $K$, we are able to compute the bounds of three variables (including $v-w_x$, $vw_x$ and $w_y$) and they turn out to fall into a basin in the parameter space after four iterations. In this basin, $w_y$ decays exponentially with a linear rate of 0.97 at most.
    \end{sketch}

\begin{proof}
We restate the update rules as
\begin{align}
    \Delta v\t
    &\coloneqq 
    v^{(t+1)} - v^{(t)}
    =
    K w_x\t \left[
    (-v\t w_x\t + 1) - v\t w_y\t\frac{w_y\t}{w_x\t}-\frac{1}{\pi}\left(\arctan{\left(\frac{w_y\t}{w_x\t}\right)}-\frac{w_y\t}{w_x\t}\right)
    \right], \nonumber \\
    &= K w_x\t \left[
    (-v\t w_x\t + 1) - \frac{1}{\pi}\left(\arctan{\left(\frac{w_y\t}{w_x\t}\right)}-\frac{w_x\t w_y\t}{\norm{w\t}^2}\right)
    \right] \nonumber \\
    &~~~~ + K\frac{(w_y\t)^2}{v\t}\left(-(v\t)^2+\frac{v\t w_y\t}{\pi \norm{w\t}^2}\right)
    \label{eq:deltav}\\
    \Delta w_x\t 
    &\coloneqq 
    w_x^{(t+1)} - w_x^{(t)} =
    K v\t\left[
    (-v\t w_x\t + 1) - \frac{1}{\pi}\left(\arctan{\left(\frac{w_y\t}{w_x\t}\right)}-\frac{w_x\t w_y\t}{\norm{w\t}^2}\right)
    \right], \label{eq:deltawx}\\
    \Delta w_y\t
    &=
    w_y^{(t)}\cdot K\left(-(v\t)^2+\frac{v\t w_y\t}{\pi \norm{w\t}^2}\right), \label{eq:deltawy}\\
    w_y\tt &= \left| w_y\t + \Delta w_y\t \right|. \label{eq:wytp1}
\end{align}
For simplicity, we will omit all superscripts of time $t$ unless clarification is necessary. From (\ref{eq:wytp1}), if the target is to show $w_y$ decaying with a linear rate, it suffices to bound the factor term in (\ref{eq:deltawy}) (by a considerable margin) as
\begin{gather}
    -2< K\left(-v^2+\frac{v w_y}{\pi \norm{w}^2}\right) <0. \label{eq:bound_deltawy}
\end{gather}
The technical part is the second inequality of (\ref{eq:bound_deltawy}). If $v, w_x>0$, it is equivalent to 
\begin{align*}
    v w_x>\frac{w_x w_y}{\pi \norm{w}^2}=\frac{w_x w_y}{\pi (w_x^2+w_y^2)},
\end{align*}
where the RHS is smaller than or equal to $\frac{1}{2\pi}$. Hence, $\frac{1}{2\pi}$ is a special threshold with which we will frequently compare $v w_x$. Another important variable to control is $v-w_x$ that reveals how the two layers are balanced. If it is too large, for the iteration $v\tt w_x\tt$ may explode as shown in the 2-D case.

The main idea of our proof is that
\begin{itemize}
    \item Stage 1 with $vw_x \le \frac{w_x w_y}{\pi\norm{w}^2}$: in this stage, $w_y$ grows but it grows in a smaller rate than that of $v$ and $w_x$. Therefore, since we have an upper bound for $v w_x$ to stay in this stage, we are able to compute the upper bound of $\#$iterations to finish this stage, which is $T_1$ in the theorem. At the end of this stage, both of $v-w_x$ and $w_y$ are bounded under our assumption of initialization.
    \item Stage 2 with $vw_x > \frac{w_x w_y}{\pi\norm{w}^2}$: in this stage, $w_y$ decreases. Since our range of a large learning rate is relatively narrow $(1<K\le 1.1)$, we are able to compute bounds of $v w_x, v-w_x$ and $w_y$. After eight iterations, it falls into (and stays in) a bounded basin of these three terms, in which $w_y$ decays at least in a linear rate.
\end{itemize}

{\centerline{\large{\textbf{Stage 1.}}}}

We are to show that, in the last iteration of this stage, there are three facts: 1) $v w_x\le \frac{1}{2\pi}$, 2) $v-w_x\in[-0.017, 0.17]$, and 3) $w_y\le0.44$.

At initialization, we assume $v^{(0)}=\norm{w^{(0)}}$. Denote $\alpha_0=\arctan(w_y^{(0)}/w_x^{(0)})\in[0,\pi/2]$. So for next iteration we have
\begin{gather}
    w_y\T{1}=v\T{0}\left(1+K\left(-(v\T{0})^2+\frac{1}{\pi}\sin\alpha_0\right)\right),\label{eq:wy1}\\
    w_x\T{1}=v\T{0}\left[\cos\alpha_0+K\left(1-(v\T{0})^2 \cos\alpha_0+\frac{\cos\alpha_0 \sin\alpha_0 - \alpha_0}{\pi}\right)\right].
\end{gather}
Apparently $w_y\T{1}$ increases with $\alpha_0$ increasing. And
\begin{align*}
    \partial_{\alpha_0} w_x\T{1} &= v\T{0}\left[-\sin\alpha_0+K\left((v\T{0})^2 \sin\alpha_0+\frac{-\sin^2\alpha_0+\cos^2\alpha_0 - 1}{\pi}\right)\right] \\
    &= v\T{0}\left[-\sin\alpha_0+K\left(\left((v\T{0})^2-\frac{\sin\alpha_0}{\pi}\right) \sin\alpha_0+\frac{-\sin^2\alpha_0}{\pi}\right)\right] .
\end{align*}
Since in stage 1 it holds $\Delta w_y>0$ which means $-(v\T{0})^2+\frac{1}{\pi}\sin\alpha_0>0$ in (\ref{eq:wy1}). So it follows $\partial_{\alpha_0} w_x\T{1}\le 0$. Combining the above arguments, we have
\begin{align*}
    w_x\T{1} &\ge w_x\T{1}|_{\alpha_0=\frac{\pi}{2}} = \frac{K}{2}v\T{0}, \\
    w_y\T{1} &\le w_y\T{1}|_{\alpha_0=\frac{\pi}{2}} = \left(1+\frac{K}{\pi}-K(v\T{0})^2\right) v\T{0}\le \left(1+\frac{K}{\pi}\right) v\T{0},\\
    \frac{w_y\T{1}}{w_x\T{1}} &\le \frac{2+\frac{2K}{\pi}}{K} \le 2.7.
\end{align*}

Regarding $\frac{v}{w_y}$, it has $v\T{0}\ge w_y\T{0}$ at initialization due to $v\T{0}=\norm{w\T{0}}$. From (\ref{eq:deltav}, \ref{eq:deltawx}, \ref{eq:deltawy}), we have $v\Delta v = w_x\Delta w_x + w_y \Delta w_y$. So it holds $v\Delta v\ge y\Delta y$. Meanwhile, $\frac{\Delta w_y}{v}=K(-v w_y + \frac{w_y^2}{\pi \norm{w}^2})\in [0, \frac{K}{\pi}]$. From Lemma~\ref{lem:ab1}, given $v\T{t}\ge w_y\T{t}$ and $\frac{\Delta w_y}{v}\in[0,1]$ for any $t$ in this stage, it always holds $v\tt \ge w_y\tt$.

Therefore, it is fair to say
\begin{align*}
    \frac{v\T{1}w_x\T{1}}{(w_y\T{1})^2}\ge\frac{1}{2.7}.
\end{align*}

Additionally, to bound the term $vw_y/\norm{w}^2$ in $\Delta w_y$, we would like to show it always has $v w_y\le \norm{w}^2$. At initialization, it naturally holds. Then, for the every next iteration, given it holds in the last iteration, we have
\begin{align*}
    & (v+\Delta v)(w_y+ \Delta w_y) - [(w_x + \Delta w_x)^2 + (w_y + \Delta w_y)^2] \\
    & = (v+\frac{w_x \Delta w_x + w_y \Delta w_y}{v})(w_y+\Delta w_y) - [(w_x + \Delta w_x)^2 + (w_y + \Delta w_y)^2] \\
    & = v w_y + v\Delta w_y + w_x \Delta w_x (\frac{w_y}{v} + \frac{\Delta w_y}{v}) + (w_y \Delta w_y+(\Delta w_y)^2)\frac{w_y}{v} - [(w_x + \Delta w_x)^2 + (w_y + \Delta w_y)^2]\\ 
    & \le vw_y + v\Delta w_y + w_y\Delta w_y \frac{w_y}{v} - (w_x^2 + w_y^2 + 2w_y\Delta w_y + (\Delta w_x)^2) \\
    & \le v\Delta w_y + w_y\Delta w_y \frac{w_y}{v} - 2w_y\Delta w_y - (\Delta w_x)^2 \\
    & = v\Delta w_y(1-\frac{w_y}{v})^2 - (\Delta w_x)^2 \\
    & \le v\Delta w_y - (\Delta w_x)^2
\end{align*}
where the first equality uses $v\Delta v = w_x\Delta w_x + w_y \Delta w_y$, the first inequality uses the proven $v\ge w_y$ and $v\ge \Delta w_y$, the second inequality uses the assumption $vw_y\le\norm{w}^2$. Now we are to show $v\Delta w_y - (\Delta w_x)^2\le0$. We have
\begin{align*}
    v\Delta w_y - (\Delta w_x)^2 &\le Kv^2\frac{w_y^2}{\pi\norm{w}^2} - K^2 v^2 \left(1-\frac{1}{2\pi}-\gamma\t\right)^2,\\
    \gamma\t &= \frac{1}{\pi}\left(\arctan{\left(\frac{w_y\t}{w_x\t}\right)}-\frac{w_x\t w_y\t}{\norm{w\t}^2}\right).
\end{align*}
Since we have proven $w_y\T{1}/w_x\T{1}\le2.7$, it is easy to check that 
\begin{align*}
    \frac{1}{\pi\left(1+(\frac{w_x\T{1}}{w_y\T{1}})^2\right)}\le(1-\frac{1}{2\pi}-\gamma\T{1})^2.
\end{align*}
As a result, $v\Delta w_y - (\Delta w_x)^2\le0$ at time 1. Furthermore, by checking each term, $v\Delta w_y - (\Delta w_x)^2$ decreases with $w_y/w_x$ decreasing. We will soon show that $w_y/w_x$ itself decreases, by showing the growth ratio of $w_x$ is larger than that of $w_y$.

Our target lower bound of the growth ratio of $w_x$ is that 
\begin{align}
    \frac{\Delta w_x}{w_x} \ge 1-\frac{1}{\pi}-\gamma, \label{eq:ratiowx}
\end{align}
which is larger than the growth ratio of $w_y$ bounded by $\frac{1}{\pi}$ due to $v\Delta w_y < \norm{w}^2$.
So it suffices to show $Kv/w_x\ge 1$. Assuming $Kv/w_x\ge 1$ for the current step, we need to show $Kv\tt/w_x\tt\ge 1$ also holds for the next step. Let's denote 
\begin{align}
    A\t &= K\left[(-v\t w_x\t + 1) - \frac{1}{\pi}\left(\arctan{\left(\frac{w_y\t}{w_x\t}\right)}-\frac{w_x\t w_y\t}{\norm{w\t}^2}\right)\right].
\end{align}
Then 
\begin{align}
    (v+\Delta v)-\frac{1}{K}(w_x+\Delta w_x) &\ge v+A w_x - \frac{w_x}{K}-\frac{A v}{K} \nonumber \\
    & = (v-\frac{w_x}{K})(1-KA)+v(K-\frac{1}{K})A. \label{eq:kvwx}
\end{align}
If $KA\le 1$, since $K>1$ and $A>0$, we have (\ref{eq:kvwx}) as positive, which is what we need. If $KA>1$, then
\begin{align*}
    (\ref{eq:kvwx}) &\ge (v-\frac{w_x}{K})(1-K^2)+v(K-\frac{1}{K})A \\
    &= \left((-K+A)v+w_x\right)(K-\frac{1}{K}),
\end{align*}
where the first inequality is due to $A\le K$ and the assumption of $Kv\t/w_x\t\ge 1$. Then it suffices to show $(-K+A)v+w_x\ge (-K+\frac{1}{K})v+w_x\ge0$. Note that $-K+1/K\in(-0.2,0]$ when $K\in(1,1.1]$. It is easy to verify that $v\T{1}\le 5w_x\T{1}$. Then, for the next step, we need to show $v\tt\le 5w_x\tt$ given $v\t\le 5w\tt$. To prove this, we are to bound $v-w_x$, as
\begin{align}
    v\tt - w_x\tt &= (1 - A)(v-w) + K\frac{w_y^2}{v}(-v^2+\frac{v w_y}{\pi\norm{w}^2}) \nonumber \\
    & \le 0.4(v-w)+Kw_y\frac{w_y^2}{\pi\norm{w}^2}\le 0.4(v-w)+\frac{Kw_y}{\pi}, \label{eq:vwdiff0}
\end{align}
where the first inequality is due to, when $w_y/w_x\le 2.7$,
\begin{align*}
    A 
    &= K\left[-v\t w_x\t +\frac{1}{\pi}\frac{w_x\t w_y\t}{\norm{w\t}^2}\right] + K\left[1 - \frac{1}{\pi}\arctan{\left(\frac{w_y\t}{w_x\t}\right)}\right]\\
    &\ge K\left[1 - \frac{1}{\pi}\arctan{\left(\frac{w_y\t}{w_x\t}\right)}\right] \ge 0.6.
\end{align*}
We will later show that $v\tt-w\tt\ge -0.1(v\t-w\t)$. Combining this with (\ref{eq:vwdiff0}), it is safe to say 
\begin{align*}
    v\tt-w\tt \le 0.4(v-w)+\frac{Kw_y}{\pi}\le 0.4\times 4w+\frac{K\times 5w}{\pi} \le  4w,
\end{align*}
where the second inequality is due to $v\le 5w$ and $v\ge w_y$. Since $w\tt\ge w\t$ (due to $A>0$) in this stage, we have $v\tt\le 5w_x\tt$.

Combining the above discussion, we have prove (\ref{eq:ratiowx}). Obviously, when $w_y/w_x\le 2.7$, RHS of (\ref{eq:ratiowx}) is at least $0.55$, larger than $1.1/\pi$, which is the upper bound of the $\Delta w_y/w_y$. As a result, $w_y/w_x$ keeps decreasing.

The next step is to show the growing ratio of $v w_x$ is much larger than that of $w_y$. From (\ref{eq:deltawx}, \ref{eq:deltawy}), it holds
\begin{align*}
    v\tt w_x\tt &= (v+\Delta v)(w_x+\Delta w_x) \ge v w_x + KA(v^2+ w_x^2) + K^2 A^2 v w_x \\
    &\ge vw_x(1+A)^2,
\end{align*}
where the first inequality is due to $\Delta w_y\ge0$. It follows $v\tt w_x\tt/v\t w_x\t\ge 1.6^2=2.56$.

So far, we have shown the following facts: under the defined initialization at time 0, starting from time 1, we have
\begin{enumerate}
    \item $v w_x\le 1/2\pi$.
    \item $\Delta w_x/w_x+1\ge 1.55$.
    \item $\Delta w_y/w_y+1\le 1+K/\pi$.
    \item $w_y/w_x \le 2.7$ and keeps decreasing.
    \item $v\tt w_x\tt/v\t w_x\t\ge 2.56$.
    \item $v\ge w_y$.
    \item $v\Delta w_y<(\Delta w_x)^2$.
\end{enumerate}

Now we are to use the above facts to bound $v w_x, w_y$ and $v-w_x$ to the end of stage 1.

For $v w_x$, in previous discussion, we have shown that $v w_x\le \frac{1}{2\pi}$. Actually, there is another special value
\begin{align}
    \frac{w_x w_y}{\pi(w_x^2+w_y^2)} = 0.104 \text{  when $w_y/w_x=2.7.$}
\end{align}
This value is slightly larger than $1/4\pi$. Hence, we would like to split the analysis into three parts: in the \textbf{first step of stage 2}, 
\begin{enumerate}
    \item $v w_x\ge \frac{1}{2\pi}$.
    \item $\frac{1}{4\pi}\le v w_x<\frac{1}{2\pi}$.
    \item $v w_x<\frac{1}{4\pi}$.
\end{enumerate}

Note that, although we are discussing the stage 1 in this section, investigating the lower bound of the first step in stage 2 helps calculate the number of iterations in stage 1. Furthermore, it helps bound several variables in stage 1.

\textbf{Case (I). If $vw_x\ge\frac{1}{2\pi}$ in first step of stage 2:}

Since we have prove $\frac{v\T{1} w_x\T{1}}{(w_y\T{1})^2}\ge1/2.7$ and $v\tt w_x\tt/v\t w_x\t\ge 2.56$, the number of iterations for $vw_x$ to reach $1/2\pi$ is at most
\begin{align}
    T_1\le \left\lceil\log_{2.56}\frac{\frac{1}{2\pi}}{(w_y\T{1})^2/2.7}\right\rceil. \label{eq:T1}
\end{align}
Meanwhile, starting from time 1, the growth ratio of $w_y$ is
\begin{align}
    (w_y + \Delta w_y)/w_y\le 1+K(-v^2+1/\pi)\le 1+1.1/\pi-(v\T{1})^2\le 1+1.1/\pi-(w_y\T{1})^2,
\end{align}
where the first inequality is due to $vw_y\le\norm{w}^2$, the second is due to $K>1$ and the third is from $v\ge w_y$. Therefore, combining with (\ref{eq:T1}), we can bound $w_y$ in the end of stage 1 as
\begin{align}
    w_y\le \left(1+1.1/\pi-(w_y\T{1})^2\right)^{\left\lceil\log_{2.56}\frac{\frac{1}{2\pi}}{(w_y\T{1})^2/2.7}\right\rceil}.
\end{align}
Since it initializes as $\norm{w\T{0}}\le 0.1$, we have $w_y\T{1}\le 0.1(1+1.1/\pi)=0.135$. Then, it can be verified that, when $w_y\T{1}\in(0,0.135]$, it holds
\begin{align}
    w_y\le 0.44.
\end{align}

The next is to bound $v-w_x$. Combining the update rules of $v$ and $w_x$ in (\ref{eq:deltav}, \ref{eq:deltawx}), we have
\begin{align}
    \Delta(v-w_x) &\coloneqq (v\tt-w_x\tt)-(v\t-w_x\t) \nonumber \\
    &=K(v-w_x)\left(v w_x-1+\frac{\arctan(w_y/w_x)-\frac{w_x w_y}{\norm{w}^2}}{\pi}\right) + K\frac{w_y^2}{v}(-v^2+\frac{v w_y}{\pi\norm{w}^2}).\label{eq:deltavwx1}
\end{align}

Note that 
\begin{align}
    -1\le v w_x-1+\frac{\arctan(w_y/w_x)-\frac{w_x w_y}{\norm{w}^2}}{\pi} \le -1+\frac{\arctan(w_y/w_x)}{\pi},
\end{align}
where the left is due to $vw_x>0$ and  , the right is from $\Delta w_y\ge 0$. When $w_y/w_x\le2.7$, the RHS follows $-1+\frac{\arctan(w_y/w_x)}{\pi}\le -0.6$. So combining both sides tells
\begin{align}
    1+K\left(v w_x-1+\frac{\arctan(w_y/w_x)-\frac{w_x w_y}{\norm{w}^2}}{\pi}\right)\in[-K+1, 0.4]\subset [-0.1,0.4].\label{eq:deltavwx_coff}
\end{align}

Since $\Delta w_y\ge 0$, we have $0\le K\frac{w_y^2}{v}(-v^2+\frac{v w_y}{\pi\norm{w}^2})\le \frac{K}{\pi} w_y\frac{w_y^2}{\norm{w}^2}$. Note that at initialization $w_x\T{0}\le v\T{0}$. Then it is easy to verify that
\begin{align}
    -0.01\le -0.1 (v\T{0}-w\T{0})\le v\T{1}-w\T{1}\le (1+\frac{K}{\pi}-\frac{K}{2})v\T{0}\le 0.082.
\end{align}
Because the coefficient on the positive side in (\ref{eq:deltavwx_coff}) is larger than $0.4>0.1$, it is appropriate to upper bound the $v-w_x$ as
\begin{align*}
    v-w_x &\le  \max\left\{0.082, 0.082\cdot 0.4^T + \sum_{t=1}^T 0.4^{t-1} \frac{K}{\pi} w_y\t\frac{(w_y\t)^2}{\norm{w\t}^2}\right\} \\
    &\le \max\left\{0.082, 0.082\cdot 0.4^T + \sum_{t=1}^T 0.4^{t-1} \frac{K}{\pi} w_y\t\frac{1}{1+\frac{1}{2.7}\left(\frac{1.55}{1+K/\pi}\right)^{2(t-1)}}\right\} \\
    &\le \max\left\{0.082, 0.082\cdot 0.4^T + \sum_{t=1}^T 0.4^{t-1} \frac{1.1\cdot 4.4}{\pi} \frac{1}{1+\frac{1}{2.7}\left(\frac{1.55}{1+1.1/\pi}\right)^{2(t-1)}}\right\},
\end{align*}
where the second inequality is from the different growth ratios of $w_x$ and $w_y$. Note that here we take all $T\ge1$ and pick the largest value of RHS to bound $w_y$. It turns out
\begin{align}
    v-w_x\le 0.17.
\end{align}
Furthermore, to lower bound $v-w_x$, since obviously $|v-w_x|\le 0.17$, it follows
\begin{align}
    v-w_x\ge -0.1\cdot |v-w_x|_{\max} \ge -0.017.
\end{align}

\textbf{Case (II). If $\frac{1}{4\pi}\le vw_x<\frac{1}{2\pi}$ in first step of stage 2:}

Similar to the discussion in Case (I), we are able to compute the number of iterations for $vw_x$ to reach $1/4\pi$. It is at most
\begin{align}
    T_1\le \lceil\log_{2.56}\frac{\frac{1}{4\pi}}{(w_y\T{1})^2/2.7}\rceil. \label{eq:T1_2}
\end{align}

Accordingly, $w_y$ is bounded as
\begin{align}
    w_y \le \left(1+1.1/\pi-(w_y\T{1})^2\right)^{\lceil\log_{2.56}\frac{\frac{1}{4\pi}}{(w_y\T{1})^2/2.7}\rceil}\le 0.37.
\end{align}

For simplicity, we just keep the bounds for $v-w_x$ as in Case (I), as
\begin{align}
    v-w_x\in[-0.017, 0.17].
\end{align}

\textbf{Case (III). If $vw_x<\frac{1}{4\pi}$ in first step of stage 2:}

From the condition, we know $vw_x<\frac{1}{4\pi}$ as well in the last step of stage 1. Since $\Delta w_y>0$ in stage 1, it tells
\begin{align}
    \frac{1}{\pi}\frac{w_x w_y}{\norm{w}^2}<v w_x\le \frac{1}{4\pi},
\end{align}
which means
\begin{align}
    \max\{\frac{w_x}{w_y}, \frac{w_y}{w_x}\}\ge 2+\sqrt{3}.
\end{align}
Since $2+\sqrt{3}>2.7$, if $w_y/w_x\ge 2+\sqrt{3}$, then for time 1, $(v\T{1}, w_x\T{1}, w_y\T{1})$ is already in the stage 2. However, it is not possible because $\norm{w\T{0}}=v\T{0}\le 0.1$, which means $v\T{1} w_x\T{1}$ can not reach $\frac{1}{\pi}\frac{2.7}{1+2.7^2}$.

Therefore, the only possible is $\frac{w_x}{w_y}\ge 2+\sqrt{3}$. In this case, we are able to bound $w_y$ as
\begin{align}
    w_y\le (2-\sqrt{3})w_x\le (2-\sqrt{3})\left(\sqrt{\frac{1}{4\pi}+0.0085^2}+0.0085\right) \le 0.078,
\end{align}
where the second inequality is due to $vw_x\le\frac{1}{4\pi}$ and $v-w_x\ge -0.017$. Note that here we still use the bound of $v-w_x$ from Case (I), although it is loose somehow but it is enough for our analysis.

We leave the analysis of the bound of number of iterations to the end of this section.

\vspace{10pt}
{\centerline{\large{\textbf{Stage 2.}}}}

In the case (I) of stage 1, where the first step in stage 2 is with $v w_x \ge \frac{1}{2\pi}$, it has $v-w_x\in[-0.017, 0.17]$ and $w_y\le 0.44$. In the case (II), where the first step of stage 2 is with $v w_x\in[\frac{1}{4\pi}, \frac{1}{2\pi}]$, it has $v-w_x\in[-0.017, 0.17]$ and $w_y\le 0.37$. In the case (III), where the first step of stage 2 is with $v w_x\in[\frac{1}{4\pi}, \frac{1}{2\pi}]$, it has $v-w_x\in[-0.017, 0.17]$ and $w_y\le 0.078$.



To upper bound $v w_x$ in the first step of stage 2, there are two candidates. One is from the case (I), 
\begin{align}
    v\tt w_x\tt &= vw_x\left( 1+K(1-v w_x - \frac{\arctan(\frac{w_y}{w_x})-\frac{w_y/w_x}{1+(w_y/w_x)^2}}{\pi}) \right)^2 + K\frac{w_x w_y^2}{v}\left(-v^2+\frac{v w_y}{\pi\norm{w}^2}\right) \nonumber\\
    & ~~~~ + K(v-w_x)^2\left( 1+K(1-v w_x - \frac{\arctan(\frac{w_y}{w_x})-\frac{w_y/w_x}{1+(w_y/w_x)^2}}{\pi}) \right) \nonumber\\
    &\le vw_x\left( 1+K(1-v w_x)\right)^2 + K\frac{w_x w_y^2}{w_x}\left(-vw_x+\frac{w_x w_y}{\pi\norm{w}^2}\right) \nonumber\\
    & ~~~~ + K(v-w_x)^2\left( 1+K(1-v w_x) \right) \nonumber \\
    &\le \frac{1}{2\pi}\left( 1+1.1(1-\frac{1}{2\pi})\right)^2 + 1.1\cdot 0.44^2\left(-\frac{1}{4\pi}+\frac{1}{2\pi}\right) + 1.1\cdot 0.17^2\left( 1+1.1(1-\frac{1}{2\pi}) \right) \nonumber \\
    &\le 0.668,
\end{align}
where we use $v w_x\ge 1/4\pi$, $x/(1+x^2)\le 0.5$ for any $x$. 

One is from the case (II),
\begin{align}
    v\tt w_x\tt 
    &\le vw_x\left( 1+K(1-v w_x)\right)^2 + K\frac{w_x w_y^2}{w_x}\left(-vw_x+\frac{w_x w_y}{\pi\norm{w}^2}\right) \nonumber\\
    & ~~~~ + K(v-w_x)^2\left( 1+K(1-v w_x) \right) \nonumber \\
    &\le \frac{1}{4\pi}\left( 1+1.1(1-\frac{1}{4\pi})\right)^2 + 1.1\cdot 0.37^2\left(\frac{1}{2\pi}\right) + 1.1\cdot 0.17^2\left( 1+1.1(1-\frac{1}{4\pi}) \right) \nonumber \\
    &\le 0.48,
\end{align}
where we use $v w_x\le 1/4\pi$, $x/(1+x^2)\le 0.5$ for any $x$. 

Therefore, we can see that, in the first step of stage 2,
\begin{align}
    v w_x\le 0.668.
\end{align}


Next we are going to show how the iteration goes in the stage 2. In Case (I), there are three facts:
\begin{enumerate}
    \item $w_y\le 0.44$.
    \item $v-w_x\in[-0.017, 0.17]$.
    \item $vw_x \in[\frac{1}{2\pi}, 0.668]$.
\end{enumerate}
Similarly, in Case (II), there are three facts as well:
\begin{enumerate}
    \item $w_y\le 0.37$.
    \item $v-w_x\in[-0.017, 0.17]$.
    \item $vw_x \in[\frac{1}{4\pi}, \frac{1}{2\pi}]$.
\end{enumerate}

The main idea is to find a basin that any iteration with the above properties (\textit{i.e.}, in the interval) will converge to and then stay in. The method is to iteratively compute the ranges of the variables for several steps, thanks to the narrow range of $K$. Before explicitly computing the ranges, let's write down the computing method, depending on whether or not $v w_x\ge1$.

Consider any iteration with $v w_x\in[m_1,m_2], v-w_x\in[d_1,d_2], w_y\le e$, we compute the bounds of $v\tt w_x\tt, v\tt-w_x\tt, w_y\tt$ in the following process (naturally assuming $d_1<0<d_2$)
\begin{enumerate}
    \item If $m_1\ge 1$:
    
    \begin{enumerate}
        \item Compute $w_x\ge \sqrt{m_1+(d_2/2)^2}-d_2/2\triangleq f$.
        \item Compute $\frac{w_y}{w_x}\le e/f\triangleq g$.
        \item Compute $\frac{\arctan(w_y/w_x)-\frac{w_x w_y}{\norm{w}^2}}{\pi}\le \frac{\arctan(g)-g/(1+g^2)}{\pi}\triangleq h$.
        \item Compute $v\tt w_x\tt \ge m_2(1+1.1(1-m_2-h))^2+1.1(1-m_2-h)\max\{|d_1|,|d_2|\}^2-1.1e^2 m_2$. This is from 
            \begin{align}
                v\tt w_x\tt &\ge vw_x\left( 1+K(1-v w_x - h) \right)^2 + K\frac{w_x w_y^2}{v}\left(-v^2+\frac{v w_y}{\pi\norm{w}^2}\right) \nonumber\\
                & ~~~~ + K(v-w_x)^2\left( 1+K(1-v w_x - h) \right) \nonumber\\
                &\ge vw_x\left( 1+K(1-v w_x - h) \right)^2 - Kw_y^2\cdot vw_x \nonumber\\
                & ~~~~ + K(v-w_x)^2\left( 1+K(1-v w_x - h) \right) \nonumber.
            \end{align}
        \item Compute $v\tt w_x\tt \le m_1(1+1.0(1-m_1))^2$. This is due to $x(1+K(1-x))^2$ decreases with $x$ increasing when $x\ge 1$.
        \item Compute $v\tt - w_x\tt \in [d_1(1+1.1(m_2-1+h)-1.1e^2\cdot(\sqrt{m_2+(d_2/2)^2}+d_2/2)), d_2(1+1.1(m_2-1+h))]$. This is due to 
        \begin{align*}
            \Delta v-\Delta w_x = K(v-w_x)\left(v w_x - 1 +\frac{1}{\pi}(\arctan(\alpha)-\frac{w_x w_y}{\norm{w}^2})\right) + K\frac{w_y^2}{v}\left(-v^2+\frac{v w_y}{\pi\norm{w}^2}\right),
        \end{align*}
        where $vw_x\ge 1$, the last term is between $-Kvw_y^2$ and 0.
        \item Compute $w_y\tt \le e\cdot\max\{|j_1|, |j_2|\}$, where
        \begin{align}
            j_1 &= 1 + 1.1\frac{\sqrt{m_1+(d_2/2)^2}+d_2/2}{\sqrt{m_1+(d_2/2)^2}-d_2/2}\cdot(-m_2), \\
            j_2 &= 1 + 1.0\frac{\sqrt{m_1+(d_1/2)^2}-d_1/2}{\sqrt{m_1+(d_1/2)^2}+d_1/2}\cdot(-m_1 + \frac{1}{2\pi}).
        \end{align}
        This is due to 
        \[
            \frac{\Delta w_y}{w_y} = K\frac{v}{w_x}(-vw_x + \frac{1}{\pi}\frac{w_x w_y}{\norm{w}^2}),
        \]
        then we would like to have the smallest value as $j_1-1$ and the largest value as $j_2-1$. Since $w_y$ is always non-negative, taking the maximum absolute value gives the upper bound.
    \end{enumerate}
    
    \item If $m_2< 1$:
    
    \begin{enumerate}
        \item Compute $w_x\ge \sqrt{m_1+(d_2/2)^2}-d_2/2\triangleq f$.
        \item Compute $\frac{w_y}{w_x}\le e/f\triangleq g$.
        \item Compute $\frac{\arctan(w_y/w_x)-\frac{w_x w_y}{\norm{w}^2}}{\pi}\le \frac{\arctan(g)-g/(1+g^2)}{\pi}\triangleq h$.
        \item Compute $v\tt w_x\tt \ge \min_{x\in[m_1, m_2]} x(1+1.0(1-x-h))^2-1.1e^2 x$. Compared with the case of $m_1\ge 1$, we drop the term $1.1(1-m_2-h)\max\{|d_1|,|d_2|\}^2$ because it is possible to have $v-w_x=0$ in some iterations.
        \item Compute $v\tt w_x\tt \le \max_{x\in[m_1, m_2]} x(1+1.1(1-x))^2+1.1(1-x)\max\{|d_1|,|d_2|\}^2$. Compared with the case of $m_1\ge 1$, we add a term depending on the $|v-w_x|_{\max}$ because it enlarges $vw_x$ in the in-balanced case.
        \item Compute $v\tt - w_x\tt \in [d_1(1+1.1(m_2-1+h)-1.1e^2\cdot(\sqrt{m_2+(d_2/2)^2}+d_2/2)), d_2(1+1.1(m_2-1+h))]$. In fact, a rigorous left bound should include more terms to select a minimum from. Here it is simple because it keeps $1+K(m_1-1)\ge 0$ in the following computing, so we do not need to worry about the flipping sign of $d_1$ and $d_2$.
        \item Compute $w_y\tt \le e\cdot\max\{|j_1|, |j_2|\}$, where $j_1, j_2$ are the same with those in the case of $m_1\ge 1$.
    \end{enumerate}
\end{enumerate}

Therefore, with the above process, we are able to brutally compute the ranges of $v\tt w_x\tt, v\tt-w_x\tt, w_y\tt$ from the current ranges. Note that this process plays a role of building a mapping from one interval to another interval, which covers all points from the source interval. However, it is loose to some extent because gradient descent is a mapping from a point to another point. The advantage of such a loose method is feasibility of obtaining bounds while losing tightness. To achieve tightness, later we will also include some wisdom in a point-to-point style.

Also note that, a nice way to combine tightness and efficiency in this method is to split and to merge intervals when necessary.

\textbf{For Case (I):}

Now we are to compute the ranges starting from the interval where $I=\{w_y\le 0.44, v-w_x\in[-0.017, 0.17], vw_x \in[\frac{1}{2\pi}, 0.668]\}$. First, we split it into three intervals:
\begin{enumerate}
    \item $I_1 = \{w_y\le 0.44, v-w_x\in[-0.017, 0.17], vw_x \in[0.213, 0.4]\}$.
    \item $I_2 = \{w_y\le 0.44, v-w_x\in[-0.017, 0.17], vw_x \in[0.4, 0.668]\}$.
    \item $I_{30} = \{w_y\le 0.44, v-w_x\in[-0.017, 0.17], vw_x \in[\frac{1}{2\pi}, 0.213]\}$.
\end{enumerate}

Then, following the above method with splitting and merging intervals, we have
\begin{enumerate}
    \item Starting from $I_1$,
    \begin{enumerate}
        \item Step 1: $I_1$ mapps to $I_3=\{w_y\le 0.416, v-w_x\in[-0.162, 0.068], vw_x \in[0.55, 1.12131]\}$.
        \item Step 2: Splitting $I_3$, we have
            \begin{enumerate}
                \item $I_4 = \{w_y\le 0.416, v-w_x\in[-0.162, 0.068], vw_x \in[0.55, 0.8]\}$.
                \item $I_5 = \{w_y\le 0.416, v-w_x\in[-0.162, 0.068], vw_x \in[0.8, 0.9]\}$.
                \item $I_6 = \{w_y\le 0.416, v-w_x\in[-0.162, 0.068], vw_x \in[0.9, 1.0]\}$.
                \item $I_7 = \{w_y\le 0.416, v-w_x\in[-0.162, 0.068], vw_x \in[1.0, 1.12131]\}$.
            \end{enumerate}
            Then, we have
            \begin{enumerate}
                \item $I_4$ mapps to \\ $I_8 = \{w_y\le 0.214, v-w_x\in[-0.309, 0.0545], vw_x \in[0.942, 1.25786]\}$.
                \item $I_5$ mapps to \\ $I_9 = \{w_y\le 0.0966, v-w_x\in[-0.335, 0.0613], vw_x \in[0.880, 1.19649]\}$.
                \item $I_6$ mapps to \\ $I_{10} = \{w_y\le 0.0756, v-w_x\in[-0.362, 0.068], vw_x \in[0.777894, 1.11178]\}$.
                \item $I_7$ mapps to \\ $I_{11} = \{w_y\le 0.134, v-w_x\in[-0.394, 0.0782], vw_x \in[0.595, 1]\}$.
            \end{enumerate}
        \item Step 3: Splitting and merging $I_8,I_9,I_{10},I_{11}$, we have
            \begin{enumerate}
                \item $I_{12}=\{w_y\le 0.134, v-w_x\in[-0.394, 0.078], vw_x \in[0.595, 0.777]\}$.
                \item $I_{13}=\{w_y\le 0.214, v-w_x\in[-0.394, 0.078], vw_x \in[0.777, 1]\}$.
                \item $I_{14}=\{w_y\le 0.214, v-w_x\in[-0.362, 0.068], vw_x \in[1, 1.11178]\}$.
                \item $I_{15}=\{w_y\le 0.214, v-w_x\in[-0.309, 0.061], vw_x \in[1.11178, 1.25786]\}$.
            \end{enumerate}
            Then, we have 
            \begin{enumerate}
                \item $I_{12}$ mapps to \\ $I_{16} = \{w_y\le 0.0372, v-w_x\in[-0.317, 0.061], vw_x \in[1.14493, 1.31246]\}$.
                \item $I_{13}$ mapps to \\ $I_{17} = \{w_y\le 0.0432, v-w_x\in[-0.448, 0.078], vw_x \in[0.943633, 1.24393]\}$.
                \item $I_{14}$ mapps to \\ $I_{18} = \{w_y\le 0.0662, v-w_x\in[-0.462, 0.077], vw_x \in[0.77846, 1]\}$.
                \item $I_{15}$ mapps to \\ $I_{20} = \{w_y\le 0.0998, v-w_x\in[-0.456, 0.0785], vw_x \in[0.550, 0.878]\}$.
            \end{enumerate}
    \end{enumerate}
    
    \item Starting from $I_2$,
    \begin{enumerate}
        \item Step 1: $I_2$ mapps to $I_{21}=\{w_y\le 0.332, v-w_x\in[-0.205, 0.114], vw_x \in[0.864, 1.25894]\}$
        \item Step 2: Splitting $I_{21}$, we have
            \begin{enumerate}
                \item $I_{22} = \{w_y\le 0.332, v-w_x\in[-0.205, 0.114], vw_x \in[0.864, 1]\}.$
                \item $I_{23} = \{w_y\le 0.332, v-w_x\in[-0.205, 0.114], vw_x \in[1, 1.125894]\}.$
            \end{enumerate}
            Then, we have
            \begin{enumerate}
                \item $I_{22}$ mapps to \\ $I_{24} = \{w_y\le 0.081, v-w_x\in[-0.336, 0.114], vw_x \in[0.858, 1.14813]\}$.
                \item $I_{23}$ mapps to \\ $I_{25} = \{w_y\le 0.184, v-w_x\in[-0.409, 0.148], vw_x \in[0.463, 1]\}$.
            \end{enumerate}
        \item Step 3: Splitting and merging $I_{24},I_{25}$, we have
            \begin{enumerate}
                \item $I_{26}=\{w_y\le 0.184, v-w_x\in[-0.409, 0.148], vw_x \in[0.463, 1]\}$.
                \item $I_{27}=\{w_y\le 0.081, v-w_x\in[-0.336, 0.114], vw_x \in[1, 1.14813]\}$.
            \end{enumerate}
            Then, we have 
            \begin{enumerate}
                \item $I_{26}$ mapps to \\ $I_{28} = \{w_y\le 0.083, v-w_x\in[-0.452, 0.148], vw_x \in[0.952783, 1.31778]\}$.
                \item $I_{27}$ mapps to \\ $I_{29} = \{w_y\le 0.034, v-w_x\in[-0.399, 0.133], vw_x \in[0.777, 1]\}$.
            \end{enumerate}
    \end{enumerate}
    
    \item Starting from $I_{30}$,
        \begin{enumerate}
        \item Step 1: $I_{30}$ mapps to $I_{31}=\{w_y\le 0.44, v-w_x\in[-0.124, 0.037], vw_x \in[0.422, 0.767]\}$
        \item Step 2: Splitting $I_{31}$, we have
            \begin{enumerate}
                \item $I_{32} = \{w_y\le 0.44, v-w_x\in[-0.124, 0.037], vw_x \in[0.422, 0.5]\}.$
                \item $I_{33} = \{w_y\le 0.44, v-w_x\in[-0.124, 0.037], vw_x \in[0.5, 0.6]\}.$
                \item $I_{34} = \{w_y\le 0.44, v-w_x\in[-0.124, 0.037], vw_x \in[0.6, 0.767]\}.$
            \end{enumerate}
            Then, we have
            \begin{enumerate}
                \item $I_{32}$ mapps to \\ $I_{35} = \{w_y\le 0.301, v-w_x\in[-0.218, 0.0185], vw_x \in[0.901, 1.20971]\}$.
                \item $I_{33}$ mapps to \\ $I_{36} = \{w_y\le 0.262, v-w_x\in[-0.245, 0.023], vw_x \in[0.96322, 1.25093]\}$.
                \item $I_{34}$ mapps to \\ $I_{37} = \{w_y\le 0.213, v-w_x\in[-0.288, 0.029], vw_x \in[0.947, 1.25345]\}$.
            \end{enumerate}
        \item Step 3: Splitting and merging $I_{35},I_{36},I_{37}$, we have
            \begin{enumerate}
                \item $I_{38}=\{w_y\le 0.301, v-w_x\in[-0.288, 0.029], vw_x \in[0.901, 1]\}$.
                \item $I_{39}=\{w_y\le 0.301, v-w_x\in[-0.288, 0.029], vw_x \in[1, 1.1]\}$.
                \item $I_{40}=\{w_y\le 0.301, v-w_x\in[-0.288, 0.029], vw_x \in[1.1, 1.25093]\}$.
                \item $I_{41}=\{w_y\le 0.262, v-w_x\in[-0.245, 0.029], vw_x \in[1.25093, 1.25345]\}$.
            \end{enumerate}
            Then, we have 
            \begin{enumerate}
                \item $I_{38}$ mapps to \\ $I_{42} = \{w_y\le 0.0404, v-w_x\in[-0.392, 0.029], vw_x \in[0.888, 1.11696]\}$.
                \item $I_{39}$ mapps to \\ $I_{43} = \{w_y\le 0.0740, v-w_x\in[-0.428, 0.033], vw_x \in[0.741, 1]\}$.
                \item $I_{40}$ mapps to \\ $I_{44} = \{w_y\le 0.125, v-w_x\in[-0.482, 0.038], vw_x \in[0.497, 0.891]\}$.
                \item $I_{41}$ mapps to \\ $I_{45} = \{w_y\le 0.109, v-w_x\in[-0.400, 0.038], vw_x \in[0.534, 0.702]\}$.
            \end{enumerate}
        \item Step 4: Splitting and merging $I_{42},I_{43},I_{44},I_{45}$, we have
            \begin{enumerate}
                \item $I_{46}=\{w_y\le 0.125, v-w_x\in[-0.482, 0.038], vw_x \in[0.497, 0.891]\}$.
                \item $I_{47}=\{w_y\le 0.074, v-w_x\in[-0.428, 0.033], vw_x \in[0.891, 1]\}$.
                \item $I_{48}=\{w_y\le 0.041, v-w_x\in[-0.40, 0.029], vw_x \in[1, 1.11696]\}$.
            \end{enumerate}
            Then, we have 
            \begin{enumerate}
                \item $I_{46}$ mapps to \\ $I_{49} = \{w_y\le 0.0424, v-w_x\in[-0.442, 0.034], vw_x \in[1.07853, 1.34708]\}$.
                \item $I_{47}$ mapps to \\ $I_{50} = \{w_y\le 0.0110, v-w_x\in[-0.435, 0.033], vw_x \in[0.993, 1.13943]\}$.
                \item $I_{48}$ mapps to \\ $I_{51} = \{w_y\le 0.0109, v-w_x\in[-0.454, 0.033], vw_x \in[0.497, 0.891]\}$.
            \end{enumerate}
    \end{enumerate}
\end{enumerate}

\textbf{For Case (II):}

Now we are to compute the ranges starting from the interval where $I=\{w_y\le 0.37, v-w_x\in[-0.017, 0.17], vw_x \in[\frac{1}{4\pi}, \frac{1}{2\pi}]\}$. First, we denote it as
\begin{enumerate}
    \item $I_{52}=\{w_y\le 0.37, v-w_x\in[-0.017, 0.17], vw_x \in[\frac{1}{4\pi}, \frac{1}{2\pi}]$.
\end{enumerate}

Then, following the above method with splitting and merging intervals, we have
\begin{enumerate}
    \item Starting from $I_{52}$,
    \begin{enumerate}
        \item Step 1: $I_{52}$ mapps to $I_{53}=\{w_y\le 0.37, v-w_x\in[-0.079, 0.0271], vw_x \in[0.222, 0.616]\}$.
        \item Step 2: $I_{53}$ mapps to $I_{54}=\{w_y\le 0.343, v-w_x\in[-0.171, 0.017], vw_x \in[0.621, 1.24894]\}$.
        \item Step 3: Splitting $I_{54}$, we have
            \begin{enumerate}
                \item $I_{55} = \{w_y\le 0.343, v-w_x\in[-0.171, 0.017], vw_x \in[0.621, 1\}$.
                \item $I_{56} = \{w_y\le 0.343, v-w_x\in[-0.171, 0.017], vw_x \in[1, 1.24894]\}$.
            \end{enumerate}
            Then, we have
            \begin{enumerate}
                \item $I_{55}$ mapps to \\ $I_{57} = \{w_y\le 0.150, v-w_x\in[-0.305, 0.017], vw_x \in[0.840, 1.25908]\}$.
                \item $I_{56}$ mapps to \\ $I_{58} = \{w_y\le 0.137, v-w_x\in[-0.367, 0.022], vw_x \in[0.472, 1]\}$.
            \end{enumerate}
        \item Step 4: Splitting and merging $I_{57}, I_{58}$, we have
            \begin{enumerate}
                \item \item $I_{59} = \{w_y\le 0.150, v-w_x\in[-0.367, 0.022], vw_x \in[0.472, 1\}$.
                \item \item $I_{60} = \{w_y\le 0.150, v-w_x\in[-0.305, 0.017], vw_x \in[1, 1.25908\}$.
            \end{enumerate}
            Then, we have
            \begin{enumerate}
                \item $I_{59}$ mapps to \\ $I_{61} = \{w_y\le 0.0705, v-w_x\in[-0.393, 0.022], vw_x \in[0.971, 1.304]\}$.
                \item $I_{60}$ mapps to \\ $I_{62} = \{w_y\le 0.0613, v-w_x\in[-0.421, 0.0219], vw_x \in[0.583, 1]\}$.
            \end{enumerate}
    \end{enumerate}
    
\end{enumerate}

\textbf{For both Cases (I, II):}

From $I_{16-20}, I_{28}, I_{29}, I_{49-51}, I_{61}, I_{62}$, we can see that it has fallen into an interval $I_f=\{w_y<0.1, v-w_x\in[-0.462,0.148], vw_x\in[0.497,1.34078]\}$. Something special here is that $w_y$ has been much smaller than $w_x$. More broadly, let's define an interval $I_s$ generated by $I_g=\{w_y=0, v-w_x\in[-0.464,0.148], vw_x\in[1,1.5]\}$. Here ``generated'' means
\begin{align}
    I_s = \bigcup_{T\ge t} \{(v\T{T}, w_x\T{T}, w_y\T{T})|(v_t\t, w_x\t, w_y\t)\in I_g\}.
\end{align}

Then each element $(v,w_x,w_y)\in I_s$ has the following properties:
\begin{enumerate}
    \item $w_y=0$.
    \item $vw_x\in[0.181,1.5]$.
    \item If $vw_x\le 1$, then $v-w_x\in [-0.735, 0.23]$. If $vw_x>1$, then $v-w_x\in [-0.474, 0.148]$.
\end{enumerate}
The first property is obvious. The third can be proven as follows: for each element $(v, w_x, w_y)\in I_g$, it has $v\tt-w_x\tt=(v-w_x)\left(1+K(vw_x-1)\right)$, where the ratio $1+K(vw_x-1)\in[1,1+1.1(1.5-1)]$ when $vw_x\in[1,1.5]$. Furthermore, in the proven 2-D case, we have shown that ``if $vw_x>1$ with some mild conditions, then $\frac{v\T{t+2}-w_x\T{t+2}}{v-w_x}\in(-1,1)$''. Actually it can be tighter as $\frac{v\T{t+2}-w_x\T{t+2}}{v-w_x}\in(-0.2,1)$ because here $K\le 1.1$ while the original bound is for $K\le 1.5$. The condition of bounded $|v-w_x|$ can also be verified, the purpose of which is to keep $v,w_x$ always positive. Then the bound $[-0.2,1]$ will tell $v-w_x\in[-0.474,0.148]$ on $vw_x\ge 1$, because
\begin{align*}
    \frac{0.148}{0.474}>0.2,  ~~~~~~\frac{0.474}{0.148}>0.2.
\end{align*}
For the second property, the left bound can be verified as 
\begin{align*}
    \min_{x\in[1,1.5]}x(1+1.1(1-x))^2 + 1.1(1-x)\cdot 0.474^2 &= \bigg(x(1+1.1(1-x))^2 + 1.1(1-x)\cdot 0.474^2\bigg) \bigg|_{x=1.5} \\
    & \ge 0.181.
\end{align*}
The right bound can be verified as
\begin{align*}
    \max_{x\in[0,1]}x(1+1.1(1-x))^2 + 1.1(1-x)*0.735^2 &<1.5.
\end{align*}
After proving these three properties, we would like to bound how far $I_f$ is away from $I_s$. More precisely, the distance is measured by $w_y$. We are going to show $w_y$ decays exponentially.

Remind the update rules in (\ref{eq:deltav}, \ref{eq:deltawx}). Denote $\gamma = \frac{1}{\pi}(\arctan(\alpha)-\frac{w_x w_y}{\norm{w}^2})$ again and $\delta = K\frac{w_y^2}{v}(-v^2+\frac{v w_y}{\pi \norm{w}^2})$, then it is
\begin{align}
    \Delta v &= Kw_x(-vw_x+1)-Kw_x\gamma + \delta, \label{eq:deltavcorr}\\
    \Delta w_x &= Kv(-vw_x+1)-Kv\gamma, \\
    \delta &\in [-Kvw_y^2, 0].
\end{align}

Note that both $\gamma$ and $\delta$ are very small, so we are to show their effects separately, which is enough to be a good approximation.

Consider an iteration where $v\t w_x\t>1$ and the corresponding $\gamma\t$. Let's denote $v\tt,w_x\tt$ as the next parameters with \textbf{no corruption} from $\gamma\t$. Similarly, we denote $\hat{v}\tt, \hat{w_x}\tt$ are corrupted with $\gamma\t$. From the 2-D analysis, we know
\begin{align}
    \frac{v\T{t+2}-w_x\T{t+2}}{v\t-w_x\t} = (1+K(v\t w_x\t-1))(1+K(v\tt w_x\tt-1))<1. \label{eq:ratio_nocrpt}
\end{align}
We would like to show, with a small $\gamma\t$ and ignoring $\delta$,
\begin{align}
    \frac{\hat{v}\T{t+2}-\hat{w_x}\T{t+2}}{v\t-w_x\t} = (1+K(v\t w_x\t-1+\gamma\t))(1+K(\hat{v}\tt \hat{w_x}\tt-1+\gamma\tt))\lessapprox 1, \label{eq:ratio_crpt}
\end{align}
where $\gamma\tt$ is in time $(t+1)$ accordingly.
The difference of LHS of the above two expressions turns out to be
\begin{align}
    (\ref{eq:ratio_crpt})-(\ref{eq:ratio_nocrpt})
    &=
    K\gamma\t (1+K(v\tt w_x\tt-1)) \nonumber \\
    &~~~~~ + (1+K(v\t w_x\t-1))K(\hat{v}\tt \hat{w_x}\tt - v\tt w_x\tt+\gamma\tt) + \O(\gamma^2) \nonumber \\
    &=
    K\gamma\t (1+K(v\tt w_x\tt-1)) \nonumber \\
    &~~~~~ + K(1+K(v\t w_x\t-1))(-K(v\t)^2 \gamma\t-K(w_x\t)^2 \gamma\t+\gamma\tt) + \O(\gamma^2) \nonumber \\
    &\le K\gamma\t\bigg(
    1+(1+K(v\t w_x\t-1))\big(
    -K(v\t)^2-K(w_x\t)^2+\frac{\gamma\tt}{\gamma\t}
    \big)
    \bigg)+\O(\gamma^2) \nonumber \\
    &\le K\gamma\t\bigg(
    1+(1+K(v\t w_x\t-1))\big(
    -2K v\t w_x\t +\frac{\gamma\tt}{\gamma\t}
    \big)
    \bigg)+\O(\gamma^2). \label{eq:diff_ratios}
\end{align}

Since $\frac{\Delta w_x}{w_x} = K\frac{v}{w_x}(-vw_x +1-\gamma)$, we have
\begin{align}
    \frac{w_x\tt}{w_x\t}=1+K\frac{v\t}{w_x\t}(-v\t w_x\t +1-\gamma\t)<1.
\end{align}
Also we have
\begin{align}
    \frac{\gamma\tt}{\gamma\t} = \frac{\arctan(\frac{w_y\tt}{w_x\tt})-\frac{w_x\tt w_y\tt}{\norm{w\tt}^2}}{\arctan(\frac{w_y\t}{w_x\t})-\frac{w_x\t w_y\t}{\norm{w\t}^2}}.
\end{align}
Since $w_y\tt\le w_y\t$ and
\begin{align}
    \frac{\arctan(mx)-\frac{mx}{1+m^2x^2}}{\arctan(x)-\frac{x}{1+x^2}}\le m^3,\text{  for any $m>0, x>0$},
\end{align}
we have 
\begin{align}
    \frac{\gamma\tt}{\gamma\t}\le \frac{1}{\left(1+K\frac{v\t}{w_x\t}(-v\t w_x\t +1-\gamma\t)\right)^3}. \label{eq:ratiogamma}
\end{align}

For general $vw_x\in(1,1.5]$, (\ref{eq:ratiogamma}) holds as
\begin{align}
    \frac{\gamma\tt}{\gamma\t}\lessapprox \frac{1}{(1+1.1\frac{\sqrt{1+0.074^2}+0.074}{\sqrt{1+0.074^2}-0.074}(-1.5+1))^3} \le 22. 
\end{align}
Since $1+K(v\t w_x\t-1)\le 1+1.1*0.5=1.55$, it is fair to say
\begin{align}
    (\ref{eq:ratio_crpt})-(\ref{eq:ratio_nocrpt})\lessapprox  K\gamma\t(1+1.55*(-2+22)) + \O(\gamma^2) = 35.2\gamma\t + \O(\gamma^2).
\end{align}
Actually $\gamma\t$ is bounded by
\begin{align}
     \frac{w_y\t}{w_x\t} &\le \frac{0.099}{\sqrt{1+0.074^2}-0.074}=0.1066,\\
    \gamma\t &\le \frac{\arctan(x)-\frac{x}{1+x^2}}{\pi}\le 2.6\times 10^{-4}.
\end{align}
As a result, 
\begin{align}
    (\ref{eq:ratio_crpt})-(\ref{eq:ratio_nocrpt})\lessapprox 0.0084.
\end{align}
Note that this small value is very easy to cover in (\ref{eq:ratio_nocrpt}), requiring
\begin{align}
    1-\frac{v\T{t+2}-w_x\T{t+2}}{v\t-w_x\t}\ge 0.0084,
\end{align}
except when $vw_x$ is pretty close to 1. When $vw_x\xrightarrow{} 1$, from the analysis of 2-D case, (derived from the case of $x_{t+1} y_{t+1}\ge x_s^2$)
\begin{align}
    1-\frac{v\T{t+2}-w_x\T{t+2}}{v\t-w_x\t} \ge (2K-2)(v\t w_x\t -1).
\end{align}
For $(\ref{eq:ratio_crpt})-(\ref{eq:ratio_nocrpt})$, denote a function $p(x)$ as
\begin{align}
    p(x) = 1+(1+Kx)\left(-2K(x+1)+\frac{1}{\left(1+K\frac{v}{w_x}(-x)\right)^3}\right),
\end{align}
where $x=v\t w_x\t - 1$ in (\ref{eq:diff_ratios}, \ref{eq:ratiogamma}). It is obvious that $p(0)=1+(-2K+1)<0$. When $x$ is small, it turns out
\begin{align}
    p(x) = -2K + 2 + K\left(-2K-1+3\frac{v\t}{w_x\t}\right)x + \O(x^2) \label{eq:px0}
\end{align}
As a result, $(\ref{eq:ratio_crpt})-(\ref{eq:ratio_nocrpt})<0$ when $v w_x-1=o(K-1)$. What if $vw_x-1=\Omega(K-1)$? Actually, we can get a better bound by a more care analysis, as
\begin{align}
        \frac{(\ref{eq:ratio_crpt})-(\ref{eq:ratio_nocrpt})}{K\gamma\t}
        &\le 
        1+(1+K(v\t w_x\t-1))\big(
        -K(v\t)^2-K(w_x\t)^2+\frac{\gamma\tt}{\gamma\t}
        \big)
        \nonumber\\
        & ~~~~ + K\left[ v\t w_x\t(1+K(1-v\t w_x\t))^2-1\right],
\end{align}
where the last term is due to $v\tt w_x\tt \le v\t w_x\t(1+K(1-v\t w_x\t))^2$. Hence, with this bound, by expanding the last term, (\ref{eq:px0}) becomes
\begin{align}
    p(x) &= -2K + 2 + K\left(-2K-1+3\frac{v\t}{w_x\t}\right)x + K(1-2K)x + \O(x^2) \\
    &= -2K+2 + K\left(-4K+3\frac{v\t}{w_x\t}\right)x+\O(x^2),
\end{align}
which is definitely negative because 
\begin{align}
    \frac{v\t}{w_x\t}\le\frac{\sqrt{1+0.074^2}+0.074}{\sqrt{1+0.074^2}-0.074}<1.16<\frac{4}{3}.
\end{align}

Meanwhile, we are to prove the $\delta$ in (\ref{eq:deltavcorr}) will not make $\tilde{I}_s$ make $v-w_x<-0.474$ starting from $v-w_x\ge -0.462$. First, in the region of $\{vw_x\in[1,1.5], v-w_x\le 0.148\}$, we have $Kvw_y^2\le 1.1\cdot(\sqrt{1.5+0.074^2}+0.074)*0.1^2\le 0.0144$. Also note that in this region with $v-w_x\ge -0.462$, we have
\begin{align}
    \frac{w_y\tt}{w_y}\le 1-\frac{\sqrt{1+0.231^2}-0.231}{\sqrt{1+0.231^2}+0.231}=0.37.
\end{align}
Hence $Kv(w_y\t)^2+Kv(w_y\tt)^2\le 0.0144*(1+0.37^2)=0.0164$. Since $|v\T{t+2}-w\T{t+2}|< |v\t-w\t|$ if there is no $\delta$, we shall see that there is no need to discuss the case of $v-w_x\ge -0.462+0.0164=-0.4456$ because it still holds $v\tt -w_x\tt>-0.462$. When $v\t-w_x\t\in[-0.462, -0.4456]$, we shall see that in (\ref{eq:xybalance}), after adding the term of $\delta$ in $v$,
\begin{align}
    \frac{v\T{t+2}-w_x\T{t+2}}{v\t-w_x\t}\le 1-(1+K(v w_x -1))\cdot K w_x \delta,
\end{align}
which means the absolute value of $v-w_x$ decays at least by a margin depending on $\delta$. After multiplying the current difference $v\t-w_x\t$ on both side, it gives
\begin{align}
    (v\T{t+2}-w_x\T{t+2})-(v\t-w_x\t) \ge v\t w_x\t w_x\t \delta.
\end{align}
Note that here $v\T{t+2}-w_x\T{t+2}$ does not include $\delta\t$ and $\delta\tt$. As stated above, we have $\frac{\delta\tt}{\delta\t}\le 0.37^2\le 0.16$ due to the decay of $w_y$. So it is safe to say $\delta\t + \delta\tt\ge 1.16\delta\t$. Combining with the above inequality, it gives
\begin{align}
    (v\T{t+2}-w_x\T{t+2})-(v\t-w_x\t) +\delta\t + \delta\tt \ge (v\t w_x\t w_x\t + 1.16)\delta\t, \label{eq:deltat1}
\end{align}
where
\begin{align}
    v\t w_x\t w_x\t + 1.16 \le  vw_x \cdot (\sqrt{vw_x +(\frac{0.4456}{2})^2}-\frac{0.4456}{2})+1.16 \le 0.6.
\end{align}
Furthermore, from our previous discussion, $w_y\T{t+2}<w_y\T{t+2}$ gives that the sum of (\ref{eq:deltat1}) is bounded by
\begin{align}
    \frac{0.6}{1-0.16}\delta\t \ge \frac{0.6}{1-0.16}\cdot (-0.0144) \ge -0.0103.
\end{align}
Since $-0.474 - (-0.462) < -0.0103$, we shall see that the term of $\delta$ cannot drive $v-w_x<-0.472$. Note that (\ref{eq:deltat1}) shall include a factor ($<1)$ in front of $\delta\t$, but we have ignored it to show a more aggressive bound.

Therefore, we are able to say an Interval $\hat{I}_s$ generated by $I_f$ also has the following properties: for each element $(v,w_x,w_y)\in \hat{I}_s$,
\begin{enumerate}
    \item $vw_x\in[0.181,1.5]$.
    \item If $vw_x\le 1$, then $v-w_x\in [-0.735, 0.23]$. If $vw_x>1$, then $v-w_x\in [-0.472, 0.148]$.
\end{enumerate}

Then the decreasing ratio of $\Delta w_y/w_y$ is bounded by
\begin{align}
    \frac{\Delta w_y}{w_y} &= K\frac{v}{w_x}\left(-vw_x + \frac{1}{\pi}\frac{w_x w_y}{\norm{w}^2}\right)\\
    &\in
    \left[
        -1.1(\sqrt{1.5+0.074^2}+0.074)^2, -0.030 \red{K}
    \right]\\
    &=[-1.87, -0.030\red{K}].
\end{align}
Hence, $w_y$ decays with a linear ratio of $0.97$ \red{(or $1-0.030K$)} at most for Cases (I, II) in stage 2.

For Case (III), in the first step of stage 2, it already has $w_y\le 0.078$ and $v-w_x\in[-0.017, 0.17]$. So surely it will also converge to $I_s$.

Here we present the time analysis for Case (III) of both stages. The number of iterations in the first stage is apparently similar to that of case (I, II), as
\begin{align}
    T_1 \le \log_{2.56} \left\lceil\frac{2.7 \psi}{\beta^2} \right\rceil,
\end{align}
where $\psi<\frac{1}{4\pi}$ is the value of $vw_x$ in the first step of stage 2. In stage 2, since our target is to find how many steps are necessary to get $vw_x\ge 0.181$, so it is
\begin{align}
    v\tt w_x\tt
    &\ge v\t w_x\t\left(
    1-0.181+1-\frac{\arctan(2-\sqrt{3})-\frac{2-\sqrt{3}}{1+(2-\sqrt{3})^2}}{\pi} - 1.1 w_y^2
    \right)\\
    &\ge 3.28 v\t w_x\t.
\end{align}
where obviously it still holds $\frac{w_y}{w_x}\le 2-\sqrt{3}$ and $w_y^2<0.1^2$ in stage 2. Since $3.28>2.56$, we have the total number of steps to have $vw_x>0.181$ bounded as
\begin{align}
    \left\lceil\log_{2.56} \frac{2.7 \psi}{\beta^2}\right\rceil + \left\lceil\log_{3.28} \frac{0.181}{\psi}\right\rceil 
    &\le
    \left\lceil\log_{2.56} \frac{0.675}{\pi\beta^2}\right\rceil + \left\lceil\log_{3.28} \frac{0.181}{\frac{1}{4\pi}}\right\rceil + 2\nonumber \\
    &\le \left\lceil\log_{2.56} \frac{0.675}{\pi\beta^2}\right\rceil + 3 \nonumber \\
    &< \left\lceil\log_{2.56} \frac{1.35}{\pi\beta^2}\right\rceil + 4, \nonumber
\end{align} 
which is not beyond the bound for Cases (I, II).
\end{proof}

\section{Proof of Matrix Factorization}

Consider a two-layer matrix factorization problem, parameterized by learnable weights $\mathbf{X}\in\R^{m\times p}$, $\mathbf{Y}\in\R^{p\times q}$, and the target matrix is $\mathbf{C}\in\R^{m\times q}$. The loss $L$ is defined as
\begin{align}
    L(\mathbf{X},\mathbf{Y})=\frac{1}{2}\norm{\mathbf{X}\mathbf{Y} - \mathbf{C}}_F^2.
\end{align}
Obviously $\{\mathbf{X},\mathbf{Y}:\mathbf{X}\mathbf{Y}=\mathbf{C}\}$ forms a minimum manifold. Focusing on this manifold, our targets are: 1) to prove our condition for stable oscillation on 1D functions holds at the minimum of $L$ for any setting of dimensions, and 2) 
to provide an observation of walking towards flattest minima with theoretical intuition.

\subsection{Asymmetric Case: 1D function at the minima} \label{sec:app_mf_1d}

Before looking into the theorem, we would like to clarify the definition of the loss Hessian. Inherently, we squeeze $\mathbf{X},\mathbf{Y}$ into a vector $\theta=\text{vec}(\mathbf{X},\mathbf{Y})\in\R^{mp+pq}$, which vectorizes the concatnation. As a result, we are able to represent the loss Hessian w.r.t. $\theta$ as a matrix in $\R^{(mp+pq)\times(mp+pq)}$. Meanwhile, the support of the loss landscape is in $\R^{mp+pq}$. In the following theorem, we are to show the leading eigenvector $\Delta\triangleq \text{vec}(\Delta\mathbf{X}, \Delta\mathbf{Y})\in\R^{mp+pq}$ of the loss Hessian. Since the cross section of the loss landscape and $\Delta$ forms a 1D function $f_{\Delta}$, we would also show the stable-oscillation condition on 1D function holds at the minima of $f_{\Delta}$.

\begin{theorem}
    For a matrix factorization problem, assume $\mathbf{X}\mathbf{Y}=\mathbf{C}$.
    Consider SVD of both matrices as $\mathbf{X} = \sum_{i=1}^{\min\{m,p\}} \sigma_{x,i} u_{x,i} v_{x,i}^\top$ and $\mathbf{Y} = \sum_{i=1}^{\min\{p,q\}} \sigma_{y,i} u_{y,i} v_{y,i}^\top$, where both groups of $\sigma_{\cdot,i}$'s are in descending order and both top singular values $\sigma_{x,1}$ and $\sigma_{y,1}$ are unique. Also assume $v_{x,1}^\top u_{y,1}\neq 0$. Then the leading eigenvector of the loss Hessian is $\Delta=\text{vec}(C_1 u_{x,1}u_{y,1}^\top, C_2 v_{x,1}v_{y,1}^\top)$ with $C_1=\frac{\sigma_{y,1}}{\sqrt{\sigma_{x,1}^2+\sigma_{y,1}^2}}, C_2=\frac{\sigma_{x,1}}{\sqrt{\sigma_{x,1}^2+\sigma_{y,1}^2}}$. Denote $f_\Delta$ as the 1D function at the cross section of the loss landscape and the line following the direction of $\Delta$ passing $\text{vec}(\Delta\mathbf{X}, \Delta\mathbf{Y})$. Then, at the minima of $f_\Delta$, it satisfies
    \begin{align}
        3[f_\Delta^{(3)}]^2 - f_\Delta^{(2)}f_\Delta^{(4)}>0.
    \end{align}
\end{theorem}

\begin{proof}
    To obtain the direction of the leading Hessian eigenvector at parameters $(\mathbf{X}, \mathbf{Y})$, consider a small deviation of the parameters as $(\mathbf{X}+\Delta \mathbf{X},\mathbf{Y}+\Delta \mathbf{Y})$. With $\mathbf{X}\mathbf{Y}=\mathbf{C}$, evaluate the loss function as
    \begin{align}
        L(\mathbf{X}+\Delta \mathbf{X},\mathbf{Y}+\Delta \mathbf{Y}) &= 
        \frac{1}{2}\norm{\Delta\mathbf{X}\mathbf{Y} + \mathbf{X}\Delta\mathbf{Y}+\Delta\mathbf{X}\Delta\mathbf{Y}}_F^2.
    \end{align}
    Expand these terms and split them by orders of $\Delta\mathbf{X},\Delta\mathbf{Y}$ as follows:
    \begin{align}
        \Theta(\norm{\Delta\mathbf{X}}^2+\norm{\Delta\mathbf{Y}}^2):& 
        ~~~~~\frac{1}{2}\norm{\Delta\mathbf{X}\mathbf{Y} + \mathbf{X}\Delta\mathbf{Y}}_F^2, \label{eq:mfloss_2nd}\\
        \Theta(\norm{\Delta\mathbf{X}}^3+\norm{\Delta\mathbf{Y}}^3):& 
        ~~~~~\langle \Delta\mathbf{X}\mathbf{Y} + \mathbf{X}\Delta\mathbf{Y}, \Delta\mathbf{X}\Delta\mathbf{Y} \rangle, \label{eq:mfloss_3rd}\\
        \Theta(\norm{\Delta\mathbf{X}}^4+\norm{\Delta\mathbf{Y}}^4):& 
        ~~~~~\frac{1}{2}\norm{\Delta\mathbf{X}\Delta\mathbf{Y}}_F^2. \label{eq:mfloss_4th}
    \end{align}
    From the second-order terms, the leading eigenvector of $\nabla^2 L$ is the solution of 
    \begin{align}
        \text{vec}(\Delta\mathbf{X}, \Delta\mathbf{Y}) = \mathop{\mathrm{arg\,max}}_{\norm{\Delta\mathbf{X}}_F^2+\norm{\Delta\mathbf{Y}}_F^2 = 1}\norm{\Delta\mathbf{X}\mathbf{Y} + \mathbf{X}\Delta\mathbf{Y}}_F^2.
    \end{align}
    Since both the top singular values of $\mathbf{X},\mathbf{Y}$ are unique, the solution shall have both $\Delta\mathbf{X}, \Delta\mathbf{Y}$ of rank 1. Actually the solution is (here for simplicity we eliminate the sign of both)
    \begin{align}
        \Delta\mathbf{X} = \frac{\sigma_{y,1}}{\sqrt{\sigma_{x,1}^2+\sigma_{y,1}^2}} u_{x,1}u_{y,1}^\top, ~~~~\Delta\mathbf{Y}=\frac{\sigma_{x,1}}{\sqrt{\sigma_{x,1}^2+\sigma_{y,1}^2}} v_{x,1}v_{y,1}^\top.
    \end{align}
    Equipped with the top eigenvector of Hessian, $\text{vec}(\Delta\mathbf{X}, \Delta\mathbf{Y})$, we consider the 1-D function $f_{\Delta}$ generated by the cross-section of the loss landscape and the eigenvector, passing the minima $(\mathbf{X},\mathbf{Y})$. 
    Define the function as 
    \begin{align}
        f_{\Delta}(\mu) = L(\mathbf{X}+\mu\Delta\mathbf{X}, \mathbf{Y}+\mu\Delta\mathbf{Y}), ~~~~\mu\in\R.
    \end{align}
    Then, around $\mu=0$, we have
    \begin{align}
        f_{\Delta}(\mu) = \frac{1}{2}\norm{\Delta\mathbf{X}\mathbf{Y} + \mathbf{X}\Delta\mathbf{Y}}_F^2\cdot \mu^2 + \langle \Delta\mathbf{X}\mathbf{Y} + \mathbf{X}\Delta\mathbf{Y}, \Delta\mathbf{X}\Delta\mathbf{Y} \rangle \cdot \mu^3 + 
        \frac{1}{2}\norm{\Delta\mathbf{X}\Delta\mathbf{Y}}_F^2 \cdot \mu^4.
    \end{align}
    Therefore, the several order derivatives of $f_{\Delta}(\mu)$ at $\mu=0$ can be obtained from Taylor expansion as
    \begin{align}
        f^{(2)}_\Delta (0) &= \norm{\Delta\mathbf{X}\mathbf{Y} + \mathbf{X}\Delta\mathbf{Y}}_F^2, \\
        f^{(3)}_\Delta (0) &= 6\langle \Delta\mathbf{X}\mathbf{Y} + \mathbf{X}\Delta\mathbf{Y}, \Delta\mathbf{X}\Delta\mathbf{Y} \rangle, \\
        f^{(4)}_\Delta (0) &= 12\norm{\Delta\mathbf{X}\Delta\mathbf{Y}}_F^2.
    \end{align}
    Then we compute the condition of stable oscillation of 1-D function as
    \begin{align}
        \big[3[f_\Delta^{(3)}]^2 - f_\Delta^{(2)}f_\Delta^{(4)}\big] (0)
        &=
        108 \langle \Delta\mathbf{X}\mathbf{Y} + \mathbf{X}\Delta\mathbf{Y}, \Delta\mathbf{X}\Delta\mathbf{Y} \rangle^2 - 12 \norm{\Delta\mathbf{X}\mathbf{Y} + \mathbf{X}\Delta\mathbf{Y}}_F^2 \norm{\Delta\mathbf{X}\Delta\mathbf{Y}}_F^2 \\
        &= 
        96 \norm{\Delta\mathbf{X}\mathbf{Y} + \mathbf{X}\Delta\mathbf{Y}}_F^2 \norm{\Delta\mathbf{X}\Delta\mathbf{Y}}_F^2 >0,
    \end{align}
    because all of $\Delta\mathbf{X}\mathbf{Y}, \mathbf{X}\Delta\mathbf{Y}, \Delta\mathbf{X}\Delta\mathbf{Y}$ are parallel to $u_{x,1} v_{y,1}^\top$ and $v_{x,1}^\top u_{y,1}\neq 0$.
    
\end{proof}

\subsection{Quasi-symmetric case: walk towards flattest minima} \label{sec:app_mf_quasisym}
\begin{obs}[Restatement of Observation~\ref{obs:quasi_mf}]
    Consider the quasi-symmetric matrix factorization with learning rate $\eta = \frac{1}{\sigma_1^2}+\beta$. Assume $0<\beta\sigma_1^2< \sqrt{4.5}-1\approx 1.121$. {Consider a minimum $(\bY_0=\alpha\bX_0, \bZ_0=\nicefrac{1}{\alpha}\mathbf{X}_0), \alpha>0$}. The initialization is around the minimum, as ${\mathbf{Y}_1 = \bY_0 +\Delta \mathbf{Y}_1, \mathbf{Z}_1 = \bZ_0 + \Delta \mathbf{Z}_1}$, with the deviations satisfying $u_1^\top {\Delta\mathbf{Y}_{1}} v_1 \neq 0, u_1^\top {\Delta\mathbf{Z}_1} v_1 \neq 0$ and $\norm{{\Delta\mathbf{Y}_1}}, \norm{{\Delta\mathbf{Z}_1}}\le \epsilon$. The second largest singular value of $\bX_0$ needs to satisfy
    \begin{align}
        \eta\cdot\max\left\{
                (\frac{\sigma_1^2}{\alpha^2}+\sigma_2^2\alpha^2, \frac{\sigma_2^2}{\alpha^2}+\sigma_1^2\alpha^2)
            \right\}
        \le 2.
    \end{align}
    Then GD would converge to a period-2 orbit $\gamma_{\eta}$ approximately with error in $\O(\epsilon)$, formally written as
    \begin{gather}
        (\mathbf{Y}_t, \mathbf{Z}_t) \rightarrow \gamma_{\eta} + (\Delta\mathbf{Y}, \Delta\mathbf{Z}), 
        ~~~~~~\norm{\Delta\mathbf{Y}}, \norm{\Delta\mathbf{Z}}=\O(\epsilon), \\\
        \gamma_{\eta} = \bigg\{\left(\mathbf{Y}_0 + \left(\rho_i-\alpha\right) \sigma_1 u_1 v_1^\top , \mathbf{Z}_0 + \left(\rho_i-\nicefrac{1}{\alpha}\right) \sigma_1 u_1 v_1^\top\right)\bigg\}, ~~~~~~(i=1,2)
    \end{gather}
    where $\rho_1\in (1,2),\rho_2\in(0,1)$ are the two solutions of solving $\rho$ in
    \begin{align}
        1+\beta \sigma_1^2 = \frac{1}{\rho^2 \left(\sqrt{\frac{1}{\rho^2}-\frac{3}{4}}+\frac{1}{2} \right)}.
    \end{align}
\end{obs}

\begin{remark}
    What is missing for a rigorous proof?
    \begin{enumerate}
        \item Control of error terms in non-asymptotic analysis.
        \item Resolving assumptions of spectrum $\mathbf{Q}_{\alpha,\eta,p}(y_t, z_t)$ in early stages.
    \end{enumerate}
\end{remark}

\begin{proof}
    Without loss of generality, we assume $\bX_0=\text{diag}([\sigma_1,\sigma_2,\dots,\sigma_d])\in\R^{d\times d}$, where $(\bX_0)_{i,i}=\sigma_i$ or $0$ in all other entries. This can be easily achieved by rotating singular vectors of $\bX_0$. Accordingly, we have $\bY_0=\text{diag}([\sigma_1\alpha,\sigma_2\alpha,\dots,\sigma_d\alpha])\in\R^{d\times d}$ and $\bZ_0=\text{diag}([\sigma_1/\alpha,\sigma_2/\alpha,\dots,\sigma_d/\alpha])\in\R^{d\times d}$.

    Starting from time $t=1$, we denote the learnable parameter matrices as $\bY_t,\bZ_t$, and their deviation as $\Delta\bY_t\triangleq \bY_t-\bY_0, \Delta\bZ_t\triangleq \bZ_t-\bZ_0$. By assumptions, we have $\norm{\Delta\bY_1}<\epsilon, \norm{\Delta\bZ_1}<\epsilon$. Furthermore, we split $\Delta\bY_t, \Delta\bZ_t$ as follows,
    \begin{gather}
        \Delta\bY_t=
            \left[ \begin{array}{c|c}
               \cc{1}_t & \cc{3}_t \\
               \midrule
               \cc{2}_t & \cc{4}_t \\
            \end{array}
            \right],
        \Delta\bZ_t=
            \left[ \begin{array}{c|c}
               \cc{5}_t & \cc{7}_t \\
               \midrule
               \cc{6}_t & \cc{8}_t \\
            \end{array}
            \right],\\
        \cc{1}_t, \cc{5}_t\in\R,~~~~ \cc{2}_t, \cc{6}_t\in\R^{(d-1)\times 1},~~~~ \cc{3}_t, \cc{7}_t\in\R^{1\times (d-1)},~~~~ \cc{4}_t, \cc{8}_t\in\R^{(d-1) \times (d-1)}.
    \end{gather}
    Since the update rules of $\bY_t, \bZ_t$ are
    \begin{align}
        \bY_{t+1} &= \bY_{t} - \eta\left(\Delta \bY_t \bZ_0^\top + \bY_0 \Delta\bZ_t^\top + \Delta\bY_t\Delta\bZ_t^\top\right)\left(\bZ_0+\Delta\bZ_t\right) \\
        \bZ_{t+1} &= \bZ_t - \eta\left(\Delta\bZ_t\bY_0^\top + \bZ_0\Delta\bY_t^\top + \Delta\bZ_t\Delta\bY_t^\top\right)\left(\bY_0+\Delta\bY_t\right)
    \end{align}
    The update rules of $\cc{1}-\cc{8}$ are
    \begin{align}
        \cc{1}_{t+1} &= \cc{1}_t-\eta\mathbb{I}_1^\top\left(\Delta \bY_t \bZ_0^\top + \bY_0 \Delta\bZ_t^\top + \Delta\bY_t\Delta\bZ_t^\top\right)(\bZ_0+\Delta\bZ_t)\mathbb{I}_1 \\
        \cc{2}_{t+1} &= \cc{2}_t-\eta\mathbb{I}_{\ge 2}^\top\left(\Delta \bY_t \bZ_0^\top + \bY_0 \Delta\bZ_t^\top + \Delta\bY_t\Delta\bZ_t^\top\right)(\bZ_0+\Delta\bZ_t)\mathbb{I}_1 \\
        \cc{3}_{t+1} &= \cc{3}_t-\eta\mathbb{I}_{1}^\top\left(\Delta \bY_t \bZ_0^\top + \bY_0 \Delta\bZ_t^\top + \Delta\bY_t\Delta\bZ_t^\top\right)(\bZ_0+\Delta\bZ_t)\mathbb{I}_{\ge 2} \\
        \cc{4}_{t+1} &= \cc{4}_t-\eta\mathbb{I}_{\ge 2}^\top\left(\Delta \bY_t \bZ_0^\top + \bY_0 \Delta\bZ_t^\top + \Delta\bY_t\Delta\bZ_t^\top\right)(\bZ_0+\Delta\bZ_t)\mathbb{I}_{\ge 2} \\
        \cc{5}_{t+1} &= \cc{5}_t-\eta\mathbb{I}_1^\top\left(\Delta\bZ_t\bY_0^\top + \bZ_0\Delta\bY_t^\top + \Delta\bZ_t\Delta\bY_t^\top\right)\left(\bY_0+\Delta\bY_t\right)\mathbb{I}_1 \\
        \cc{6}_{t+1} &= \cc{6}_t-\eta\mathbb{I}_{\ge 2}^\top\left(\Delta\bZ_t\bY_0^\top + \bZ_0\Delta\bY_t^\top + \Delta\bZ_t\Delta\bY_t^\top\right)\left(\bY_0+\Delta\bY_t\right)\mathbb{I}_1 \\
        \cc{7}_{t+1} &= \cc{7}_t-\eta\mathbb{I}_{1}^\top\left(\Delta\bZ_t\bY_0^\top + \bZ_0\Delta\bY_t^\top + \Delta\bZ_t\Delta\bY_t^\top\right)\left(\bY_0+\Delta\bY_t\right)\mathbb{I}_{\ge 2} \\
        \cc{8}_{t+1} &= \cc{8}_t-\eta\mathbb{I}_{\ge 2}^\top\left(\Delta\bZ_t\bY_0^\top + \bZ_0\Delta\bY_t^\top + \Delta\bZ_t\Delta\bY_t^\top\right)\left(\bY_0+\Delta\bY_t\right)\mathbb{I}_{\ge 2},
    \end{align}
    where $\mathbb{I}_1=(\mathbb{I}_{d})_{:,1}\in\R^{d\times 1}, \mathbb{I}_{\ge 2}=(\mathbb{I}_d)_{:,2:d}\in\R^{d\times(d-1)}$ are the dimension-reduction matrix, defined from blocks of the $d\times d$ identity matrix $\mathbb{I}$. In other words, $\mathbb{I}_1$ (respectively $\mathbb{I}_{\ge2}$) is to pick the first row/column (respectively all remaining rows/columns) from a matrix, which is extracting $\cc{1}_t - \cc{8}_t$ from $\Delta\bY_t,\Delta\bZ_t$.

    Denote $\mathbf{M}_t\triangleq \left(\Delta \bY_{t} \bZ_0^\top + \bY_0 \Delta\bZ_{t}^\top + \Delta\bY_{t}\Delta\bZ_{t}^\top\right)=\bY_t\bZ_t^\top-\bX_0\bX_0^\top$.

    At initialization, we assume all of $\cc{1}_1, \cc{2}_1, \cc{3}_1, \cc{4}_1, \cc{5}_1, \cc{6}_1, \cc{7}_1, \cc{8}_1$ are in $\Theta(\epsilon)$, which means all $\norm{\mathbb{I}_{1}\mathbf{M}_1\mathbb{I}_{1}}, \norm{\mathbb{I}_{\ge2}\mathbf{M}_1\mathbb{I}_{1}}, \norm{\mathbb{I}_{1}\mathbf{M}_1\mathbb{I}_{\ge2}}, \norm{\mathbb{I}_{\ge2}\mathbf{M}_1\mathbb{I}_{\ge2}}$ are in $\Theta(\epsilon)$ as well. Our goal is to show that, as $t\rightarrow \infty$,
    \begin{enumerate}
        \item $\cc{1}_\infty$, $\cc{5}_\infty$ are in a period-2 orbit,
        \item $\cc{2}_\infty, \cc{3}_\infty, \cc{4}_\infty, \cc{6}_\infty, \cc{7}_\infty, \cc{8}_\infty$ are in $\Theta(\epsilon)$,
        \item $\norm{\mathbb{I}_{\ge2}\mathbf{M}_\infty\mathbb{I}_{1}}, \norm{\mathbb{I}_{1}\mathbf{M}_\infty\mathbb{I}_{\ge2}}, \norm{\mathbb{I}_{\ge2}\mathbf{M}_\infty\mathbb{I}_{\ge2}}$, $\norm{\mathbb{I}_{1}^\top\bZ_\infty\bZ_{\infty}^\top\mathbb{I}_{\ge2}}$, $\norm{\mathbb{I}_{\ge2}^\top\bY_\infty\bY_{\infty}^\top\mathbb{I}_{\ge2}}$ decay to zero.
    \end{enumerate}

    Then, following the above definitions, we have another representation of $\left(\Delta \bY_t \bZ_0^\top + \bY_0 \Delta\bZ_t^\top + \Delta\bY_t\Delta\bZ_t^\top\right)$, or equivalently its transpose $\left(\Delta\bZ_t\bY_0^\top + \bZ_0\Delta\bY_t^\top + \Delta\bZ_t\Delta\bY_t^\top\right)$, as 
    \begin{align}
        \mathbb{I}_1^\top\left(\Delta \bY_t \bZ_0^\top + \bY_0 \Delta\bZ_t^\top + \Delta\bY_t\Delta\bZ_t^\top\right)\mathbb{I}_1 
        &=
        \cc{1}_t\mathbb{I}_1^\top \bZ_0^\top\mathbb{I}_{1}+\mathbb{I}_1^\top \bY_0 \mathbb{I}_1 \cc{5}_t + \cc{1}_t\cc{5}_t + \cc{3}_t\cc{7}_t^\top \\
        \mathbb{I}_{\ge2}^\top\left(\Delta \bY_t \bZ_0^\top + \bY_0 \Delta\bZ_t^\top + \Delta\bY_t\Delta\bZ_t^\top\right)\mathbb{I}_1 
        &=
        \cc{2}_t\mathbb{I}_1^\top \bZ_0^\top\mathbb{I}_{1}+\mathbb{I}_{\ge2}^\top \bY_0 \mathbb{I}_{\ge2} \cc{7}_t^\top + \cc{2}_t\cc{5}_t + \cc{4}_t\cc{7}_t^\top \\
        \mathbb{I}_{1}^\top\left(\Delta \bY_t \bZ_0^\top + \bY_0 \Delta\bZ_t^\top + \Delta\bY_t\Delta\bZ_t^\top\right)\mathbb{I}_{\ge2} 
        &=
        \cc{3}_t\mathbb{I}_{\ge2}^\top \bZ_0^\top\mathbb{I}_{\ge2}+\mathbb{I}_{1}^\top \bY_0 \mathbb{I}_{1} \cc{6}_t^\top + \cc{1}_t\cc{6}_t^\top + \cc{3}_t\cc{8}_t^\top \\
        \mathbb{I}_{\ge2}^\top\left(\Delta \bY_t \bZ_0^\top + \bY_0 \Delta\bZ_t^\top + \Delta\bY_t\Delta\bZ_t^\top\right)\mathbb{I}_{\ge2}
        &=
        \cc{4}_t\mathbb{I}_{\ge2}^\top \bZ_0^\top\mathbb{I}_{\ge2}+\mathbb{I}_{\ge2}^\top \bY_0 \mathbb{I}_{\ge2} \cc{8}_t^\top + \cc{2}_t\cc{6}_t^\top + \cc{4}_t\cc{8}_t^\top.
    \end{align}
     After substituting with $\cc{1}_{t+1}-\cc{8}_{t+1}$, we have
    \begin{align*}
        \mathbb{I}_{1}^\top\mathbf{M}_{t+1} \mathbb{I}_{1}
        &=
            \mathbb{I}_{1}^\top\mathbf{M}_{t} \mathbb{I}_{1}
            - \eta\mathbb{I}_{1}^\top \mathbf{M}_{t} (\bZ_0+\Delta\bZ_t)\mathbb{I}_{1}\mathbb{I}_{1}^\top \bZ_0^\top\mathbb{I}_{1} 
            - \eta \mathbb{I}_{1}^\top \bY_0 \mathbb{I}_{1}\mathbb{I}_{1}^\top \left(\bY_0+\Delta\bY_t\right)^\top \mathbf{M}_{t} \mathbb{I}_{1}  \\
        &~~~~ 
            -\eta \cc{1}_t \mathbb{I}_{1}^\top\left(\bY_0+\Delta\bY_t\right)^\top\mathbf{M}_{t}\mathbb{I}_{1} 
            - \eta \mathbb{I}_{1}^\top\mathbf{M}_{t}(\bZ_0+\Delta\bZ_t)\mathbb{I}_1\cc{5}_t 
            - \eta\mathbb{I}_{1}^\top\mathbf{M}_{t}(\bZ_0+\Delta\bZ_t)\mathbb{I}_{\ge 2}\cc{7}_t^\top \\
        &~~~~ 
            -\eta \cc{3}_t \mathbb{I}_{\ge2}^\top\left(\bY_0+\Delta\bY_t\right)^\top \mathbf{M}_{t}\mathbb{I}_{1} 
            + \eta^2 \mathbb{I}_{1}^\top\mathbf{M}_{t}(\bZ_0+\Delta\bZ_t)\mathbb{I}_{\ge2}\mathbb{I}_{\ge2}^\top\left(\bY_0+\Delta\bY_t\right)^\top \mathbf{M}_{t}\mathbb{I}_{1}
        \\
        &~~~~
            + \eta^2\mathbb{I}_{1}^\top\mathbf{M}_{t}(\bZ_0+\Delta\bZ_t)\mathbb{I}_1\mathbb{I}_{1}^\top\left(\bY_0+\Delta\bY_t\right)^\top\mathbf{M}_{t}\mathbb{I}_{1} 
        \\
        &=
            \mathbb{I}_{1}^\top\mathbf{M}_{t} \mathbb{I}_{1}
            - \eta\mathbb{I}_{1}^\top \mathbf{M}_{t}(\mathbb{I}_{1}\mathbb{I}_{1}^\top+\mathbb{I}_{\ge 2}\mathbb{I}_{\ge 2}^\top) (\bZ_0+\Delta\bZ_t)\mathbb{I}_{1}\mathbb{I}_{1}^\top \bZ_0^\top\mathbb{I}_{1} 
            \\
            &~~~~ - \eta \mathbb{I}_{1}^\top \bY_0 \mathbb{I}_{1}\mathbb{I}_{1}^\top \left(\bY_0+\Delta\bY_t\right)^\top (\mathbb{I}_{1}\mathbb{I}_{1}^\top+\mathbb{I}_{\ge 2}\mathbb{I}_{\ge 2}^\top)\mathbf{M}_{t}^\top \mathbb{I}_{1}  
        \\
        &~~~~ 
            -\eta \cc{1}_t \mathbb{I}_{1}^\top\left(\bY_0+\Delta\bY_t\right)^\top(\mathbb{I}_{1}\mathbb{I}_{1}^\top+\mathbb{I}_{\ge 2}\mathbb{I}_{\ge 2}^\top)\mathbf{M}_{t}\mathbb{I}_{1} 
            - \eta \mathbb{I}_{1}^\top\mathbf{M}_{t}(\mathbb{I}_{1}\mathbb{I}_{1}^\top+\mathbb{I}_{\ge 2}\mathbb{I}_{\ge 2}^\top)(\bZ_0+\Delta\bZ_t)\mathbb{I}_1\cc{5}_t 
        \\
        &~~~~ 
            - \eta\mathbb{I}_{1}^\top\mathbf{M}_{t}(\mathbb{I}_{1}\mathbb{I}_{1}^\top+\mathbb{I}_{\ge 2}\mathbb{I}_{\ge 2}^\top)(\bZ_0+\Delta\bZ_t)\mathbb{I}_{\ge 2}\cc{7}_t^\top 
            -\eta \cc{3}_t \mathbb{I}_{\ge2}^\top\left(\bY_0+\Delta\bY_t\right)^\top (\mathbb{I}_{1}\mathbb{I}_{1}^\top+\mathbb{I}_{\ge 2}\mathbb{I}_{\ge 2}^\top)\mathbf{M}_{t}\mathbb{I}_{1} 
        \\
        &~~~~
            + \eta^2 \mathbb{I}_{1}^\top\mathbf{M}_{t}(\mathbb{I}_{1}\mathbb{I}_{1}^\top+\mathbb{I}_{\ge 2}\mathbb{I}_{\ge 2}^\top)(\bZ_0+\Delta\bZ_t)\mathbb{I}_{\ge2}\mathbb{I}_{\ge2}^\top\left(\bY_0+\Delta\bY_t\right)^\top (\mathbb{I}_{1}\mathbb{I}_{1}^\top+\mathbb{I}_{\ge 2}\mathbb{I}_{\ge 2}^\top)\mathbf{M}_{t}\mathbb{I}_{1}
        \\
        &~~~~ 
            + \eta^2 \mathbb{I}_{1}^\top\mathbf{M}_{t}(\mathbb{I}_{1}\mathbb{I}_{1}^\top+\mathbb{I}_{\ge 2}\mathbb{I}_{\ge 2}^\top)(\bZ_0+\Delta\bZ_t)\mathbb{I}_{1}\mathbb{I}_{1}^\top\left(\bY_0+\Delta\bY_t\right)^\top (\mathbb{I}_{1}\mathbb{I}_{1}^\top+\mathbb{I}_{\ge 2}\mathbb{I}_{\ge 2}^\top)\mathbf{M}_{t}\mathbb{I}_{1}
        \\
        &=
            \mathbb{I}_{1}^\top\mathbf{M}_{t} \mathbb{I}_{1}
            - \eta\mathbb{I}_{1}^\top \mathbf{M}_{t}\mathbb{I}_{1}\mathbb{I}_{1}^\top\bZ_0\mathbb{I}_{1}\mathbb{I}_{1}^\top \bZ_0^\top\mathbb{I}_{1}
            - \eta\mathbb{I}_{1}^\top \mathbf{M}_{t}\mathbb{I}_{1}\cc{5}_t\mathbb{I}_{1}^\top \bZ_0^\top\mathbb{I}_{1}
            - \eta\mathbb{I}_{1}^\top \mathbf{M}_{t}\mathbb{I}_{\ge2}\cc{6}_t\mathbb{I}_{1}^\top \bZ_0^\top\mathbb{I}_{1}
        \\
        &~~~~ 
            - \eta \mathbb{I}_{1}^\top \bY_0 \mathbb{I}_{1}\mathbb{I}_{1}^\top \bY_0\mathbb{I}_{1}\mathbb{I}_{1}^\top\mathbf{M}_{t} \mathbb{I}_{1} 
            - \eta \mathbb{I}_{1}^\top \bY_0 \mathbb{I}_{1}\cc{1}_t\mathbb{I}_{1}^\top\mathbf{M}_{t} \mathbb{I}_{1}
            - \eta \mathbb{I}_{1}^\top \bY_0 \mathbb{I}_{1}\cc{2}_t^\top\mathbb{I}_{\ge2}^\top\mathbf{M}_{t} \mathbb{I}_{1}
        \\
        &~~~~ 
            - \eta \cc{1}_t\mathbb{I}_{1}^\top \bY_0\mathbb{I}_{1}\mathbb{I}_{1}^\top\mathbf{M}_{t} \mathbb{I}_{1} 
            - \eta \cc{1}_t\cc{1}_t\mathbb{I}_{1}^\top\mathbf{M}_{t} \mathbb{I}_{1}
            - \eta \cc{1}_t\cc{2}_t^\top\mathbb{I}_{\ge2}^\top\mathbf{M}_{t} \mathbb{I}_{1}
        \\
        &~~~~ 
            - \eta\mathbb{I}_{1}^\top \mathbf{M}_{t}\mathbb{I}_{1}\mathbb{I}_{1}^\top\bZ_0\mathbb{I}_{1}\cc{5}_t 
            - \eta\mathbb{I}_{1}^\top \mathbf{M}_{t}\mathbb{I}_{1}\cc{5}_t\cc{5}_t 
            - \eta\mathbb{I}_{1}^\top \mathbf{M}_{t}\mathbb{I}_{\ge2}\cc{6}_t\cc{5}_t 
        \\
        &~~~~ 
            - \eta\mathbb{I}_{1}^\top\mathbf{M}_{t}\mathbb{I}_{1}\mathbb{I}_{1}^\top\bZ_0\mathbb{I}_{\ge 2}\cc{7}_t^\top 
            - \eta\mathbb{I}_{1}^\top\mathbf{M}_{t}\mathbb{I}_{1}\cc{7}_t\cc{7}_t^\top 
            - \eta\mathbb{I}_{1}^\top\mathbf{M}_{t}\mathbb{I}_{1}\cc{8}_t\cc{7}_t^\top 
        \\
        &~~~~ 
            - \eta \cc{3}_t\mathbb{I}_{\ge 2}^\top \bY_0\mathbb{I}_{\ge2}\mathbb{I}_{\ge2}^\top\mathbf{M}_{t} \mathbb{I}_{1} 
            - \eta \cc{3}_t\cc{4}_t^\top\mathbb{I}_{1}^\top\mathbf{M}_{t} \mathbb{I}_{1}
            - \eta \cc{3}_t\cc{3}_t^\top\mathbb{I}_{\ge2}^\top\mathbf{M}_{t} \mathbb{I}_{1}
        \\
        &~~~~ 
            + \eta^2 \mathbb{I}_{1}^\top\mathbf{M}_{t}
            (\mathbb{I}_{\ge2}\mathbb{I}_{\ge2}^\top\bZ_0\mathbb{I}_{\ge2}+\mathbb{I}_{1}\cc{7}_t+\mathbb{I}_{\ge2}\cc{8}_t)
            (\mathbb{I}_{\ge2}^\top\bY_0\mathbb{I}_{\ge2}\mathbb{I}_{\ge2}^\top+\cc{3}_t^\top\mathbb{I}_{1}^\top + \cc{4}_t^\top\mathbb{I}_{\ge2}^\top)
            \mathbf{M}_{t}\mathbb{I}_{1}
        \\
        &~~~~ 
            + \eta^2 \mathbb{I}_{1}^\top\mathbf{M}_{t}
            (\mathbb{I}_{1}\mathbb{I}_{1}^\top\bZ_0\mathbb{I}_{1}+\mathbb{I}_{1}\cc{5}_t+\mathbb{I}_{\ge2}\cc{6}_t)
            (\mathbb{I}_{1}^\top\bY_0\mathbb{I}_{1}\mathbb{I}_{1}^\top+\cc{1}_t\mathbb{I}_{1}^\top + \cc{2}_t^\top\mathbb{I}_{\ge2}^\top)
            \mathbf{M}_{t}\mathbb{I}_{1}
    \end{align*}
    \begin{align*}
        \mathbb{I}_{\ge2}^\top\mathbf{M}_{t+1} \mathbb{I}_{1}
        &=
        \mathbb{I}_{\ge2}^\top\mathbf{M}_{t}\mathbb{I}_{1}
        - \eta\mathbb{I}_{\ge2}^\top \mathbf{M}_{t} (\bZ_0+\Delta\bZ_t)\mathbb{I}_{1}\mathbb{I}_{1}^\top \bZ_0^\top\mathbb{I}_{1} - \eta \mathbb{I}_{\ge2}^\top \bY_0 \mathbb{I}_{\ge2}\mathbb{I}_{\ge2}^\top \left(\bY_0+\Delta\bY_t\right)^\top \mathbf{M}_{t} \mathbb{I}_{1}  \\
        &~~~~ -\eta \cc{2}_t \mathbb{I}_{1}^\top\left(\bY_0+\Delta\bY_t\right)^\top\mathbf{M}_{t}\mathbb{I}_{1} - \eta \mathbb{I}_{\ge2}^\top\mathbf{M}_{t}(\bZ_0+\Delta\bZ_t)\mathbb{I}_1\cc{5}_t - \eta\mathbb{I}_{\ge2}^\top\mathbf{M}_{t}(\bZ_0+\Delta\bZ_t)\mathbb{I}_{\ge 2}\cc{7}_t^\top \\
        &~~~~ -\eta \cc{4}_t \mathbb{I}_{\ge 2}^\top\left(\bY_0+\Delta\bY_t\right)^\top \mathbf{M}_{t}\mathbb{I}_{1} + \eta^2 \mathbb{I}_{\ge2}^\top\mathbf{M}_{t}(\bZ_0+\Delta\bZ_t)\mathbb{I}_{\ge 2}\mathbb{I}_{\ge 2}^\top\left(\bY_0+\Delta\bY_t\right)^\top \mathbf{M}_{t}\mathbb{I}_{1} \\
        &~~~~ + \eta^2 \mathbb{I}_{\ge2}^\top\mathbf{M}_{t}(\bZ_0+\Delta\bZ_t)\mathbb{I}_{1}\mathbb{I}_{1}^\top\left(\bY_0+\Delta\bY_t\right)^\top \mathbf{M}_{t}\mathbb{I}_{1}
        \\
        &=
            \mathbb{I}_{\ge2}^\top\mathbf{M}_{t}\mathbb{I}_{1}
            - \eta\mathbb{I}_{\ge2}^\top \mathbf{M}_{t}(\mathbb{I}_{1}\mathbb{I}_{1}^\top+\mathbb{I}_{\ge 2}\mathbb{I}_{\ge 2}^\top) 
            (\bZ_0+\Delta\bZ_t)\mathbb{I}_{1}\mathbb{I}_{1}^\top \bZ_0^\top\mathbb{I}_{1} 
            \\
            &~~~~ - \eta \mathbb{I}_{\ge2}^\top \bY_0 \mathbb{I}_{\ge2}\mathbb{I}_{\ge2}^\top \left(\bY_0+\Delta\bY_t\right)^\top (\mathbb{I}_{1}\mathbb{I}_{1}^\top+\mathbb{I}_{\ge 2}\mathbb{I}_{\ge 2}^\top)\mathbf{M}_{t} \mathbb{I}_{1} 
        \\
        &~~~~ 
            -\eta \cc{2}_t \mathbb{I}_{1}^\top\left(\bY_0+\Delta\bY_t\right)^\top(\mathbb{I}_{1}\mathbb{I}_{1}^\top+\mathbb{I}_{\ge 2}\mathbb{I}_{\ge 2}^\top)\mathbf{M}_{t}\mathbb{I}_{1} 
            - \eta \mathbb{I}_{\ge2}^\top\mathbf{M}_{t}(\mathbb{I}_{1}\mathbb{I}_{1}^\top+\mathbb{I}_{\ge 2}\mathbb{I}_{\ge 2}^\top)(\bZ_0+\Delta\bZ_t)\mathbb{I}_1\cc{5}_t 
        \\
        &~~~~ 
            - \eta\mathbb{I}_{\ge2}^\top\mathbf{M}_{t}(\mathbb{I}_{1}\mathbb{I}_{1}^\top+\mathbb{I}_{\ge 2}\mathbb{I}_{\ge 2}^\top)(\bZ_0+\Delta\bZ_t)\mathbb{I}_{\ge 2}\cc{7}_t^\top 
            -\eta \cc{4}_t \mathbb{I}_{\ge 2}^\top\left(\bY_0+\Delta\bY_t\right)^\top (\mathbb{I}_{1}\mathbb{I}_{1}^\top+\mathbb{I}_{\ge 2}\mathbb{I}_{\ge 2}^\top)\mathbf{M}_{t}\mathbb{I}_{1} 
        \\
        &~~~~
            + \eta^2 \mathbb{I}_{\ge2}^\top\mathbf{M}_{t}(\mathbb{I}_{1}\mathbb{I}_{1}^\top+\mathbb{I}_{\ge 2}\mathbb{I}_{\ge 2}^\top)(\bZ_0+\Delta\bZ_t)\mathbb{I}_{\ge 2}\mathbb{I}_{\ge 2}^\top\left(\bY_0+\Delta\bY_t\right)^\top (\mathbb{I}_{1}\mathbb{I}_{1}^\top+\mathbb{I}_{\ge 2}\mathbb{I}_{\ge 2}^\top)\mathbf{M}_{t}\mathbb{I}_{1}
        \\
        &~~~~
            + \eta^2 \mathbb{I}_{\ge2}^\top\mathbf{M}_{t}(\mathbb{I}_{1}\mathbb{I}_{1}^\top+\mathbb{I}_{\ge 2}\mathbb{I}_{\ge 2}^\top)(\bZ_0+\Delta\bZ_t)\mathbb{I}_{1}\mathbb{I}_{1}^\top\left(\bY_0+\Delta\bY_t\right)^\top (\mathbb{I}_{1}\mathbb{I}_{1}^\top+\mathbb{I}_{\ge 2}\mathbb{I}_{\ge 2}^\top)\mathbf{M}_{t}\mathbb{I}_{1}
        \\
        &=
            \mathbb{I}_{\ge2}^\top\mathbf{M}_{t}\mathbb{I}_{1}
            - \eta\mathbb{I}_{\ge2}^\top \mathbf{M}_{t}\mathbb{I}_{1}\mathbb{I}_{1}^\top\bZ_0\mathbb{I}_{1}\mathbb{I}_{1}^\top \bZ_0^\top\mathbb{I}_{1} 
            - \eta\mathbb{I}_{\ge2}^\top \mathbf{M}_{t}\mathbb{I}_{1}\cc{5}_t\mathbb{I}_{1}^\top \bZ_0^\top\mathbb{I}_{1} 
            - \eta\mathbb{I}_{\ge2}^\top \mathbf{M}_{t}\mathbb{I}_{\ge2}\cc{6}_t\mathbb{I}_{1}^\top \bZ_0^\top\mathbb{I}_{1} 
        \\
        &~~~~ 
            - \eta \mathbb{I}_{\ge2}^\top \bY_0 \mathbb{I}_{\ge2}\mathbb{I}_{\ge2}^\top \bY_0\mathbb{I}_{\ge 2}\mathbb{I}_{\ge 2}^\top\mathbf{M}_{t} \mathbb{I}_{1}
            - \eta \mathbb{I}_{\ge2}^\top \bY_0 \mathbb{I}_{\ge2}\cc{3}_t^\top\mathbb{I}_{1}^\top\mathbf{M}_{t} \mathbb{I}_{1}
            - \eta \mathbb{I}_{\ge2}^\top \bY_0 \mathbb{I}_{\ge2}\cc{4}_t^\top\mathbb{I}_{\ge2}^\top\mathbf{M}_{t} \mathbb{I}_{1}
        \\
        &~~~~ 
            -\eta \cc{2}_t \mathbb{I}_{1}^\top\bY_0\mathbb{I}_{1}\mathbb{I}_{1}^\top\mathbf{M}_{t}\mathbb{I}_{1} 
            -\eta \cc{2}_t \cc{1}_t\mathbb{I}_{1}^\top\mathbf{M}_{t}\mathbb{I}_{1} 
            -\eta \cc{2}_t \cc{2}_t^\top\mathbb{I}_{\ge2}^\top\mathbf{M}_{t}\mathbb{I}_{1} 
        \\
        &~~~~ 
            - \eta\mathbb{I}_{\ge2}^\top \mathbf{M}_{t}\mathbb{I}_{1}\mathbb{I}_{1}^\top\bZ_0\mathbb{I}_{1}\cc{5}_t 
            - \eta\mathbb{I}_{\ge2}^\top \mathbf{M}_{t}\mathbb{I}_{1}\cc{5}_t\cc{5}_t 
            - \eta\mathbb{I}_{\ge2}^\top \mathbf{M}_{t}\mathbb{I}_{\ge2}\cc{6}_t\cc{5}_t 
        \\
        &~~~~ 
            - \eta\mathbb{I}_{\ge2}^\top\mathbf{M}_{t}\mathbb{I}_{\ge 2}\mathbb{I}_{\ge 2}^\top\bZ_0\mathbb{I}_{\ge 2}\cc{7}_t^\top 
            - \eta\mathbb{I}_{\ge2}^\top\mathbf{M}_{t}\mathbb{I}_{1}\cc{7}_t\cc{7}_t^\top 
            - \eta\mathbb{I}_{\ge2}^\top\mathbf{M}_{t}\mathbb{I}_{\ge2}\cc{8}_t\cc{7}_t^\top 
        \\
        &~~~~
            - \eta \cc{4}_t\mathbb{I}_{\ge2}^\top \bY_0\mathbb{I}_{\ge 2}\mathbb{I}_{\ge 2}^\top\mathbf{M}_{t}^\top \mathbb{I}_{1}
            - \eta \cc{4}_t\cc{3}_t^\top\mathbb{I}_{1}^\top\mathbf{M}_{t} \mathbb{I}_{1}
            - \eta \cc{4}_t\cc{4}_t^\top\mathbb{I}_{\ge2}^\top\mathbf{M}_{t} \mathbb{I}_{1}
        \\
        &~~~~
            + \eta^2 \mathbb{I}_{\ge2}^\top\mathbf{M}_{t}
            (\mathbb{I}_{\ge 2}\mathbb{I}_{\ge 2}^\top\bZ_0\mathbb{I}_{\ge 2}+\mathbb{I}_1\cc{7}_t+\mathbb{I}_{\ge2}\cc{8}_t)
            (\mathbb{I}_{\ge 2}^\top\bY_0\mathbb{I}_{\ge 2}\mathbb{I}_{\ge 2}^\top + \cc{3}_t^\top \mathbb{I}_1^\top + \cc{4}_t^\top \mathbb{I}_{\ge2}^\top)
            \mathbf{M}_{t}\mathbb{I}_{1}
        \\
        &~~~~
            + \eta^2 \mathbb{I}_{\ge2}^\top\mathbf{M}_{t}
            (\mathbb{I}_{1}\mathbb{I}_{1}^\top\bZ_0\mathbb{I}_1^\top+\mathbb{I}_{1}\cc{5}_t+\mathbb{I}_{\ge2}\cc{6}_t)
            (\mathbb{I}_{1}^\top\bY_0\mathbb{I}_{1}\mathbb{I}_1^\top+\cc{1}_t\mathbb{I}_{1}^\top+\cc{2}_t^\top\mathbb{I}_{\ge2}^\top)\mathbf{M}_{t}\mathbb{I}_{1}
        \\
        \end{align*}
    \begin{align*}
        \mathbb{I}_{1}^\top\mathbf{M}_{t+1} \mathbb{I}_{\ge2}
        &=
        \mathbb{I}_{1}^\top\mathbf{M}_{t}\mathbb{I}_{\ge2}
        - \eta\mathbb{I}_{1}^\top \mathbf{M}_{t} (\bZ_0+\Delta\bZ_t)\mathbb{I}_{\ge 2}\mathbb{I}_{\ge2}^\top \bZ_0^\top\mathbb{I}_{\ge2} - \eta \mathbb{I}_1^\top \bY_0 \mathbb{I}_1\mathbb{I}_1^\top \left(\bY_0+\Delta\bY_t\right)^\top \mathbf{M}_{t} \mathbb{I}_{\ge 2}  \\
        &~~~~ -\eta \cc{1}_t \mathbb{I}_{1}^\top\left(\bY_0+\Delta\bY_t\right)^\top\mathbf{M}_{t}\mathbb{I}_{\ge 2} - \eta \mathbb{I}_{1}^\top\mathbf{M}_{t}(\bZ_0+\Delta\bZ_t)\mathbb{I}_1\cc{6}_t^\top - \eta\mathbb{I}_{1}^\top\mathbf{M}_{t}(\bZ_0+\Delta\bZ_t)\mathbb{I}_{\ge 2}\cc{8}_t^\top \\
        &~~~~ -\eta \cc{3}_t \mathbb{I}_{\ge 2}^\top\left(\bY_0+\Delta\bY_t\right)^\top \mathbf{M}_{t}\mathbb{I}_{\ge 2} + \eta^2 \mathbb{I}_{1}^\top\mathbf{M}_{t}(\bZ_0+\Delta\bZ_t)\mathbb{I}_{\ge 2}\mathbb{I}_{\ge 2}^\top\left(\bY_0+\Delta\bY_t\right)^\top \mathbf{M}_{t}\mathbb{I}_{\ge 2} \\
        &~~~~ + \eta^2 \mathbb{I}_{1}^\top\mathbf{M}_{t}(\bZ_0+\Delta\bZ_t)\mathbb{I}_{1}\mathbb{I}_{1}^\top\left(\bY_0+\Delta\bY_t\right)^\top \mathbf{M}_{t}\mathbb{I}_{\ge 2}
        \\
        &=
            \mathbb{I}_{1}^\top\mathbf{M}_{t}\mathbb{I}_{\ge2}
            - \eta\mathbb{I}_{1}^\top \mathbf{M}_{t} (\mathbb{I}_{1}\mathbb{I}_{1}^\top+\mathbb{I}_{\ge 2}\mathbb{I}_{\ge 2}^\top)(\bZ_0+\Delta\bZ_t)\mathbb{I}_{\ge 2}\mathbb{I}_{\ge2}^\top \bZ_0^\top\mathbb{I}_{\ge2} 
            \\
            &~~~~ - \eta \mathbb{I}_1^\top \bY_0 \mathbb{I}_1\mathbb{I}_1^\top \left(\bY_0+\Delta\bY_t\right)^\top (\mathbb{I}_{1}\mathbb{I}_{1}^\top+\mathbb{I}_{\ge 2}\mathbb{I}_{\ge 2}^\top)\mathbf{M}_{t} \mathbb{I}_{\ge 2}  
        \\
        &~~~~ 
            -\eta \cc{1}_t \mathbb{I}_{1}^\top\left(\bY_0+\Delta\bY_t\right)^\top(\mathbb{I}_{1}\mathbb{I}_{1}^\top+\mathbb{I}_{\ge 2}\mathbb{I}_{\ge 2}^\top)\mathbf{M}_{t}\mathbb{I}_{\ge 2} 
            - \eta \mathbb{I}_{1}^\top\mathbf{M}_{t}(\mathbb{I}_{1}\mathbb{I}_{1}^\top+\mathbb{I}_{\ge 2}\mathbb{I}_{\ge 2}^\top)(\bZ_0+\Delta\bZ_t)\mathbb{I}_1\cc{6}_t^\top 
        \\
        &~~~~
            - \eta\mathbb{I}_{1}^\top\mathbf{M}_{t}(\mathbb{I}_{1}\mathbb{I}_{1}^\top+\mathbb{I}_{\ge 2}\mathbb{I}_{\ge 2}^\top)(\bZ_0+\Delta\bZ_t)\mathbb{I}_{\ge 2}\cc{8}_t^\top 
            -\eta \cc{3}_t \mathbb{I}_{\ge 2}^\top\left(\bY_0+\Delta\bY_t\right)^\top (\mathbb{I}_{1}\mathbb{I}_{1}^\top+\mathbb{I}_{\ge 2}\mathbb{I}_{\ge 2}^\top)\mathbf{M}_{t}\mathbb{I}_{\ge 2} 
        \\
        &~~~~
            + \eta^2 \mathbb{I}_{1}^\top\mathbf{M}_{t}(\mathbb{I}_{1}\mathbb{I}_{1}^\top+\mathbb{I}_{\ge 2}\mathbb{I}_{\ge 2}^\top)(\bZ_0+\Delta\bZ_t)\mathbb{I}_{\ge 2}\mathbb{I}_{\ge 2}^\top\left(\bY_0+\Delta\bY_t\right)^\top (\mathbb{I}_{1}\mathbb{I}_{1}^\top+\mathbb{I}_{\ge 2}\mathbb{I}_{\ge 2}^\top)\mathbf{M}_{t}\mathbb{I}_{\ge 2}
        \\
        &~~~~
            + \eta^2 \mathbb{I}_{1}^\top\mathbf{M}_{t}(\mathbb{I}_{1}\mathbb{I}_{1}^\top+\mathbb{I}_{\ge 2}\mathbb{I}_{\ge 2}^\top)(\bZ_0+\Delta\bZ_t)\mathbb{I}_{1}\mathbb{I}_{1}^\top\left(\bY_0+\Delta\bY_t\right)^\top (\mathbb{I}_{1}\mathbb{I}_{1}^\top+\mathbb{I}_{\ge 2}\mathbb{I}_{\ge 2}^\top)\mathbf{M}_{t}\mathbb{I}_{\ge 2}
        \\
        &=
            \mathbb{I}_{1}^\top\mathbf{M}_{t}\mathbb{I}_{\ge2}
            - \eta\mathbb{I}_{1}^\top \mathbf{M}_{t} \mathbb{I}_{\ge 2}\mathbb{I}_{\ge 2}^\top \bZ_0 \mathbb{I}_{\ge 2}\mathbb{I}_{\ge2}^\top \bZ_0^\top\mathbb{I}_{\ge2} 
            - \eta\mathbb{I}_{1}^\top \mathbf{M}_{t} \mathbb{I}_{1}\cc{7}_t\mathbb{I}_{\ge2}^\top \bZ_0^\top\mathbb{I}_{\ge2} 
            - \eta\mathbb{I}_{1}^\top \mathbf{M}_{t} \mathbb{I}_{\ge2}\cc{8}_t\mathbb{I}_{\ge2}^\top \bZ_0^\top\mathbb{I}_{\ge2} 
        \\
        &~~~~
            - \eta \mathbb{I}_1^\top \bY_0 \mathbb{I}_1\mathbb{I}_1^\top \bY_0\mathbb{I}_{1}\mathbb{I}_{1}^\top\mathbf{M}_{t} \mathbb{I}_{\ge 2}  
            - \eta \mathbb{I}_1^\top \bY_0 \mathbb{I}_1\cc{1}_t\mathbb{I}_{1}^\top\mathbf{M}_{t} \mathbb{I}_{\ge 2}
            - \eta \mathbb{I}_1^\top \bY_0 \mathbb{I}_1\cc{2}_t^\top\mathbb{I}_{\ge2}^\top\mathbf{M}_{t} \mathbb{I}_{\ge 2}
        \\
        &~~~~ 
            - \eta \cc{1}_t \mathbb{I}_1\mathbb{I}_1^\top \bY_0\mathbb{I}_{1}\mathbb{I}_{1}^\top\mathbf{M}_{t} \mathbb{I}_{\ge 2}  
            - \eta \cc{1}_t\cc{1}_t\mathbb{I}_{1}^\top\mathbf{M}_{t}\mathbb{I}_{\ge 2}
            - \eta \cc{1}_t\cc{2}_t^\top\mathbb{I}_{\ge2}^\top\mathbf{M}_{t} \mathbb{I}_{\ge 2}
        \\
        &~~~~ 
            - \eta \mathbb{I}_{1}^\top\mathbf{M}_{t}\mathbb{I}_{1}\mathbb{I}_{1}^\top\bZ_0\mathbb{I}_1\cc{6}_t^\top 
            - \eta \mathbb{I}_{1}^\top\mathbf{M}_{t}\mathbb{I}_{1}\cc{5}_t\cc{6}_t^\top 
            - \eta \mathbb{I}_{1}^\top\mathbf{M}_{t}\mathbb{I}_{\ge2}\cc{6}_t\cc{6}_t^\top 
        \\
        &~~~~
            - \eta\mathbb{I}_{1}^\top \mathbf{M}_{t} \mathbb{I}_{\ge 2}\mathbb{I}_{\ge 2}^\top \bZ_0 \mathbb{I}_{\ge 2} \cc{8}_t^\top 
            - \eta\mathbb{I}_{1}^\top \mathbf{M}_{t} \mathbb{I}_{1}\cc{7}_t\cc{8}_t^\top 
            - \eta\mathbb{I}_{1}^\top \mathbf{M}_{t} \mathbb{I}_{\ge2}\cc{8}_t \cc{8}_t^\top 
        \\
        &~~~~
            -\eta \cc{3}_t \mathbb{I}_{\ge 2}^\top\bY_0\mathbb{I}_{\ge 2}\mathbb{I}_{\ge 2}^\top\mathbf{M}_{t}\mathbb{I}_{\ge 2} 
            -\eta \cc{3}_t \cc{3}_t^\top \mathbb{I}_{1}^\top\mathbf{M}_{t}\mathbb{I}_{\ge 2} 
            -\eta \cc{3}_t \cc{4}_t^\top \mathbb{I}_{\ge2}^\top\mathbf{M}_{t}\mathbb{I}_{\ge 2} 
        \\
        &~~~~
            + \eta^2 \mathbb{I}_{1}^\top\mathbf{M}_{t}
            (\mathbb{I}_{\ge 2}\mathbb{I}_{\ge 2}^\top\bZ_0\mathbb{I}_{\ge 2}+\mathbb{I}_1\cc{7}_t+\mathbb{I}_{\ge2}\cc{8}_t)
            (\mathbb{I}_{\ge 2}^\top\bY_0\mathbb{I}_{\ge 2}\mathbb{I}_{\ge 2}^\top + \cc{3}_t^\top \mathbb{I}_1^\top + \cc{4}_t^\top \mathbb{I}_{\ge2}^\top)
            \mathbf{M}_{t}\mathbb{I}_{\ge 2}
        \\
        &~~~~
            + \eta^2 \mathbb{I}_{1}^\top\mathbf{M}_{t}
            (\mathbb{I}_{1}\mathbb{I}_{1}^\top\bZ_0\mathbb{I}_1^\top+\mathbb{I}_{1}\cc{5}_t+\mathbb{I}_{\ge2}\cc{6}_t)
            (\mathbb{I}_{1}^\top\bY_0\mathbb{I}_{1}\mathbb{I}_1^\top+\cc{1}_t\mathbb{I}_{1}^\top+\cc{2}_t^\top\mathbb{I}_{\ge2}^\top)\mathbf{M}_{t}\mathbb{I}_{\ge 2}
        \\
        \end{align*}
        \begin{align*}
        \mathbb{I}_{\ge2}^\top\mathbf{M}_{t+1} \mathbb{I}_{\ge2}
        &=
        \mathbb{I}_{\ge2}^\top\mathbf{M}_{t}\mathbb{I}_{\ge2} - \eta\mathbb{I}_{\ge 2}^\top \mathbf{M}_{t} (\bZ_0+\Delta\bZ_t)\mathbb{I}_{\ge 2}\mathbb{I}_{\ge2}^\top \bZ_0^\top\mathbb{I}_{\ge2} - \eta \mathbb{I}_{\ge2}^\top \bY_0 \mathbb{I}_{\ge2}\mathbb{I}_{\ge 2}^\top \left(\bY_0+\Delta\bY_t\right)^\top \mathbf{M}_{t} \mathbb{I}_{\ge 2}  \\
        &~~~~ -\eta \cc{2}_t \mathbb{I}_{1}^\top\left(\bY_0+\Delta\bY_t\right)^\top\mathbf{M}_{t}\mathbb{I}_{\ge 2} - \eta \mathbb{I}_{\ge 2}^\top\mathbf{M}_{t}(\bZ_0+\Delta\bZ_t)\mathbb{I}_1\cc{6}_t^\top - \eta\mathbb{I}_{\ge 2}^\top\mathbf{M}_{t}(\bZ_0+\Delta\bZ_t)\mathbb{I}_{\ge 2}\cc{8}_t^\top \\
        &~~~~ -\eta \cc{4}_t \mathbb{I}_{\ge 2}^\top\left(\bY_0+\Delta\bY_t\right)^\top \mathbf{M}_{t}\mathbb{I}_{\ge 2} + \eta^2 \mathbb{I}_{\ge 2}^\top\mathbf{M}_{t}(\bZ_0+\Delta\bZ_t)\mathbb{I}_{\ge 2}\mathbb{I}_{\ge 2}^\top\left(\bY_0+\Delta\bY_t\right)^\top \mathbf{M}_{t}\mathbb{I}_{\ge 2} \\
        &~~~~ + \eta^2 \mathbb{I}_{\ge2}^\top\mathbf{M}_{t}(\bZ_0+\Delta\bZ_t)\mathbb{I}_{1}\mathbb{I}_{1}^\top\left(\bY_0+\Delta\bY_t\right)^\top \mathbf{M}_{t}\mathbb{I}_{\ge 2}\\ 
        &= 
            \mathbb{I}_{\ge2}^\top\mathbf{M}_{t}\mathbb{I}_{\ge2} 
            - \eta\mathbb{I}_{\ge 2}^\top \mathbf{M}_{t}(\mathbb{I}_{1}\mathbb{I}_{1}^\top+\mathbb{I}_{\ge 2}\mathbb{I}_{\ge 2}^\top) (\bZ_0+\Delta\bZ_t)\mathbb{I}_{\ge 2}\mathbb{I}_{\ge2}^\top \bZ_0^\top\mathbb{I}_{\ge2}
        \\
        &~~~~ 
            - \eta \mathbb{I}_{\ge2}^\top \bY_0 \mathbb{I}_{\ge2}\mathbb{I}_{\ge 2}^\top \left(\bY_0+\Delta\bY_t\right)^\top (\mathbb{I}_{1}\mathbb{I}_{1}^\top+\mathbb{I}_{\ge 2}\mathbb{I}_{\ge 2}^\top)\mathbf{M}_{t} \mathbb{I}_{\ge 2} 
         \\
        &~~~~ 
            -\eta \cc{2}_t \mathbb{I}_{1}^\top\left(\bY_0+\Delta\bY_t\right)^\top(\mathbb{I}_{1}\mathbb{I}_{1}^\top+\mathbb{I}_{\ge 2}\mathbb{I}_{\ge 2}^\top)\mathbf{M}_{t}\mathbb{I}_{\ge 2} 
            - \eta \mathbb{I}_{\ge 2}^\top\mathbf{M}_{t}(\mathbb{I}_{1}\mathbb{I}_{1}^\top+\mathbb{I}_{\ge 2}\mathbb{I}_{\ge 2}^\top)(\bZ_0+\Delta\bZ_t)\mathbb{I}_1\cc{6}_t^\top 
        \\
        &~~~~ 
            -\eta\mathbb{I}_{\ge 2}^\top\mathbf{M}_{t}(\mathbb{I}_{1}\mathbb{I}_{1}^\top+\mathbb{I}_{\ge 2}\mathbb{I}_{\ge 2}^\top)(\bZ_0+\Delta\bZ_t)\mathbb{I}_{\ge 2}\cc{8}_t^\top
            -\eta \cc{4}_t \mathbb{I}_{\ge 2}^\top\left(\bY_0+\Delta\bY_t\right)^\top (\mathbb{I}_{1}\mathbb{I}_{1}^\top+\mathbb{I}_{\ge 2}\mathbb{I}_{\ge 2}^\top)\mathbf{M}_{t}\mathbb{I}_{\ge 2} 
        \\
        &~~~~ 
            + \eta^2 \mathbb{I}_{\ge 2}^\top\mathbf{M}_{t}(\mathbb{I}_{1}\mathbb{I}_{1}^\top+\mathbb{I}_{\ge 2}\mathbb{I}_{\ge 2}^\top)(\bZ_0+\Delta\bZ_t)\mathbb{I}_{\ge 2}\mathbb{I}_{\ge 2}^\top\left(\bY_0+\Delta\bY_t\right)^\top (\mathbb{I}_{1}\mathbb{I}_{1}^\top+\mathbb{I}_{\ge 2}\mathbb{I}_{\ge 2}^\top)\mathbf{M}_{t}\mathbb{I}_{\ge 2}
        \\
        &~~~~
            + \eta^2 \mathbb{I}_{\ge2}^\top\mathbf{M}_{t}(\mathbb{I}_{1}\mathbb{I}_{1}^\top+\mathbb{I}_{\ge 2}\mathbb{I}_{\ge 2}^\top)(\bZ_0+\Delta\bZ_t)\mathbb{I}_{1}\mathbb{I}_{1}^\top\left(\bY_0+\Delta\bY_t\right)^\top (\mathbb{I}_{1}\mathbb{I}_{1}^\top+\mathbb{I}_{\ge 2}\mathbb{I}_{\ge 2}^\top)\mathbf{M}_{t}\mathbb{I}_{\ge 2}
        \\
        &= 
            \mathbb{I}_{\ge2}^\top\mathbf{M}_{t}\mathbb{I}_{\ge2} 
            - \eta\mathbb{I}_{\ge 2}^\top \mathbf{M}_{t}\mathbb{I}_{\ge 2}\mathbb{I}_{\ge 2}^\top\bZ_0\mathbb{I}_{\ge 2}\mathbb{I}_{\ge2}^\top \bZ_0^\top\mathbb{I}_{\ge2}
            - \eta\mathbb{I}_{\ge 2}^\top \mathbf{M}_{t}\mathbb{I}_{\ge 2}\cc{8}_t\mathbb{I}_{\ge2}^\top \bZ_0^\top\mathbb{I}_{\ge2}
            - \eta\mathbb{I}_{\ge 2}^\top \mathbf{M}_{t}\mathbb{I}_{1}\cc{7}_t\mathbb{I}_{\ge2}^\top \bZ_0^\top\mathbb{I}_{\ge2}
            \\
        &~~~~ 
            - \eta \mathbb{I}_{\ge2}^\top \bY_0 \mathbb{I}_{\ge2}\mathbb{I}_{\ge 2}^\top \bY_0\mathbb{I}_{\ge 2}\mathbb{I}_{\ge 2}^\top\mathbf{M}_{t} \mathbb{I}_{\ge 2}
            - \eta \mathbb{I}_{\ge2}^\top \bY_0 \mathbb{I}_{\ge2}\cc{4}_t^\top\mathbb{I}_{\ge 2}^\top\mathbf{M}_{t} \mathbb{I}_{\ge 2} 
            - \eta \mathbb{I}_{\ge2}^\top \bY_0 \mathbb{I}_{\ge2}\cc{3}_t^\top\mathbb{I}_{1}^\top\mathbf{M}_{t} \mathbb{I}_{\ge 2}
         \\
        &~~~~ 
            -\eta \cc{2}_t \mathbb{I}_{1}^\top\bY_0\mathbb{I}_{1}\mathbb{I}_{1}^\top\mathbf{M}_{t}\mathbb{I}_{\ge 2} 
            -\eta \cc{2}_t \cc{1}_t\mathbb{I}_{1}^\top\mathbf{M}_{t}\mathbb{I}_{\ge 2} 
            -\eta \cc{2}_t \cc{3}_t^\top\mathbb{I}_{1}^\top\mathbf{M}_{t}\mathbb{I}_{\ge 2} 
        \\
        &~~~~ 
            - \eta \mathbb{I}_{\ge 2}^\top\mathbf{M}_{t}\mathbb{I}_{1}\mathbb{I}_{1}^\top\bZ_0\mathbb{I}_1\cc{6}_t^\top 
            - \eta \mathbb{I}_{\ge 2}^\top\mathbf{M}_{t}\mathbb{I}_{1}\cc{5}_t\cc{6}_t^\top 
            - \eta \mathbb{I}_{\ge 2}^\top\mathbf{M}_{t}\mathbb{I}_{\ge2}\cc{6}_t\cc{6}_t^\top 
        \\
        &~~~~ 
            -\eta\mathbb{I}_{\ge 2}^\top\mathbf{M}_{t}\mathbb{I}_{\ge 2}\mathbb{I}_{\ge 2}^\top\bZ_0\mathbb{I}_{\ge 2}\cc{8}_t^\top
            -\eta\mathbb{I}_{\ge 2}^\top\mathbf{M}_{t}\mathbb{I}_{1}\cc{7}\cc{8}_t^\top
            -\eta\mathbb{I}_{\ge 2}^\top\mathbf{M}_{t}\mathbb{I}_{\ge2}\cc{8}\cc{8}_t^\top
        \\
        &~~~~ 
            -\eta \cc{4}_t \mathbb{I}_{\ge 2}^\top\bY_0\mathbb{I}_{\ge 2}\mathbb{I}_{\ge 2}^\top\mathbf{M}_{t}\mathbb{I}_{\ge 2} 
            -\eta \cc{4}_t \cc{3}_t^\top\mathbb{I}_{1}^\top\mathbf{M}_{t}\mathbb{I}_{\ge 2} 
            -\eta \cc{4}_t \cc{4}^\top\mathbb{I}_{\ge 2}^\top\mathbf{M}_{t}\mathbb{I}_{\ge 2} 
        \\
        &~~~~
            + \eta^2 \mathbb{I}_{\ge 2}^\top\mathbf{M}_{t}
            (\mathbb{I}_{\ge 2}\mathbb{I}_{\ge 2}^\top\bZ_0\mathbb{I}_{\ge 2}+\mathbb{I}_1\cc{7}_t+\mathbb{I}_{\ge2}\cc{8}_t)
            (\mathbb{I}_{\ge 2}^\top\bY_0\mathbb{I}_{\ge 2}\mathbb{I}_{\ge 2}^\top + \cc{3}_t^\top \mathbb{I}_1^\top + \cc{4}_t^\top \mathbb{I}_{\ge2}^\top)            
            \mathbf{M}_{t}\mathbb{I}_{\ge 2}
        \\
        &~~~~
            + \eta^2 \mathbb{I}_{\ge2}^\top\mathbf{M}_{t}
            (\mathbb{I}_{1}\mathbb{I}_{1}^\top\bZ_0\mathbb{I}_1^\top+\mathbb{I}_{1}\cc{5}_t+\mathbb{I}_{\ge2}\cc{6}_t)
            (\mathbb{I}_{1}^\top\bY_0\mathbb{I}_{1}\mathbb{I}_1^\top+\cc{1}_t\mathbb{I}_{1}^\top+\cc{2}_t^\top\mathbb{I}_{\ge2}^\top)\mathbf{M}_{t}\mathbb{I}_{\ge 2}
        \\
    \end{align*}

    In the following equations, {\color{red}red} terms are expected to be $O(1)$ while {\color{blue}blue} terms are expected to be $O(\epsilon)$.

    \begin{align}
        \cc{1}_{t+1} &= \cc{1}_t-\eta\mathbb{I}_1^\top\mathbf{M}_t(\bZ_0+\Delta\bZ_t)\mathbb{I}_1 \\
        &= \cc{1}_t-\eta\mathbb{I}_1^\top\mathbf{M}_t\mathbb{I}_1\mathbb{I}_1^\top(\bZ_0+\Delta\bZ_t)\mathbb{I}_1-\eta\mathbb{I}_1^\top\mathbf{M}_t\mathbb{I}_{\ge2}\mathbb{I}_{\ge2}^\top(\bZ_0+\Delta\bZ_t)\mathbb{I}_1 \\
        &= \cc{1}_t-\eta\crr{\mathbb{I}_1^\top\mathbf{M}_t\mathbb{I}_1}\mathbb{I}_1^\top \bZ_0\mathbb{I}_1 - \eta\crr{\mathbb{I}_1^\top\mathbf{M}_t\mathbb{I}_1}\crr{\cc{5}_t}- \eta\cbb{\mathbb{I}_1^\top\mathbf{M}_t\mathbb{I}_{\ge2}}\cc{6}_t,
        \\
        \cc{2}_{t+1} &= \cc{2}_t-\eta\mathbb{I}_{\ge 2}^\top\mathbf{M}_t\mathbb{I}_1\mathbb{I}_1^\top(\bZ_0+\Delta\bZ_t)\mathbb{I}_1 -\eta\mathbb{I}_{\ge 2}^\top\mathbf{M}_t\mathbb{I}_{\ge2}\mathbb{I}_{\ge2}^\top(\bZ_0+\Delta\bZ_t)\mathbb{I}_1 \\
        &= \cc{2}_t-\eta\cbb{\mathbb{I}_{\ge 2}^\top\mathbf{M}_t\mathbb{I}_1}\mathbb{I}_1^\top\bZ_0\mathbb{I}_1 -\eta\cbb{\mathbb{I}_{\ge 2}^\top\mathbf{M}_t\mathbb{I}_1}\crr{\cc{5}_t} -\eta\crr{\mathbb{I}_{\ge 2}^\top\mathbf{M}_t\mathbb{I}_{\ge2}}\cc{6}_t,
        \\
        \cc{3}_{t+1} &= \cc{3}_t-\eta\mathbb{I}_{1}^\top\mathbf{M}_t\mathbb{I}_1\mathbb{I}_1^\top(\bZ_0+\Delta\bZ_t)\mathbb{I}_{\ge 2}-\eta\mathbb{I}_{1}^\top\mathbf{M}_t\mathbb{I}_{\ge2}\mathbb{I}_{\ge2}^\top(\bZ_0+\Delta\bZ_t)\mathbb{I}_{\ge 2} \\
        &= \cc{3}_t-\eta\crr{\mathbb{I}_{1}^\top\mathbf{M}_t\mathbb{I}_1}\cc{7}_t-\eta\cbb{\mathbb{I}_{1}^\top\mathbf{M}_t\mathbb{I}_{\ge2}}\mathbb{I}_{\ge2}^\top\bZ_0\mathbb{I}_{\ge2}-\eta\cbb{\mathbb{I}_{1}^\top\mathbf{M}_t\mathbb{I}_{\ge2}}\cc{8}_t,
        \\
        \cc{4}_{t+1} &= \cc{4}_t-\eta\mathbb{I}_{\ge 2}^\top\mathbf{M}_t\mathbb{I}_1\mathbb{I}_1^\top(\bZ_0+\Delta\bZ_t)\mathbb{I}_{\ge 2}-\eta\mathbb{I}_{\ge 2}^\top\mathbf{M}_t\mathbb{I}_{\ge2}\mathbb{I}_{\ge2}^\top(\bZ_0+\Delta\bZ_t)\mathbb{I}_{\ge 2} \\
        &= \cc{4}_t-\eta\cbb{\mathbb{I}_{\ge 2}^\top\mathbf{M}_t\mathbb{I}_1}\cc{7}_t-\eta\cbb{\mathbb{I}_{\ge 2}^\top\mathbf{M}_t\mathbb{I}_{\ge2}}\mathbb{I}_{\ge2}^\top\bZ_0\mathbb{I}_{\ge 2} - \eta\cbb{\mathbb{I}_{\ge 2}^\top\mathbf{M}_t\mathbb{I}_{\ge2}}\cc{8}_t,
        \\
        \cc{5}_{t+1} &= \cc{5}_t-\eta\mathbb{I}_1^\top\mathbf{M}_t^\top\mathbb{I}_1\mathbb{I}_1^\top\left(\bY_0+\Delta\bY_t\right)\mathbb{I}_1-\eta\mathbb{I}_1^\top\mathbf{M}_t^\top\mathbb{I}_{\ge2}\mathbb{I}_{\ge2}^\top\left(\bY_0+\Delta\bY_t\right)\mathbb{I}_1 \\
        &= \cc{5}_t-\eta\crr{\mathbb{I}_1^\top\mathbf{M}_t^\top\mathbb{I}_1}\mathbb{I}_1^\top\bY_0\mathbb{I}_1-\eta\crr{\mathbb{I}_1^\top\mathbf{M}_t^\top\mathbb{I}_1}\crr{\cc{1}_t}-\eta\cbb{\mathbb{I}_1^\top\mathbf{M}_t^\top\mathbb{I}_{\ge2}}\cc{2}_t,
        \\
        \cc{6}_{t+1} &= \cc{6}_t-\eta\mathbb{I}_{\ge 2}^\top\mathbf{M}_t^\top\mathbb{I}_{1}\mathbb{I}_{1}^\top\left(\bY_0+\Delta\bY_t\right)\mathbb{I}_1-\eta\mathbb{I}_{\ge 2}^\top\mathbf{M}_t^\top\mathbb{I}_{\ge2}\mathbb{I}_{\ge2}^\top\left(\bY_0+\Delta\bY_t\right)\mathbb{I}_1 \\
        &= \cc{6}_t-\eta\cbb{\mathbb{I}_{\ge 2}^\top\mathbf{M}_t^\top\mathbb{I}_{1}}\mathbb{I}_{1}^\top\bY_0\mathbb{I}_1 - \eta\cbb{\mathbb{I}_{\ge 2}^\top\mathbf{M}_t^\top\mathbb{I}_{1}}\crr{\cc{1}_t} -\eta\cbb{\mathbb{I}_{\ge 2}^\top\mathbf{M}_t^\top\mathbb{I}_{\ge2}}\cc{2}_t,
        \\
        \cc{7}_{t+1} &= \cc{7}_t-\eta\mathbb{I}_{1}^\top\mathbf{M}_t^\top\mathbb{I}_{1}\mathbb{I}_{1}^\top\left(\bY_0+\Delta\bY_t\right)\mathbb{I}_{\ge 2} - \eta\mathbb{I}_{1}^\top\mathbf{M}_t^\top\mathbb{I}_{\ge2}\mathbb{I}_{\ge2}^\top\left(\bY_0+\Delta\bY_t\right)\mathbb{I}_{\ge 2} \\
        &= \cc{7}_t-\eta\crr{\mathbb{I}_{1}^\top\mathbf{M}_t^\top\mathbb{I}_{1}}\cc{3}_t - \eta\cbb{\mathbb{I}_{1}^\top\mathbf{M}_t^\top\mathbb{I}_{\ge2}}\mathbb{I}_{\ge2}^\top\bY_0\mathbb{I}_{\ge 2} - \eta\cbb{\mathbb{I}_{1}^\top\mathbf{M}_t^\top\mathbb{I}_{\ge2}}\cc{4}_t,
        \\
        \cc{8}_{t+1} &= \cc{8}_t - \eta\mathbb{I}_{\ge 2}^\top\mathbf{M}_t^\top\mathbb{I}_{1}\mathbb{I}_{1}^\top\left(\bY_0+\Delta\bY_t\right)\mathbb{I}_{\ge 2} - \eta\mathbb{I}_{\ge 2}^\top\mathbf{M}_t^\top\mathbb{I}_{\ge2}\mathbb{I}_{\ge2}^\top\left(\bY_0+\Delta\bY_t\right)\mathbb{I}_{\ge 2} \\
        &= \cc{8}_t - \eta\cbb{\mathbb{I}_{\ge 2}^\top\mathbf{M}_t^\top\mathbb{I}_{1}}\cc{3}_t - \eta\cbb{\mathbb{I}_{\ge 2}^\top\mathbf{M}_t^\top\mathbb{I}_{\ge2}}\mathbb{I}_{\ge2}^\top\bY_0\mathbb{I}_{\ge 2} - \eta\cbb{\mathbb{I}_{\ge 2}^\top\mathbf{M}_t^\top\mathbb{I}_{\ge2}}\cc{4}_t,
    \end{align}

    By expanding the definition of $\mathbb{I}_1\mathbf{M}_t\mathbb{I}_1$, the update rules of $\cc{1}_t$ and $\cc{5}_t$ are
    \begin{align}
        \cc{1}_{t+1} &= \cc{1}_t - \eta(\cc{1}_t\frac{\sigma_1}{\alpha}+\sigma_1\alpha\cc{5}_t+\cc{1}_t\cc{5}_t+\cc{3}_t\cc{7}_t^\top)(\frac{\sigma_1}{\alpha}+\cc{5}_t) - \eta\cbb{\mathbb{I}_1^\top\mathbf{M}_t\mathbb{I}_{\ge2}}\cc{6}_t \\
        &= \cc{1}_t - \eta(\cc{1}_t\frac{\sigma_1}{\alpha}+\sigma_1\alpha\cc{5}_t+\cc{1}_t\cc{5}_t)(\frac{\sigma_1}{\alpha}+\cc{5}_t) - \eta\cbb{\mathbb{I}_1^\top\mathbf{M}_t\mathbb{I}_{\ge2}}\cc{6}_t - \eta\cc{3}_t\cc{7}_t^\top(\frac{\sigma_1}{\alpha}+\cc{5}_t), \\
        \cc{5}_{t+1} &= \cc{5}_t - \eta(\cc{1}_t\frac{\sigma_1}{\alpha}+\sigma_1\alpha\cc{5}_t+\cc{1}_t\cc{5}_t+\cc{3}_t\cc{7}_t^\top)(\sigma_1\alpha+\cc{1}_t) - \eta\cbb{\mathbb{I}_1^\top\mathbf{M}_t^\top\mathbb{I}_{\ge2}}\cc{2}_t \\
        &= \cc{5}_t - \eta(\cc{1}_t\frac{\sigma_1}{\alpha}+\sigma_1\alpha\cc{5}_t+\cc{1}_t\cc{5}_t)(\sigma_1\alpha+\cc{1}_t) - \eta\cbb{\mathbb{I}_1^\top\mathbf{M}_t^\top\mathbb{I}_{\ge2}}\cc{2}_t - \eta\cc{3}_t\cc{7}_t^\top(\sigma_1\alpha+\cc{1}_t) 
    \end{align}
    At initialization $t=1$, all of $\mathbb{I}_1^\top\mathbf{M}_t\mathbb{I}_{\ge2}, \cc{2}_t, \cc{3}_t, \cc{6}_t, \cc{7}_t$ are in $\O(\epsilon)$

    Since we have assumed 
    \begin{align*}
        \mathbb{I}_{\ge2}^\top\mathbf{M}_{t+1} \mathbb{I}_{\ge2}
        &\approx
            \mathbb{I}_{\ge2}^\top\mathbf{M}_{t}\mathbb{I}_{\ge2} 
            - \eta\mathbb{I}_{\ge 2}^\top \mathbf{M}_{t}\mathbb{I}_{\ge 2}\mathbb{I}_{\ge 2}^\top\bZ_0\mathbb{I}_{\ge 2}\mathbb{I}_{\ge2}^\top \bZ_0^\top\mathbb{I}_{\ge2}
            \\
        &~~~~ 
            - \eta \mathbb{I}_{\ge2}^\top \bY_0 \mathbb{I}_{\ge2}\mathbb{I}_{\ge 2}^\top \bY_0\mathbb{I}_{\ge 2}\mathbb{I}_{\ge 2}^\top\mathbf{M}_{t} \mathbb{I}_{\ge 2}
        \\
        &~~~~ 
        +\O(\epsilon\cdot\epsilon_t)
    \end{align*}

    \begin{align*}
        \mathbb{I}_{1}^\top\mathbf{M}_{t+1} \mathbb{I}_{\ge2}
        &\approx
            \mathbb{I}_{1}^\top\mathbf{M}_{t}\mathbb{I}_{\ge2}
            - \eta\mathbb{I}_{1}^\top \mathbf{M}_{t} \mathbb{I}_{\ge 2}\mathbb{I}_{\ge 2}^\top \bZ_0 \mathbb{I}_{\ge 2}\mathbb{I}_{\ge2}^\top \bZ_0^\top\mathbb{I}_{\ge2} 
            - \eta\mathbb{I}_{1}^\top \mathbf{M}_{t} \mathbb{I}_{1}\cc{7}_t\mathbb{I}_{\ge2}^\top \bZ_0^\top\mathbb{I}_{\ge2} 
        \\
        &~~~~
            - \eta \mathbb{I}_1^\top \bY_0 \mathbb{I}_1\mathbb{I}_1^\top \bY_0\mathbb{I}_{1}\mathbb{I}_{1}^\top\mathbf{M}_{t} \mathbb{I}_{\ge 2}  
            - \eta \mathbb{I}_1^\top \bY_0 \mathbb{I}_1\cc{1}_t\mathbb{I}_{1}^\top\mathbf{M}_{t} \mathbb{I}_{\ge 2}
        \\
        &~~~~ 
            - \eta \cc{1}_t \mathbb{I}_1\mathbb{I}_1^\top \bY_0\mathbb{I}_{1}\mathbb{I}_{1}^\top\mathbf{M}_{t} \mathbb{I}_{\ge 2}  
            - \eta \cc{1}_t\cc{1}_t\mathbb{I}_{1}^\top\mathbf{M}_{t}\mathbb{I}_{\ge 2}
        \\
        &~~~~ 
            - \eta \mathbb{I}_{1}^\top\mathbf{M}_{t}\mathbb{I}_{1}\mathbb{I}_{1}^\top\bZ_0\mathbb{I}_1\cc{6}_t^\top 
            - \eta \mathbb{I}_{1}^\top\mathbf{M}_{t}\mathbb{I}_{1}\cc{5}_t\cc{6}_t^\top 
        \\
        &~~~~
            - \eta\mathbb{I}_{1}^\top \mathbf{M}_{t} \mathbb{I}_{1}\cc{7}_t\cc{8}_t^\top 
        \\
        &~~~~
            + \eta^2 \mathbb{I}_{1}^\top\mathbf{M}_{t}\mathbb{I}_{1}
            (\mathbb{I}_{1}^\top\bZ_0\mathbb{I}_1^\top+\cc{5}_t)
            (\mathbb{I}_{1}^\top\bY_0\mathbb{I}_{1}+\cc{1}_t)\mathbb{I}_1^\top\mathbf{M}_{t}\mathbb{I}_{\ge 2}
        \\
        &~~~~
            +\O(\epsilon\cdot\epsilon_t)
        \\
        &=
            \mathbb{I}_{1}^\top\mathbf{M}_{t}\mathbb{I}_{\ge2}
            - \eta \mathbb{I}_{1}^\top\mathbf{M}_{t} \mathbb{I}_{\ge 2}  \left(\mathbb{I}_1^\top \bY_0 \mathbb{I}_1 + \cc{1}_t\right)^2
            - \eta\mathbb{I}_{1}^\top \mathbf{M}_{t} \mathbb{I}_{\ge 2}\mathbb{I}_{\ge 2}^\top \bZ_0 \mathbb{I}_{\ge 2}\mathbb{I}_{\ge2}^\top \bZ_0^\top\mathbb{I}_{\ge2} 
        \\
        &~~~~ 
            - \eta \mathbb{I}_{1}^\top\mathbf{M}_{t}\mathbb{I}_{1}\underbrace{\left(\mathbb{I}_{1}^\top\bZ_0\mathbb{I}_1\cc{6}_t^\top + \cc{5}_t\cc{6}_t^\top + \cc{7}_t\cc{8}_t^\top   +\cc{7}_t\mathbb{I}_{\ge2}^\top \bZ_0^\top\mathbb{I}_{\ge2} \right)}_{=\mathbb{I}_{1}^\top\left(\mathbf{Z}_0+\Delta\mathbf{Z}_{t}\right)\left(\mathbf{Z}_0+\Delta\mathbf{Z}_{t}\right)^\top\mathbb{I}_{\ge2}}
        \\
        &~~~~
            + \eta^2 \mathbb{I}_{1}^\top\mathbf{M}_{t}\mathbb{I}_{1}
            (\mathbb{I}_{1}^\top\bZ_0\mathbb{I}_1^\top+\cc{5}_t)
            (\mathbb{I}_{1}^\top\bY_0\mathbb{I}_{1}+\cc{1}_t)\mathbb{I}_1^\top\mathbf{M}_{t}\mathbb{I}_{\ge 2}
        \\
        &~~~~
        +\O(\epsilon\cdot\epsilon_t)
    \end{align*}

    \begin{align*}
        \mathbb{I}_{\ge2}^\top\mathbf{M}_{t+1} \mathbb{I}_{1}
        &\approx
            \mathbb{I}_{\ge2}^\top\mathbf{M}_{t}\mathbb{I}_{1}
            - \eta\mathbb{I}_{\ge2}^\top \mathbf{M}_{t}\mathbb{I}_{1}\mathbb{I}_{1}^\top\bZ_0\mathbb{I}_{1}\mathbb{I}_{1}^\top \bZ_0^\top\mathbb{I}_{1} 
            - \eta\mathbb{I}_{\ge2}^\top \mathbf{M}_{t}\mathbb{I}_{1}\cc{5}_t\mathbb{I}_{1}^\top \bZ_0^\top\mathbb{I}_{1} 
        \\
        &~~~~ 
            - \eta \mathbb{I}_{\ge2}^\top \bY_0 \mathbb{I}_{\ge2}\mathbb{I}_{\ge2}^\top \bY_0\mathbb{I}_{\ge 2}\mathbb{I}_{\ge 2}^\top\mathbf{M}_{t} \mathbb{I}_{1}
            - \eta \mathbb{I}_{\ge2}^\top \bY_0 \mathbb{I}_{\ge2}\cc{3}_t^\top\mathbb{I}_{1}^\top\mathbf{M}_{t} \mathbb{I}_{1}
        \\
        &~~~~ 
            -\eta \cc{2}_t \mathbb{I}_{1}^\top\bY_0\mathbb{I}_{1}\mathbb{I}_{1}^\top\mathbf{M}_{t}\mathbb{I}_{1} 
            -\eta \cc{2}_t \cc{1}_t\mathbb{I}_{1}^\top\mathbf{M}_{t}\mathbb{I}_{1} 
        \\
        &~~~~ 
            - \eta\mathbb{I}_{\ge2}^\top \mathbf{M}_{t}\mathbb{I}_{1}\mathbb{I}_{1}^\top\bZ_0\mathbb{I}_{1}\cc{5}_t 
            - \eta\mathbb{I}_{\ge2}^\top \mathbf{M}_{t}\mathbb{I}_{1}\cc{5}_t\cc{5}_t 
        \\
        &~~~~
            - \eta \cc{4}_t\cc{3}_t^\top\mathbb{I}_{1}^\top\mathbf{M}_{t} \mathbb{I}_{1}
        \\
        &~~~~
            + \eta^2 \mathbb{I}_{\ge2}^\top\mathbf{M}_{t}\mathbb{I}_{1}
            (\mathbb{I}_{1}^\top\bZ_0\mathbb{I}_1^\top+\cc{5}_t)
            (\mathbb{I}_{1}^\top\bY_0\mathbb{I}_{1}+\cc{1}_t)\mathbb{I}_{1}^\top\mathbf{M}_{t}\mathbb{I}_{1}
        \\
        &=
            \mathbb{I}_{\ge2}^\top\mathbf{M}_{t}\mathbb{I}_{1}
            - \eta\mathbb{I}_{\ge2}^\top \mathbf{M}_{t}\mathbb{I}_{1}\left(\mathbb{I}_{1}^\top\bZ_0\mathbb{I}_{1}+\cc{5}_t\right)^2
            - \eta \mathbb{I}_{\ge2}^\top \bY_0 \mathbb{I}_{\ge2}\mathbb{I}_{\ge2}^\top \bY_0\mathbb{I}_{\ge 2}\mathbb{I}_{\ge 2}^\top\mathbf{M}_{t} \mathbb{I}_{1}
            \\
            &~~~~
            - \eta \mathbb{I}_{1}^\top\mathbf{M}_{t} \mathbb{I}_{1}\underbrace{\left(\mathbb{I}_{\ge2}^\top \bY_0 \mathbb{I}_{\ge2}\cc{3}_t^\top+\cc{2}_t \mathbb{I}_{1}^\top\bY_0\mathbb{I}_{1}+\cc{2}_t \cc{1}_t+ \cc{4}_t\cc{3}_t^\top\right)}_{=\mathbb{I}_{\ge2}^\top\left(\bY_0+\Delta\bY_t\right)\left(\bY_0+\Delta\bY_t\right)^\top\mathbb{I}_{1}}
            \\
            &~~~~
            + \eta^2 \mathbb{I}_{\ge2}^\top\mathbf{M}_{t}\mathbb{I}_{1}
            (\mathbb{I}_{1}^\top\bZ_0\mathbb{I}_1^\top+\cc{5}_t)
            (\mathbb{I}_{1}^\top\bY_0\mathbb{I}_{1}+\cc{1}_t)\mathbb{I}_{1}^\top\mathbf{M}_{t}\mathbb{I}_{1}
    \end{align*}

    \begin{align*}
        \mathbb{I}_{1}^\top\mathbf{M}_{t+1} \mathbb{I}_{1}
        &\approx
            \mathbb{I}_{1}^\top\mathbf{M}_{t} \mathbb{I}_{1}
            - \eta\mathbb{I}_{1}^\top \mathbf{M}_{t}\mathbb{I}_{1}\mathbb{I}_{1}^\top\bZ_0\mathbb{I}_{1}\mathbb{I}_{1}^\top \bZ_0^\top\mathbb{I}_{1}
            - \eta\mathbb{I}_{1}^\top \mathbf{M}_{t}\mathbb{I}_{1}\cc{5}_t\mathbb{I}_{1}^\top \bZ_0^\top\mathbb{I}_{1}
        \\
        &~~~~ 
            - \eta \mathbb{I}_{1}^\top \bY_0 \mathbb{I}_{1}\mathbb{I}_{1}^\top \bY_0\mathbb{I}_{1}\mathbb{I}_{1}^\top\mathbf{M}_{t} \mathbb{I}_{1} 
            - \eta \mathbb{I}_{1}^\top \bY_0 \mathbb{I}_{1}\cc{1}_t\mathbb{I}_{1}^\top\mathbf{M}_{t} \mathbb{I}_{1}
        \\
        &~~~~ 
            - \eta \cc{1}_t\mathbb{I}_{1}^\top \bY_0\mathbb{I}_{1}\mathbb{I}_{1}^\top\mathbf{M}_{t} \mathbb{I}_{1} 
            - \eta \cc{1}_t\cc{1}_t\mathbb{I}_{1}^\top\mathbf{M}_{t} \mathbb{I}_{1}
        \\
        &~~~~ 
            - \eta\mathbb{I}_{1}^\top \mathbf{M}_{t}\mathbb{I}_{1}\mathbb{I}_{1}^\top\bZ_0\mathbb{I}_{1}\cc{5}_t 
            - \eta\mathbb{I}_{1}^\top \mathbf{M}_{t}\mathbb{I}_{1}\cc{5}_t\cc{5}_t 
        \\
        &~~~~ 
            + \eta^2 \mathbb{I}_{1}^\top\mathbf{M}_{t}\mathbb{I}_{1}
            (\mathbb{I}_{1}^\top\bZ_0\mathbb{I}_{1}+\cc{5}_t)
            (\mathbb{I}_{1}^\top\bY_0\mathbb{I}_{1}+\cc{1}_t)
            \mathbb{I}_{1}^\top\mathbf{M}_{t}\mathbb{I}_{1}
    \end{align*}

    \begin{align*}
        \mathbb{I}_{\ge2}^\top\left(\bY_0+\Delta\bY_{t+1}\right)\left(\bY_0+\Delta\bY_{t+1}\right)^\top\mathbb{I}_{1}
        &=
            \mathbb{I}_{\ge2}^\top\left(\bY_0+\Delta\bY_t\right)\left(\bY_0+\Delta\bY_t\right)^\top\mathbb{I}_{1}
        \\
        &~~~~
            -\eta\mathbb{I}_{\ge2}^\top\mathbf{M}_t(\bZ_0+\Delta\bZ_t)\mathbb{I}_{\ge2}\cc{3}_t^\top
            -\eta(\mathbb{I}_{\ge2}^\top\bY_0\mathbb{I}_{\ge2}+\cc{4}_t)\mathbb{I}_{\ge2}^\top(\bZ_0+\Delta\bZ_t)^\top\mathbf{M}_t^\top\mathbb{I}_{1}
        \\
        &~~~~
            -\eta\mathbb{I}_{\ge2}^\top\mathbf{M}_t(\bZ_0+\Delta\bZ_t)\mathbb{I}_{1}(\cc{1}_t+\sigma_1\alpha) 
            -\eta\cc{2}\mathbb{I}_{1}^\top(\bZ_0+\Delta\bZ_t)^\top\mathbf{M}_t^\top\mathbb{I}_{1}
        \\
        &~~~~
            +\eta^2\mathbb{I}_{\ge2}^\top\mathbf{M}_t(\bZ_0+\Delta\bZ_t)\mathbb{I}_{\ge2}\mathbb{I}_{\ge2}^\top(\bZ_0+\Delta\bZ_t)^\top\mathbf{M}_t^\top\mathbb{I}_{1}
        \\
        &~~~~
            +\eta^2\mathbb{I}_{\ge2}^\top\mathbf{M}_t(\bZ_0+\Delta\bZ_t)\mathbb{I}_{1}\mathbb{I}_{1}^\top(\bZ_0+\Delta\bZ_t)^\top\mathbf{M}_t^\top\mathbb{I}_{1}
        \\
        &=
            \mathbb{I}_{\ge2}^\top\left(\bY_0+\Delta\bY_t\right)\left(\bY_0+\Delta\bY_t\right)^\top\mathbb{I}_{1}
        \\
        &~~~~
            -\eta\mathbb{I}_{\ge2}^\top\mathbf{M}_t\mathbb{I}_{\ge2}\mathbb{I}_{\ge2}^\top\bZ_0\mathbb{I}_{\ge2}\cc{3}_t^\top
            -\eta\mathbb{I}_{\ge2}^\top\mathbf{M}_t\mathbb{I}_{1}^\top\cc{7}_t\cc{3}_t^\top
            -\eta\mathbb{I}_{\ge2}^\top\mathbf{M}_t\mathbb{I}_{\ge2}^\top\cc{8}_t\cc{3}_t^\top
        \\
        &~~~~
            -\eta(\mathbb{I}_{\ge2}^\top\bY_0\mathbb{I}_{\ge2}+\cc{4}_t)\mathbb{I}_{\ge2}^\top\bZ_0\mathbb{I}_{\ge2}\mathbb{I}_{\ge2}^\top\mathbf{M}_t^\top\mathbb{I}_{1}
            -\eta(\mathbb{I}_{\ge2}^\top\bY_0\mathbb{I}_{\ge2}+\cc{4}_t)\cc{7}_t^\top \mathbb{I}_{1}^\top \mathbf{M}_t^\top\mathbb{I}_{1}
        \\
        &~~~~
            -\eta(\mathbb{I}_{\ge2}^\top\bY_0\mathbb{I}_{\ge2}+\cc{4}_t)\cc{8}_t^\top \mathbb{I}_{\ge2}^\top \mathbf{M}_t^\top\mathbb{I}_{1}
        \\
        &~~~~
            -\eta\mathbb{I}_{\ge2}^\top\mathbf{M}_t\mathbb{I}_{1}\mathbb{I}_{1}^\top\bZ_0\mathbb{I}_{1}(\cc{1}_t+\sigma_1\alpha) 
            -\eta\mathbb{I}_{\ge2}^\top\mathbf{M}_t\mathbb{I}_{1}\cc{5}_t(\cc{1}_t+\sigma_1\alpha) 
            \\&~~~~
            -\eta\mathbb{I}_{\ge2}^\top\mathbf{M}_t\mathbb{I}_{\ge2}\cc{6}_t(\cc{1}_t+\sigma_1\alpha) 
        \\
        &~~~~
            -\eta\cc{2}\mathbb{I}_{1}^\top\bZ_0\mathbb{I}_{1}\mathbb{I}_{1}^\top\mathbf{M}_t^\top\mathbb{I}_{1}
            -\eta\cc{2}\cc{5}_t\mathbb{I}_{1}^\top\mathbf{M}_t^\top\mathbb{I}_{1}
            -\eta\cc{2}\cc{6}_t^\top\mathbb{I}_{\ge2}^\top\mathbf{M}_t^\top\mathbb{I}_{1}
        \\
        &~~~~
            +\eta^2\mathbb{I}_{\ge2}^\top\mathbf{M}_t(\bZ_0+\Delta\bZ_t)\mathbb{I}_{\ge2}\mathbb{I}_{\ge2}^\top(\bZ_0+\Delta\bZ_t)^\top\mathbf{M}_t^\top\mathbb{I}_{1}
        \\
        &~~~~
            +\eta^2\mathbb{I}_{\ge2}^\top\mathbf{M}_t(\bZ_0+\Delta\bZ_t)\mathbb{I}_{1}\mathbb{I}_{1}^\top(\bZ_0+\Delta\bZ_t)^\top\mathbf{M}_t^\top\mathbb{I}_{1}
        \\
        &\approx 
        \mathbb{I}_{\ge2}^\top\left(\bY_0+\Delta\bY_t\right)\left(\bY_0+\Delta\bY_t\right)^\top\mathbb{I}_{1}
        \\
        &~~~~
            -\eta\mathbb{I}_{\ge2}^\top\bY_0\mathbb{I}_{\ge2}\mathbb{I}_{\ge2}^\top\bZ_0\mathbb{I}_{\ge2}\mathbb{I}_{\ge2}^\top\mathbf{M}_t^\top\mathbb{I}_{1}
        \\
        &~~~~
            -\eta(\cc{1}_t+\sigma_1\alpha)(\cc{5}_t+\frac{\sigma_1}{\alpha})\mathbb{I}_{\ge2}^\top\mathbf{M}_t\mathbb{I}_{1}
        \\
        &~~~~
            -\eta\mathbb{I}_{1}^\top\mathbf{M}_t\mathbb{I}_{1}\underbrace{\left( \cc{2}\mathbb{I}_{1}^\top\bZ_0\mathbb{I}_{1} + \cc{2}\cc{5}_t\mathbb{I}_{1}^\top + (\mathbb{I}_{\ge2}^\top\bY_0\mathbb{I}_{\ge2}+\cc{4}_t)\cc{7}_t^\top \right)}_{=\mathbb{I}_{\ge2}^\top\mathbf{M}_t\mathbb{I}_{1}}
        \\
        &~~~~
            +\eta^2\mathbb{I}_{\ge2}^\top\mathbf{M}_t\mathbb{I}_{1}
            ( \mathbb{I}_{1}^\top\bZ_0\mathbb{I}_{1} + \cc{5}_t)
            ( \mathbb{I}_{1}^\top\bZ_0\mathbb{I}_{1} + \cc{5}_t)
            \mathbb{I}_{1}^\top\mathbf{M}_t^\top\mathbb{I}_{1}
        \\
        &=
            \mathbb{I}_{\ge2}^\top\left(\bY_0+\Delta\bY_t\right)\left(\bY_0+\Delta\bY_t\right)^\top\mathbb{I}_{1}
            \\
            &~~~~
                -\eta\left((\cc{1}_t+\sigma_1\alpha)(\cc{5}_t+\frac{\sigma_1}{\alpha})+\mathbb{I}_{1}^\top\mathbf{M}_t\mathbb{I}_{1}(1-\eta(\cc{5}_t+\frac{\sigma_1}{\alpha})^2)\right)\mathbb{I}_{\ge2}^\top\mathbf{M}_t\mathbb{I}_{1}
            \\
            &~~~~
                -\eta\mathbb{I}_{\ge2}^\top\bY_0\mathbb{I}_{\ge2}\mathbb{I}_{\ge2}^\top\bZ_0\mathbb{I}_{\ge2}\mathbb{I}_{\ge2}^\top\mathbf{M}_t^\top\mathbb{I}_{1}
    \end{align*}

    \begin{align*}
        \mathbb{I}_{1}^\top\left(\bZ_0+\Delta\bZ_{t+1}\right)\left(\bZ_0+\Delta\bZ_{t+1}\right)^\top\mathbb{I}_{\ge2}
        &\approx
            \mathbb{I}_{1}^\top\left(\bZ_0+\Delta\bZ_t\right)\left(\bZ_0+\Delta\bZ_t\right)^\top\mathbb{I}_{\ge2}
            \\
            &~~~~
                -\eta\left((\cc{1}_t+\sigma_1\alpha)(\cc{5}_t+\frac{\sigma_1}{\alpha})+\mathbb{I}_{1}^\top\mathbf{M}_t\mathbb{I}_{1}(1-\eta(\cc{1}_t+\sigma_1\alpha)^2)\right)\mathbb{I}_{1}^\top\mathbf{M}_t\mathbb{I}_{\ge2}
            \\
            &~~~~
                -\eta\mathbb{I}_{\ge2}^\top\bY_0\mathbb{I}_{\ge2}\mathbb{I}_{\ge2}^\top\bZ_0\mathbb{I}_{\ge2}\mathbb{I}_{1}^\top\mathbf{M}_t^\top\mathbb{I}_{\ge2}
    \end{align*}

    Therefore, we have built a $4\times4$ matrix to characterize the dynamics of $\mathbb{I}_{\ge2}^\top\mathbf{M}_t\mathbb{I}_1$, $\mathbb{I}_{1}^\top\mathbf{M}_t\mathbb{I}_{\ge2}$, $\mathbb{I}_{\ge2}^\top\left(\bY_0+\Delta\bY_{t+1}\right)\left(\bY_0+\Delta\bY_{t+1}\right)^\top\mathbb{I}_{1}$, $\mathbb{I}_{1}^\top\left(\bZ_0+\Delta\bZ_{t+1}\right)\left(\bZ_0+\Delta\bZ_{t+1}\right)^\top\mathbb{I}_{\ge2}$ as, for $\forall p\in\{2,3,\dots,d\}$
    \begin{gather*}
        \begin{bmatrix}
            [\mathbb{I}_{\ge2}^\top\left(\bY_0+\Delta\bY_{t+1}\right)\left(\bY_0+\Delta\bY_{t+1}\right)^\top\mathbb{I}_{1}]_p\\
            [\mathbb{I}_{1}^\top\left(\bZ_0+\Delta\bZ_{t+1}\right)\left(\bZ_0+\Delta\bZ_{t+1}\right)^\top\mathbb{I}_{\ge2}]_p\\
            [\mathbb{I}_{\ge2}^\top\mathbf{M}_{t+1}\mathbb{I}_1]_p \\
            [\mathbb{I}_{1}^\top\mathbf{M}_{t+1}\mathbb{I}_{\ge2}]_p
        \end{bmatrix}
        \leftarrow 
        \mathbf{Q}_{\alpha,\eta,p}(y_t, z_t)
        \begin{bmatrix}
            [\mathbb{I}_{\ge2}^\top\left(\bY_0+\Delta\bY_{t}\right)\left(\bY_0+\Delta\bY_{t}\right)^\top\mathbb{I}_{1}]_p\\
            [\mathbb{I}_{1}^\top\left(\bZ_0+\Delta\bZ_{t}\right)\left(\bZ_0+\Delta\bZ_{t}\right)^\top\mathbb{I}_{\ge2}]_p\\
            [\mathbb{I}_{\ge2}^\top\mathbf{M}_{t}\mathbb{I}_1]_p \\
            [\mathbb{I}_{1}^\top\mathbf{M}_{t}\mathbb{I}_{\ge2}]_p
        \end{bmatrix},\\
        \mathbf{Q}_{\alpha,\eta,p}(y_t, z_t)
        \triangleq
        \begin{bmatrix}
            1 & 0 & u_{1,t} & -\eta\sigma_p^2 \\
            0 & 1 & -\eta\sigma_p^2 & u_{2,t} \\
            -\eta(y_t z_t-\sigma_1^2) & 0 & w_{1,t} & 0 \\
            0 & -\eta(y_t z_t-\sigma_1^2) & 0 & w_{2,t}
        \end{bmatrix},
        \\
        y_t \triangleq \cc{1}_t+\sigma_1\alpha, ~~~~ z_t\triangleq \cc{5}_t+\nicefrac{\sigma_1}{\alpha_p},\\
        u_{1,t}\triangleq -\eta\left( y_t z_t + (y_t z_t-\sigma_1^2)(1-\eta z_t^2) \right),\\
        u_{2,t}\triangleq -\eta\left( y_t z_t + (y_t z_t-\sigma_1^2)(1-\eta y_t^2) \right),\\
        w_{1,t}\triangleq 1-\eta z_t^2-\eta\sigma_p^2\alpha^2+\eta^2 y_t z_t (y_t z_t - \sigma_1^2), \\
        w_{2,t}\triangleq 1-\eta y_t^2-\eta\nicefrac{\sigma_p^2}{\alpha^2}+\eta^2 y_t z_t (y_t z_t - \sigma_1^2),
    \end{gather*}
    where $[\cdot]_p$ means the $p$-th value in a vector.

    Recall we have $y_t, z_t$ following the training dynamics of minimizing $\frac{1}{2}(\sigma_1^2-y z)^2$ with learning rate $\eta>\frac{1}{\sigma_1^2}$, where leads to $y=z=\gamma_i$, with $\gamma_i$ ($i=1,2$) are the two roots of solving the 1-D function~(\ref{eq:1d_orbit}) as $\delta$. We denote their corresponding $\mathbf{Q}$ as
    $\mathbf{Q}_{\alpha,\eta,p}(\gamma_1, \gamma_1)$ and $\mathbf{Q}_{\alpha,\eta,p}(\gamma_2, \gamma_2)$. We assume that their product $\mathbf{Q}_{\alpha,\eta,p}(\gamma_2, \gamma_2)\mathbf{Q}_{\alpha,\eta,p}(\gamma_1, \gamma_1)$ is diagonalizable with all eigenvalues falling into $(-1,1)$, which means its infinite power $\lim_{k\rightarrow \infty}[\mathbf{Q}_{\alpha,\eta,p}(\gamma_2, \gamma_2)\mathbf{Q}_{\alpha,\eta,p}(\gamma_1, \gamma_1)]^{k} = 0$. Meanwhile, due to the 2-D analysis of dynamics of GD on $\frac{1}{2}(\sigma_1^2-y z)^2$, we know $(y_t, z_t)\rightarrow \{(\gamma_1,\gamma_1), (\gamma_2, \gamma_2)\}$ exponentially after finite steps. This is equivalent to say, there exists finite $t_0$, for any $t>t_0$, there exists $i\in\{1,2\}$, constant $C_0$ and $\mathbf{R}_t\in\R^{4\times4}$, such that
    \begin{gather*}
        \mathbf{Q}_{\alpha,\eta,p}(y_{t+1}, z_{t+1})\mathbf{Q}_{\alpha,\eta,p}(y_t, z_t) = \mathbf{Q}_{\alpha,\eta,p}(\gamma_{3-i}, \gamma_{3-i})\mathbf{Q}_{\alpha,\eta,p}(\gamma_i, \gamma_i) + \mathbf{R}_{t}, ~~~~ \norm{\mathbf{R}_{t}}\le C_0 r^t, ~~~0<r<1.
    \end{gather*}
    The decay rate $r$ can be estimated via local analysis around the convergence orbit.
    As a result, it is safe to say $\lim_{t\rightarrow\infty}\mathbf{Q}_{\alpha,\eta,p}(y_{2t+1}, z_{2t+1})\mathbf{Q}_{\alpha,\eta,p}(y_{2t}, z_{2t})=0$, which means all of $\mathbb{I}_{\ge2}^\top\mathbf{M}_t\mathbb{I}_1$, $\mathbb{I}_{1}^\top\mathbf{M}_t\mathbb{I}_{\ge2}$, $\mathbb{I}_{\ge2}^\top\left(\bY_0+\Delta\bY_{t+1}\right)\left(\bY_0+\Delta\bY_{t+1}\right)^\top\mathbb{I}_{1}$, $\mathbb{I}_{1}^\top\left(\bZ_0+\Delta\bZ_{t+1}\right)\left(\bZ_0+\Delta\bZ_{t+1}\right)^\top\mathbb{I}_{\ge2}$ exponentially go to zero.

    There is one concern here: what happens before $t_0$? More concretely, $t_0$ is dependent of $\nicefrac{1}{\epsilon}$ because it requires more steps (intuitively proportional to $\log\nicefrac{1}{\epsilon}$) to increase to a certain value from a small $\epsilon$. Assuming $t_0\sim \log\nicefrac{1}{\epsilon}$ holds, the product $\{\mathbf{Q}_{\alpha,\eta,p}(y_{2t+1}, z_{2t+1})\mathbf{Q}_{\alpha,\eta,p}(y_{2t}, z_{2t})\}_{t\ge 1}$ gives a (loose) upper bound with the norm of products grows exponentially with time $\log\nicefrac{1}{\epsilon}$, which introduces $\nicefrac{1}{\epsilon}$ to the upper bound of $\norm{\mathbb{I}_{\ge2}^\top\mathbf{M}_{t}\mathbb{I}_1}$ and $\norm{\mathbb{I}_{1}^\top\mathbf{M}_{t}\mathbb{I}_{\ge2}}$, breaking the assumption of the norms staying in $\O(\epsilon)$. Fortunately, there are two aspects to resolve this. Firstly, with initialization $\epsilon$ small enough, for a relative long time, $\mathbf{Q}_{\alpha,\eta,p}(y_{2t+1}, z_{2t+1})\mathbf{Q}_{\alpha,\eta,p}(y_{2t}, z_{2t})$ is approximately having eigenvalues bounded by 1. More precisely, $\mathbf{Q}$ and the product are
    \begin{gather}
        \mathbf{Q}_{\alpha,\eta,p}(\cdot, \cdot)
        \approx 
        \begin{bmatrix}
            1 & 0 & -\eta\sigma_1^2 & -\eta\sigma_p^2 \\
            0 & 1 & -\eta\sigma_p^2 & -\eta\sigma_1^2 \\
            0 & 0 & 1-\eta \nicefrac{\sigma_1^2}{\alpha^2}-\eta\sigma_p^2\alpha^2 & 0 \\
            0 & 0 & 0 & 1-\eta \sigma_1^2\alpha^2-\eta\nicefrac{\sigma_p^2}{\alpha^2}
        \end{bmatrix},
        \\
        \Lambda(\mathbf{Q}_{\alpha,\eta,p}(\cdot, \cdot)\mathbf{Q}_{\alpha,\eta,p}(\cdot, \cdot))
        =
        \{
            1, 1, (1-\eta \nicefrac{\sigma_1^2}{\alpha^2}-\eta\sigma_p^2\alpha^2)^2, (1-\eta \sigma_1^2\alpha^2-\eta\nicefrac{\sigma_p^2}{\alpha^2})^2,
        \}
        \label{eq:quasi_q_lambda}
    \end{gather}
    where the eigenvalues in $\Lambda$ are upper bounded by 1, if assuming $\eta(\nicefrac{\sigma_1^2}{\alpha^2}+\sigma_p^2\alpha^2)<2$ and $1-\eta \sigma_1^2\alpha^2-\eta\nicefrac{\sigma_p^2}{\alpha^2}<2$. As a result, in these steps, $\norm{\mathbb{I}_{\ge2}^\top\mathbf{M}_{t}\mathbb{I}_1}$ and $\norm{\mathbb{I}_{1}^\top\mathbf{M}_{t}\mathbb{I}_{\ge2}}$ stay in $\O(\epsilon)$ due to $\mathbf{Q}_{\alpha,\eta,p}(\cdot, \cdot)\mathbf{Q}_{\alpha,\eta,p}(\cdot, \cdot)$ is a semi-convergent matrix. Secondly, the eigenvectors of $\mathbf{Q}_{\alpha,\eta,p}(\cdot, \cdot)\mathbf{Q}_{\alpha,\eta,p}(\cdot, \cdot)$ corresponding to eigenvalue 1 are $[1,0,0,0]^\top$ and $[0,1,0,0]^\top$, which means $\norm{\mathbb{I}_{\ge2}^\top\mathbf{M}_{t}\mathbb{I}_1}$ and $\norm{\mathbb{I}_{1}^\top\mathbf{M}_{t}\mathbb{I}_{\ge2}}$ are decaying exponentially. Therefore, smaller $\epsilon$ strengthens the assumption of $\norm{\mathbb{I}_{\ge2}^\top\mathbf{M}_{t}\mathbb{I}_1}$ and $\norm{\mathbb{I}_{1}^\top\mathbf{M}_{t}\mathbb{I}_{\ge2}}$ staying in $\O(\epsilon)$ instead of breaking it.

    Also note that $\norm{\mathbb{I}_{\ge2}^\top\mathbf{M}_{t+1} \mathbb{I}_{\ge2}}\lessapprox\norm{\mathbb{I}_{\ge2}^\top\mathbf{M}_{t} \mathbb{I}_{\ge2}}\cdot\max\{ |1-\eta\sigma_2\left( \alpha^2 + \nicefrac{1}{\alpha^2} \right)|, |1-\eta\sigma_{d-1}\left( \alpha^2 + \nicefrac{1}{\alpha^2} \right)| \}$, so $\norm{\mathbb{I}_{\ge2}^\top\mathbf{M}_{t+1} \mathbb{I}_{\ge2}}$ decays exponentially.

    Since all of $\norm{\mathbb{I}_{\ge2}^\top\mathbf{M}_{t}\mathbb{I}_1}$, $\norm{\mathbb{I}_{1}^\top\mathbf{M}_{t}\mathbb{I}_{\ge2}}$ and $\norm{\mathbb{I}_{\ge2}^\top\mathbf{M}_{t+1} \mathbb{I}_{\ge2}}$ decay exponentially after some steps, all of them are have the sum upper-bounded, which means $\norm{\cc{2}_t}, \norm{\cc{3}_t}, \norm{\cc{4}_t}, \norm{\cc{6}_t}, \norm{\cc{7}_t}, \norm{\cc{8}_t}$ stay in $\O(\epsilon)$.

    To summarize, it holds
    \begin{enumerate}
        \item $\norm{\cc{2}_t}, \norm{\cc{3}_t}, \norm{\cc{4}_t}, \norm{\cc{6}_t}, \norm{\cc{7}_t}, \norm{\cc{8}_t}$ stay in $\O(\epsilon)$.
        \item $\norm{\mathbb{I}_{1}^\top\mathbf{M}_{t}\mathbb{I}_{\ge2}}$ and $\norm{\mathbb{I}_{\ge2}^\top\mathbf{M}_{t+1} \mathbb{I}_{\ge2}}$ decays to zero.
        \item $\norm{\mathbb{I}_{1}^\top\mathbf{M}_{t}\mathbb{I}_{1}}$ stays in a period-2 orbit.
    \end{enumerate}

\end{proof}

\section{Useful lemmas}
\begin{lemma}
Assume $a\cdot\Delta a\ge b\cdot\Delta b$ and $a\ge b$. All of $a,b,\Delta a, \Delta b$ are positive. If $\Delta b\le a$, then $a+\Delta a\ge b+\Delta b$.
\label{lem:ab1}
\end{lemma}
\begin{proof}
$(a+\Delta a)-(b+\Delta b)\ge a+b\frac{\Delta b}{a}-b-\Delta b=(\frac{\Delta b}{a}-1)(b-a)\ge 0$.
\end{proof}

\section{Illustration of period-2 and period-4 orbits}
\label{app:add_figs}


In the setting of $f(x)=\frac{1}{4}(x^2-1)^2$, local convergence is guaranteed if $\eta<\sqrt{5}-1\approx 1.236$ by taylor expansion of $F^2_\eta$ around the orbit. Conversely, if the learning rate is larger than it, although the period-2 orbit still exists, GD starting from a point infinitesimally close to the orbit still escapes from it. This is when GD converges to a higher-order orbit.

Figure~\ref{fig:xy_three_lr} precisely shows the effectiveness of such a bound where GD converges to the period-2 orbit when $\eta=1.235<\sqrt{5}-1$ and a period-4 orbit when $\eta=1.237>\sqrt{5}-1$.

\begin{figure}[h]
    \centering
    \includegraphics[width=0.95\textwidth]{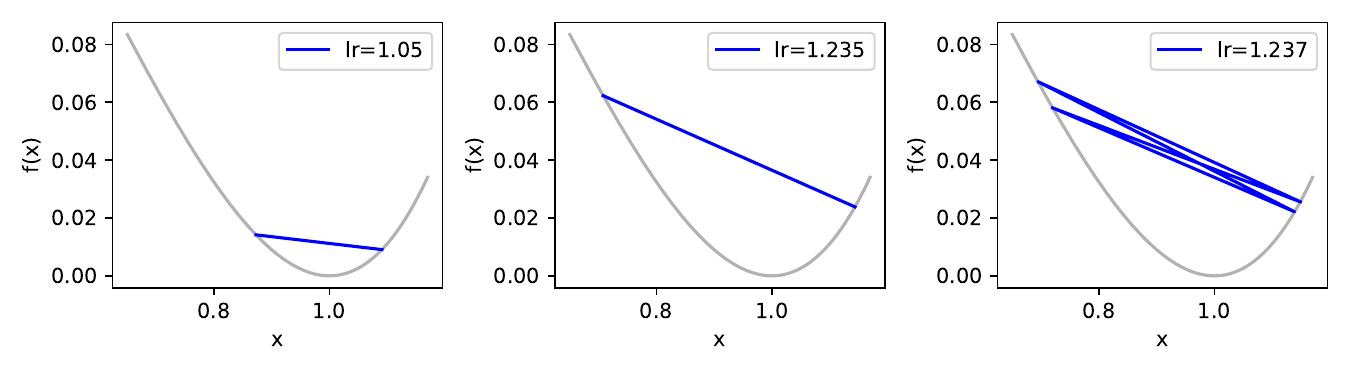}
    \caption{The convergent orbits of GD on $f(x)=\frac{1}{4}(x^2-1)^2$ with learning rate=1.05, 1.235 and 1.237. The first two smaller learning rates drive to period-2 orbits while the last one goes to an period-4 orbit. The significant bound between period-2 and period-4 is predictable by Taylor expansion around the period-2 orbit, as $\eta=\sqrt{5}-1\approx 1.236$.}
    \label{fig:xy_three_lr}
\end{figure}

\section{Discussions}
\label{app:discuss}
First, we provide a general roadmap of our theoretical results in Appendix~\ref{app:connections}, as illustrated in Figure~\ref{fig:connections}. Then, in Appendix~\ref{app:implications} we discuss three implications from our current low-dimensional settings to more complicated models for future understanding of EoS in pratical NNs, where low-dimension theorems are enhanced with high-dim experiments.

\subsection{Connections between theoretical results}
\label{app:connections}

In this section, we discuss the connections between our presented theoretical results, as illustrated in Figure~\ref{fig:connections}.

\begin{figure}[t]
    \centering
\begin{tikzcd}
                                                         &  & \text{Local Geometry (Thm.~\ref{thm:1dlocal})} \arrow[d] \arrow[rrd] \arrow[lld] \arrow[lldd] &  &                                                                       \\
\text{High-order LG (Lem.~\ref{lem:1dhigher})} \arrow[d] &  & \text{1-D case (Thm.~\ref{thm:1dglobal})} \arrow[d]                                           &  & \text{LG for MF (Thm.~\ref{thm:mf_1d_cond})} \arrow[dd, bend left=49] \\
(g(x)-y)^2 \text{~(Prop.~\ref{prop:l2loss})} \arrow[d]     &  & \text{2-D case (Prop.~\ref{prop:xy})} \arrow[d] \arrow[rrd]                                   &  & \text{Balancing effect (Thm.~\ref{thm:xy_diff_decay})} \arrow[ll]     \\
\text{Composition rule of $g$ (Prop~\ref{prop:comp})}                                       &  & \text{Single-neuron (Prop.~\ref{prop:single_neuron})}                                         &  & \text{Quasi-sym MF (Obs.~\ref{obs:quasi_mf})}                 \\
                                                         &  & w_y\rightarrow 0 \text{~(Thm.~\ref{thm:one_neuron})} \arrow[u]                                &  &                                                                      
\end{tikzcd}
    \caption{Connections between our presented theoretical results. The arrows stand for ``implies''. LG stands for Local Geometry. MF stands for Matrix Factorization.}
    \label{fig:connections}
\end{figure}
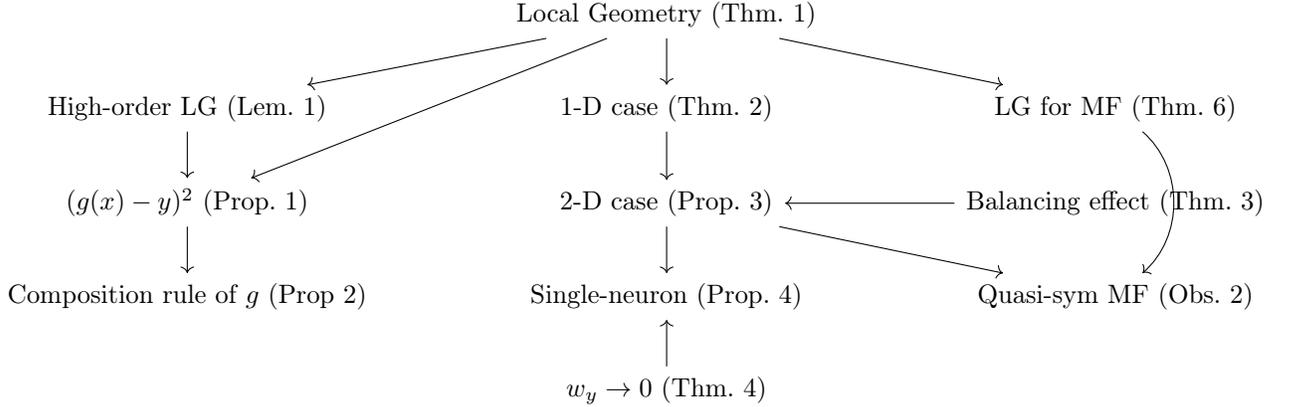

Theorem~\ref{thm:1dlocal} and Lemma~\ref{lem:1dhigher} present (local) intrinsic geometric properties for a 1-D function to allow stable oscillations. Such properties provide us the 1-D function $f(x)=(\mu-x^2)^2$ and, furthermore, we generalize the local property to a global convergence result in Theorem 2. Then we are to generalize the 1-D analysis to cases of i) multi-parameter, ii) nonlinear and iii) high-dimension.

\begin{enumerate}[label=(\alph*)]
    \item \textbf{Multi-parameter}. Compared with 1-D $f(x)=(\mu-x^2)^2$, the 2-D function $f(x,y)=(\mu-xy)^2$ can be viewed as the simplest setting of two-layer models. We prove that the 2-D case converges to the region of $x=y$ in Theorem~\ref{thm:xy_diff_decay} in Section~\ref{sec:convergence}, which means it shares the same convergence as the 1-D model. Also, $x=y$ means its sharpness is the flattest.

    \item \textbf{Nonlinear}. We extend the 2-D model to a two-layer single-neuron ReLU model in Section~\ref{sec:one_neuron}. Although the student neuron can be initialized far from the direction of the teacher neuron, we prove the student neuron converges to the correct direction (as $w_y\rightarrow 0$) in Theorem~\ref{thm:one_neuron}. Then the problem degenerates to the above 2-D analysis, which means it shares the same convergence with the 2-D, where $(v,w_x)$ corresponds to $(x,y)$ in 2-D.

    \item \textbf{High-dimension}. We extend the 2-D model to quasi-symmetric matrix factorization in Section~\ref{sec:new_mf}. Although the parameters are initialized near a sharp minima, GD still walks towards the flattest minima, as shown in Observation~\ref{obs:quasi_mf}. At convergence, the top singular values of $\mathbf{Y}, \mathbf{Z}$ are the same, following the 2-D analysis. So the singular values are in the same period-2 orbit as the 1-D case.
\end{enumerate}

Meanwhile, from Theorem~\ref{thm:1dlocal} and Lemma~\ref{lem:1dhigher}, we prove a condition for base models $g$ in regression tasks to allow stable oscillation in Prop~\ref{prop:l2loss}. Furthermore, we provide a composition rule of two base models to find a more complicated model that allows stable oscillation in Prop~\ref{prop:comp}.

\subsection{Implications from low-dimension to high-dimension} \label{app:implications}

We would like to emphasize that, although our current simple settings are a little far from practical NNs, it still helps understand the ability of GD at large LRs to discover flat minima in three steps as follows. We include more experiments in Appendix~\ref{app:high_dim_exp} to present the following hopes for complicated networks:

\begin{enumerate}[label=(\alph*)]
    \item By Theorem~\ref{thm:1dlocal}, especially its second condition, we wish to discover an intrinsic geometric property around local minima of more complicated models. The key is to investigate the 1-D function at the cross-section of the leading eigenvector and the loss landscape.

    \begin{itemize}
        \item[\ding{117}] Theoretical: we prove the 1-D condition holds at any minima for non-trivial matrix factorization, shown as Theorem~\ref{thm:mf_1d_cond} in Appendix~\ref{app:add_res_mf}.
        
        \item[\ding{93}] Empirical: we show the 1-D condition holds around minima of 3,4,5-layer ReLU MLPs on MNIST, shown in Figure~\ref{fig:add_exp_mnist}(d), \ref{fig:add_exp_mnist_four_layer}(d), \ref{fig:add_exp_mnist_five_layer}(d) in Appendix~\ref{app:mnist}.
    \end{itemize}

    \item With the above intrinsic geometric property, the next question is whether the training trajectory utilizes this property.

    \begin{itemize}
        \item[\ding{117}] Theoretical: in the case of quasi-symmetric matrix factorization, we observe and provide theoretical intuition that the training trajectory follows the leading eigenvector of the Hessian (i.e. the leading component of $\mathbf{X}_0$) in Observation~\ref{obs:quasi_mf}, where the only top components of weights are changing in $\omega(\epsilon)$.

        \item[\ding{93}] Empirical: for MLPs on MNIST, we show the almost perfect alignment of the gradient and the top Hessian eigenvector in Figure~\ref{fig:add_exp_mnist}(c), \ref{fig:add_exp_mnist_four_layer}(c), \ref{fig:add_exp_mnist_five_layer}(c).
    \end{itemize}

    \item The final implication is the implicit bias of EoS after such oscillation. It turns out GD is driven to flatter minima from sharper minima. In the 1-D case, obviously there is nothing about implicit bias since the only thing GD is doing is to approximate the target value. However, an implicit bias from the oscillation appears starting from the 2-D case.

    \begin{itemize}
        \item[\ding{117}] Theoretical 1: in the 2-D case in Theorem~\ref{thm:xy_diff_decay}, we prove the two learnable parameters $x,y$ will converge to the same values after oscillations of their product $xy$. Actually in the minimum manifold, smaller $|x-y|$ means a flatter minimizer.

        \item[\ding{117}] Theoretical 2: in the single-neuron ReLU network in Theorem~\ref{thm:one_neuron} and Prop~\ref{prop:single_neuron}, we show the model degenerates to the 2-D case since $w_y\rightarrow 0$. The 2-D argument tells that this nonlinear model also walks towards the balanced situation, verified with experiments in Figure~\ref{fig:one_neuron}.

        \item[\ding{117}] Theoretical 3: in the quasi-symmetric MF in Obs~\ref{obs:quasi_mf}, although the initialization is around a sharp minima, GD is still driven towards the flattest minima where $\sigma_{\max}(\bY)=\sigma_{\max}(\bZ)$.

        \item[\ding{93}] Empirical 1: for 2-layer 16-neuron ReLU network in a student-teacher setting, it turns out learning rate decay after beyond-EoS oscillations drives the model very close to the flattest minima, as shown in Figure~\ref{fig:add_exp_high_ts} and in Appendix~\ref{app:16-neuron}.

        \item[\ding{93}] Empirical 2: for 3,4,5-layer MLPs on MNIST, larger learning rate drives to a flatter minima, as shown in Figure~\ref{fig:add_exp_mnist}(b).
    \end{itemize}
\end{enumerate}

\end{document}